\documentclass{article}

\usepackage{microtype}
\usepackage{graphicx}
\usepackage{subcaption}
\usepackage{booktabs} 
\usepackage{graphicx}

\usepackage{hyperref}

\usepackage[preprint]{anonymousconference}

\usepackage{amsmath}
\usepackage{amssymb}
\usepackage{mathtools}
\usepackage{amsthm}
\usepackage{bm}
\usepackage{comment}
\allowdisplaybreaks[3]
\usepackage{enumitem}

\usepackage[capitalize,noabbrev]{cleveref}

\theoremstyle{plain}
\newtheorem{theorem}{Theorem}[section]

\newtheorem{lemma}[theorem]{Lemma}

\theoremstyle{definition}
\newtheorem{definition}[theorem]{Definition}

\theoremstyle{remark}

\usepackage[textsize=tiny]{todonotes}

\anonymousconferencetitlerunning{Learning from Similarity-Confidence and Confidence-Difference}

\begin{document}

\twocolumn[
  \anonymousconferencetitle{Learning from Similarity-Confidence and Confidence-Difference}

  \anonymousconferencesetsymbol{equal}{*}

  \begin{anonymousconferenceauthorlist}
    \anonymousconferenceauthor{Tomoya Tate}{sch}
    \anonymousconferenceauthor{Kosuke Sugiyama}{sch}
    \anonymousconferenceauthor{Masato Uchida}{sch}
  \end{anonymousconferenceauthorlist}

  \anonymousconferenceaffiliation{sch}{Waseda University,
  3-4-1 Okubo, Shinjuku, Tokyo 169-8555, Japan}

  \anonymousconferencecorrespondingauthor{Tomoya Tate}{tomot203@toki.waseda.jp}
  \anonymousconferencecorrespondingauthor{Kosuke Sugiyama}{kohsuke0322@asagi.waseda.jp}
  \anonymousconferencecorrespondingauthor{Masato Uchida}{m.uchida@waseda.jp}

  \anonymousconferencekeywords{statistical learning theory, weakly supervised learning, pairwise label, unbiased risk estimator, Rademacher complexity, anonymousconference}

  \vskip 0.3in
]

\printAffiliationsAndNotice{}  

\begin{abstract}
In practical machine learning applications, it is often challenging to assign accurate labels, and increasing the number of labeled instances is often limited. In such cases, Weakly Supervised Learning (WSL), which enables training with incomplete supervision, is effective. However, most existing WSL methods focus on leveraging a single type of weak supervision, and the framework on improving learning performance by assigning multiple weak labels to a limited instances remains underdeveloped. In this paper, we propose \textit{SconfConfDiff classification}, a method that integrates two distinct forms of weak labels: similarity-confidence and confidence-difference, which are assigned to unlabeled data pairs. To implement this method, we derive two types of unbiased risk estimators for classification: one based on a convex combination of existing estimators, and another newly designed by modeling the interaction between two weak labels. We prove that both estimators achieve optimal convergence rates with respect to estimation error bounds. Furthermore, we introduce a risk correction approach to mitigate overfitting, and provide theoretical analysis on the robustness against inaccurate class prior probability and label noise. Experimental results demonstrate the effectiveness of the proposed methods across a variety of settings.  
\end{abstract}

\section{Introduction}
In recent years, machine learning models have achieved remarkable performance across a wide range of applications. However, training high-performance models typically requires large amounts of accurately labeled data. In many real-world scenarios, acquiring such labels is often expensive, time-consuming, or even impractical, thereby limiting the availability of sufficient training data. To address this challenge, various Weakly Supervised Learning (WSL) methods~\cite{sugiyama2022machine} have been proposed to exploit incomplete supervision, including Positive-Unlabeled Learning~\cite{PU-learning,PU_analysis,PU_convex,PU_NN_risk,PNU-learning}, Partial-Label Learning~\cite{partial}, Unlabeled-Unlabeled Learning~\cite{UU}, and Complementary-Label Learning~\cite{complementary,multi-complementary}. More recently, research has shifted toward learning from labeled pairs of instances rather than individual samples~\cite{SU, SD, pcomp}, demonstrating promising results. Notable examples include Similarity-Confidence (Sconf) learning~\cite{sconf} and Confidence-Difference (ConfDiff) classification~\cite{confdiff}.

In addition to labeling challenges, there are scenarios where increasing the number of instances is simply not feasible -- not due to cost or effort, but because the data pertains to phenomena that are inherently rare, unobservable, or access-restricted. For instance, in predicting the risk of rare diseases, obtaining labels may be challenging due to the need for expert knowledge. Similarly, in domains such as security and privacy, data availability may be constrained by ethical, legal, or technical barriers. When the number of instances is fundamentally limited, increasing the sample size becomes difficult regardless of intent. In such scenarios, there is a critical need for learning frameworks that enhance model performance by enriching limited data with information derived from multiple diverse perspectives.

Conventional WSL methods typically assume that each training instance is annotated with a single type of weak label, and accordingly design risk estimators specific to that label type. In contrast, we investigate a novel setting in which multiple forms of weak supervision are simultaneously available and can be utilized jointly for training. Our objective is to derive risk estimators that aggregate these labels in a unified manner. Our key hypothesis is that leveraging multiple sources of weak supervision can lead to more accurate models, even when the number of labeled instances is limited.

In this paper, we propose an effective training method, \textit{SconfConfDiff classification}, which learns a binary classifier using two types of weak labels: similarity-confidence, introduced in Sconf learning~\cite{sconf}, and confidence-difference, introduced in ConfDiff classification~\cite{confdiff}. Similarity-confidence is defined as the probability that a pair of data points belongs to the same class, while confidence-difference represents the difference between the probabilities that each point in the pair belongs to the positive class. By leveraging them, we can construct richer supervision that facilitates more effective learning. We have also analyzed and developed a framework for effectively utilizing hard weak labels such as similar/dissimilar or which one is more positive, named \textit{SD-Pcomp classification}. The results of SD-Pcomp are presented in a separate paper~\cite{SD-Pcomp}.

The main findings of this study are as follows.
First, we derive two unbiased risk estimators for binary classification, both of which can be computed from similarity-confidence and confidence-difference. In Section~\ref{sec:convex_unbiased}, we construct the first estimator as a convex combination of the unbiased estimators individually derived from Sconf learning and ConfDiff classification. While this approach can be effective, as demonstrated in the experiments, it treats similarity-confidence and confidence-difference in isolation, ignoring how they jointly reflect the relationship between labels. To address this limitation, we develop a second estimator in Section~\ref{sec:risk estimator} that directly models the interaction between the two weak supervision signals, thereby capturing their combined influence on label prediction, and we analyze the variance of the estimator in Section~\ref{sec:minimum_var}.  Furthermore, in Section~\ref{sec:bound}, 
we analyze the estimation error of our proposed method, and we theoretically investigate the impact of class prior shift and label noise on the estimation error bound in Section~\ref{sec:robust}.
Additionally, to mitigate overfitting resulting from negative empirical risk values, we introduce a risk correction approach and derive the estimation error bound under risk correction in Section~\ref{sec:risk correction}.

We conduct numerical experiments to demonstrate that the proposed method achieves higher classification performance than the individual Sconf and ConfDiff.
We describe the settings in Section~\ref{exp_setup}, and assess the classification accuracy in Section~\ref{exp:acc}.
Section~\ref{exp:few data} and~\ref{exp:robust} present experiments examining the behavior of classification performance under a reduced number of training samples and noisy settings.

\section{Related Works and Preliminaries}

In this section, we describe the problem setting and present the formulations of binary classification, Sconf learning, and ConfDiff classification.

\subsection{Ordinary Binary Classification Setting}

In a binary classification problem, let $d$-dimensional feature space be $\mathcal{X}\subseteq\mathbb{R}^d$ and the label space be $\mathcal{Y}=\{-1,+1\}$. We assume that an instance and its class label $(\bm{x},y)$ are drawn from an unknown distribution $p(\bm{x},y)$. The goal is to find a binary classifier $g:\mathcal{X}\rightarrow\mathbb{R}$ that minimizes the following classification risk~\cite{mohri}:
\begin{align}
R(g)=\mathbb{E}_{p(\bm{x},y)}[\ell(g(\bm{x}),y)],\label{classification risk}
\end{align}
where $\ell:\mathbb{R}\times\mathcal{Y}\rightarrow\mathbb{R}^+$ is a non-negative binary classification loss function, such as 0-1 loss and logistic loss. 
Under this setting, the classification risk $R(g)$ defined in Eq.~\eqref{classification risk} can also be written as follows.
\begin{align}
R(g)=\pi_+\mathbb{E}&_{p(\bm{x} \mid y=+1)}[\ell(g(\bm{x}),+1)]\notag\\
&+\pi_-\mathbb{E}_{p(\bm{x} \mid y=-1)}[\ell(g(\bm{x}),-1)],\label{classification risk 2}
\end{align}
where $\pi_+:=p(y=+1)$ and $\pi_-:=p(y=-1)$ are the class prior probabilities of the positive and negative classes, respectively.

\subsection{Similarity-Confidence (Sconf) Learning}\label{sconf setting}

In Sconf learning~\cite{sconf}, we are given unlabeled data pairs annotated with similarity-confidence, $\mathcal{D}_n=\{(\bm{x}_i,\bm{x}_i'),s_i\}_{i=1}^{n}$, without ordinary class labels. Here, $s_i=s(\bm{x}_i,\bm{x}_i')=p(y_i=y_i' \mid \bm{x}_i,\bm{x}_i')$ denotes the probability that two instances $\bm{x}_i$ and $\bm{x}_i'$ are the same class. 
An unbiased estimator of the classification risk for Sconf learning is as follows:
\begin{align}
\textstyle \widehat{R}_{\mathrm{Sconf}}(g)=\frac{1}{2n}&\textstyle\sum_{i=1}^{n}\Bigl\{\frac{(\pi_--s_i)(\ell(g(\bm{x}_i),+1)+\ell(g(\bm{x}_i'),+1))}{\pi_--\pi_+}\notag\\
&\textstyle +\frac{(\pi_+-s_i)(\ell(g(\bm{x}_i),-1)+\ell(g(\bm{x}_i'),-1))}{\pi_+-\pi_-}\Bigr\}.\label{sconf risk}
\end{align}
Here, it is assumed that the unlabeled data pair $(\bm{x}_i,\bm{x}_i')$ independently follows a probability distribution with density $p(\bm{x},\bm{x}')=p(\bm{x})p(\bm{x}')$.

\subsection{Confidence-Difference (ConfDiff) Classification}\label{confdiff setting}

In ConfDiff classification~\cite{confdiff}, we are given unlabeled data pairs annotated with confidence-difference, $\mathcal{D}_n=\{(\bm{x}_i,\bm{x}_i'),c_i\}_{i=1}^{n}$, without ordinary class labels. Here, $c_i=c(\bm{x}_i,\bm{x}_i')=p(y_i'=+1 \mid \bm{x}_i')-p(y_i=+1 \mid \bm{x}_i)$ denotes the difference in the posterior probabilities between the unlabeled data pair. 
An unbiased estimator of the classification risk for ConfDiff classification is as follows:
\begin{align}
\textstyle \widehat{R}_{\mathrm{CD}}(g)=\frac{1}{2n}\sum_{i=1}^{n}\{\mathcal{L}_\mathrm{CD}(\bm{x}_i,\bm{x}_i')+\mathcal{L}_\mathrm{CD}(\bm{x}_i',\bm{x}_i)\},\label{confdiff risk}
\end{align}
where
\begin{align*}
&\mathcal{L}_{\mathrm{CD}}(\bm{x}_i,\bm{x}_i')\\
&=(\pi_+-c_i)\ell(g(\bm{x}_i),+1)+(\pi_--c_i)\ell(g(\bm{x}_i'),-1)
\end{align*}
The assumption regarding the unlabeled data pair $(\bm{x}_i,\bm{x}_i')$ is the same as that in Section~\ref{sconf setting}.

We remark that similarity-confidence and confidence-difference encode different aspects of the relationships between data points. In particular, they can be interpreted as capturing their relative positioning from distinct geometric perspectives, as illustrated in Appendix~\ref{appendix:geometry}, which offers additional intuition about the nature of these labels.

\section{The Proposed Method and Analysis}\label{proposed}

\begin{figure*}[ht]
\vskip 0.2in
\begin{center}
\centerline{\includegraphics[width=1.0\linewidth]{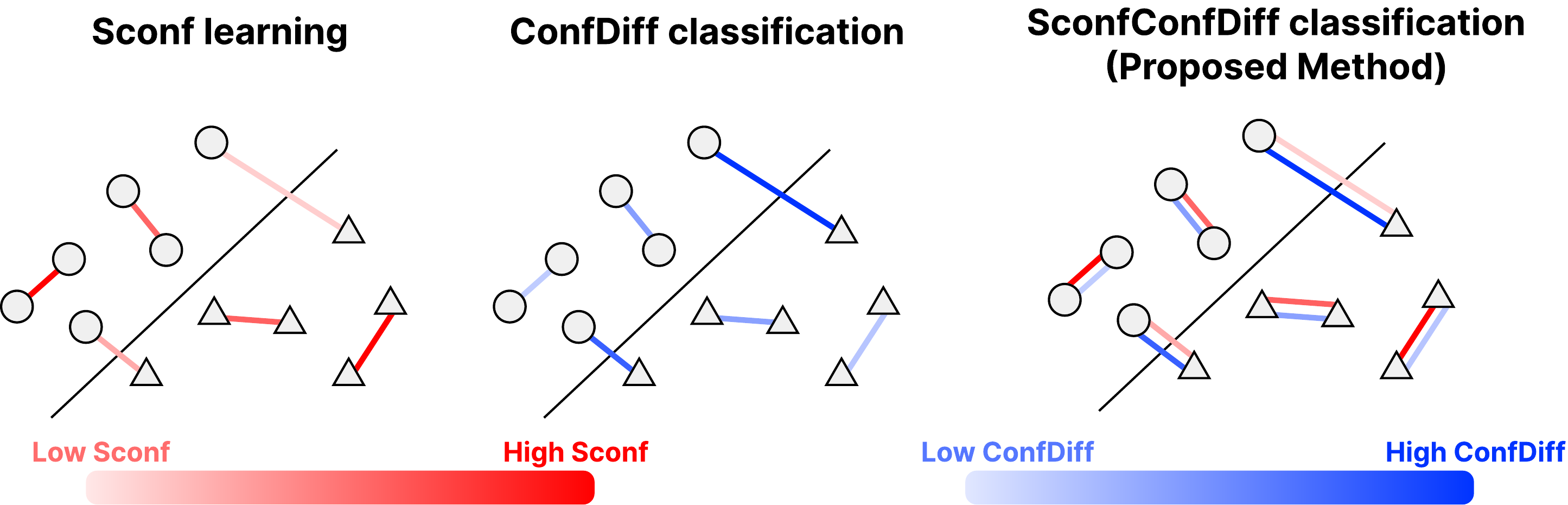}}
\caption{Illustration of Sconf learning, ConfDiff classification, and SconfConfDiff classification}
\label{fig:sconfconfdiff}
\end{center}
\vskip -0.2in
\end{figure*}

Under the above setting, we consider the problem of learning a binary classifier from unlabeled data pairs annotated with similarity-confidence and confidence-difference, $\mathcal{D}_n=\{(\bm{x}_i,\bm{x}_i'),s_i,c_i\}_{i=1}^{n}$. In this section, we formally define the proposed method and provide its theoretical guarantees. 
Furthermore, we analyze the impact of an noisy class prior probability and noise in similarity-confidence and confidence-difference, and we describe a risk correction approach to mitigate overfitting.
Theoretical proofs are provided in the supplementary materials.

\subsection{Simple Convex Combination for Unbiased Risk Estimation}\label{sec:convex_unbiased}

As previously discussed, unbiased risk estimators based solely on either similarity-confidence or confidence-difference can be expressed by Eq.~\eqref{sconf risk} and Eq.~\eqref{confdiff risk}, respectively. Given these prior results, it follows naturally that an unbiased estimator for the classification risk in Eq.~\eqref{classification risk} can be constructed as a convex combination of the two estimators. We formally present this result in the following theorem.

\begin{theorem}\label{thm:unbiased risk sum}

The following is an unbiased risk estimator of Eq.~\eqref{classification risk}.
\begin{align}
\widehat{R}_{\mathrm{SCD_{convex}}}(g)=\gamma\widehat{R}_{\mathrm{Sconf}}(g)+(1-\gamma)\widehat{R}_{\mathrm{CD}}(g),\label{unbiased risk sum}
\end{align}
where $\gamma\in[0,1]$ is an arbitrary weight.
\end{theorem}

Figure~\ref{fig:convex_failure} shows the classification performance based on empirical risk minimization as defined by Eq.~\eqref{unbiased risk sum}, where the effects of applying the ReLU correction function and the absolute value correction function are compared. The absolute value correction retains the loss even when the estimated risk is negative, thus allowing the method to leverage both similarity-confidence and confidence-difference information. As a result, a simple convex combination using the absolute value correction achieves higher classification accuracy than the original unbiased estimator without correction. In contrast, the ReLU correction replaces negative risk values with zero, thereby inappropriately ignoring the corresponding losses. This may cause the loss associated with either $s_i$ or $c_i$ to be omitted for certain instances, thereby failing to account for the inherent correlation between the two signals, which are potentially complementary and may benefit from being used together. Consequently, we observe a degradation in classification performance when using the ReLU correction, as compared to using either similarity-confidence ($\gamma = 1.0$) or confidence-difference ($\gamma = 0.0$) alone. The theoretical analysis of the convex combination approach and the associated risk correction techniques is based on existing work and is provided in Appendix~\ref{appendix_convex}.

Although the design of the risk correction function influences the estimation error, a more fundamental limitation lies in the structure of the estimator based on a simple convex combination. Specifically, the current approach computes the losses for similarity-confidence $s_i$ and confidence-difference $c_i$ independently for each data pair, and aggregates them without coherently utilizing both sources of information for individual instances. As a result, the estimator fails to fully exploit the combined effects of similarity-confidence and confidence-difference, which may structurally limit classification performance. To address this issue, we propose a novel unbiased risk estimator that directly incorporates both similarity-confidence and confidence-difference for each instance, rather than relying on a simple convex combination.

\begin{figure*}
\begin{center}
\vskip 0.2in
\centerline{\includegraphics[width=1.0\linewidth]{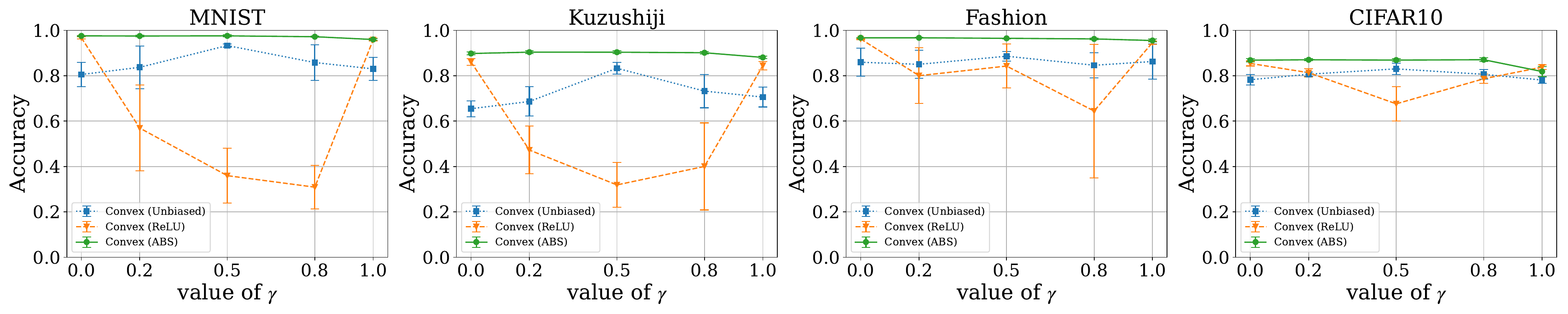}}
\caption{Experimental results for based on empirical risk minimization in Eq.~\eqref{unbiased risk sum}. $\gamma=0$ means ConfDiff Classificaiton and $\gamma=1$ means Sconf learning.}
\label{fig:convex_failure}
\end{center}
\vskip -0.2in
\end{figure*}

\subsection{Unbiased Risk Estimator of SconfConfDiff}\label{sec:risk estimator}

In this section, we show that the classification risk defined in Eq.~\eqref{classification risk} can be equivalently expressed using Sconf and ConfDiff.

\begin{theorem}\label{thm:unbiased risk}
The classification risk $R(g)$ in Eq.~\eqref{classification risk} can be equivalently expressed as
\begin{align}
\textstyle R_{\mathrm{SCD}}(g)=\mathbb{E}_{p(\bm{x},\bm{x}')}\left[\frac{1}{2}(\mathcal{L}(\bm{x},\bm{x}')+\mathcal{L}(\bm{x}',\bm{x}))\right],\label{expected risk}
\end{align}

where
\begin{align*}
\mathcal{L}&(\bm{x},\bm{x}')\\
=&(2\pi_+(\pi_+-c(\bm{x},\bm{x}'))+\pi_--s(\bm{x},\bm{x}'))\ell(g(\bm{x}),+1)\\
&+(2\pi_-(\pi_--c(\bm{x},\bm{x}'))+\pi_+-s(\bm{x},\bm{x}'))\ell(g(\bm{x}'),-1).
\end{align*}
\end{theorem}
Therefore, an unbiased risk estimator for SconfConfDiff classification can be derived as follows.
\begin{align}
\textstyle \widehat{R}_{\mathrm{SCD}}(g)=\frac{1}{2n}\sum_{i=1}^{n}&(\mathcal{L}(\bm{x}_i,\bm{x}_i')+\mathcal{L}(\bm{x}_i',\bm{x}_i)).\label{unbiased risk}
\end{align}

\subsection{Minimum Variance of The Risk Estimator}\label{sec:minimum_var}

Equation.~\eqref{unbiased risk lambda}, which replaces the 1/2 in Eq.~\eqref{unbiased risk} with a weighted sum using $\lambda\in[0,1]$, is also an unbiased risk estimator.

\begin{lemma}\label{lemma:unbiased risk lambda}
The following expression also serves as an unbiased risk estimator.
\begin{align}
\textstyle \frac{1}{n}\sum_{i=1}^{n}(\lambda\mathcal{L}(\bm{x}_i,\bm{x}_i')+(1-\lambda)\mathcal{L}(\bm{x}_i',\bm{x}_i)),\label{unbiased risk lambda}
\end{align}
where $\lambda\in[0,1]$ is an arbitrary weight.
\end{lemma}

Following the analysis of~\cite{confdiff}, we now present the following theorem, which characterizes the relationship between Eqs.~\eqref{unbiased risk} and \eqref{unbiased risk lambda}.

\begin{theorem}\label{thm:minimum variance}
The variance of Eq.~\eqref{unbiased risk} is minimized among the candidates for unbiased risk estimators represented by Eq.~\eqref{unbiased risk lambda} w.r.t. $\lambda\in[0,1]$.
\end{theorem}

Based on the result of Theorem~\ref{thm:minimum variance}, we adopt the risk estimator in Eq.~\eqref{unbiased risk} in the following analysis and experiments.

\subsection{Estimation Error Analysis}\label{sec:bound}

In this section, we discuss a theoretical investigation into the convergence of $\widehat{R}_{\mathrm{SCD}}(g)$ by providing an estimation error bound.
As we will show, the bound exhibits non-trivial behavior, reflecting the interplay between similarity-confidence and confidence-difference information.

Following the notation in~\cite{confdiff}, let $\mathcal{G}$ be a function class of binary classifiers. We assume that there exist a constant $C_g$ such that $\sup_{g\in\mathcal{G}}\lVert{g}\rVert_\infty\le{C_g}$ and a constant $C_{\ell}$ such that $\sup_{|z|\le{C_g}}\ell(z,y)\le{C_\ell}$. Furthermore, we assume that the binary loss function $\ell(z,y)$ is Lipschitz continuous w.r.t. $z$ and $y$ with a Lipschitz constant $L_{\ell}$. 
Here, we define $g^*=\arg\min_{g\in\mathcal{G}}R(g)$ and $\hat{g}_{\mathrm{SCD}}=\arg\min_{g\in\mathcal{G}}\widehat{R}_{\mathrm{SCD}}(g)$, and
$\mathfrak{R}_n(\mathcal{G})$ is the \textit{Rademacher complexity} of $\mathcal{G}$ for unlabeled data pair with size $n$ (Its definition is in Appendix~\ref{def:rademacher}).
Under these definitions, the following theorem holds.

\begin{theorem}\label{thm:estimation error}
For any $\delta>0$, the following inequality holds with probability at least $1-\delta$:
\begin{align}
&\textstyle R(\hat{g}_{\mathrm{SCD}})-R(g^*)\notag\\
&\textstyle \le12L_{\ell}\mathfrak{R}_n(\mathcal{G})+(2\pi_+^2+2\pi_-^2+5)C_{\ell}\sqrt{\frac{\log2/\delta}{2n}},\label{error bound}
\end{align}
\end{theorem}

Theorem~\ref{thm:estimation error} implies that
$R(\hat{g}_{\mathrm{SCD}}) \to R(g^*)$ as $n \to \infty$, because $\mathfrak{R}_n(\mathcal{G}) = O(1/\sqrt{n}) \to 0$ as $n \to \infty$ for all parametric models with a bounded norm~\cite{sugiyama2022machine}, such as deep neural networks trained with weight decay~\cite{Rademacher_bound}. 
Furthermore, the convergence rate of the bound in Eq~\eqref{error bound} is $\mathcal{O}_p(1/\sqrt{n})$, where $\mathcal{O}_p$ is the order in probability.
This rate is known to be typical for empirical risk minimization over a norm-bounded and finite-dimensional hypothesis class, and is considered optimal in the absence of additional smoothness or margin conditions~\cite{optimal_param}.

\subsection{Class Prior Estimation from Pairwise Label}
In our SconfConfDiff classification, the class prior probability $\pi_+$ is important in calculating the risk estimators. In this section, we describe the estimation method of class prior.

For the sample average of similarity-confidence, the following theorem holds:
\begin{theorem}[cf. Theorem 5 in~\cite{sconf}]\label{thm:class prior estimation}
The sample average of similarity-confidence is an unbiased estimator of $\pi_+^2+\pi_-^2$:
\begin{align*}
\textstyle \mathbb{E}[\frac{\sum_{i=1}^{n}s(\bm{x}_i,\bm{x}_i')}{n}]=\pi_+^2+\pi_-^2.
\end{align*}
Furthermore, according to McDiarmid's inequality~\cite{McDiarmid}, the sample average of similarity-confidence is consistent and converges to $\pi_+^2+\pi_-^2$ in the rate of $\mathcal{O}_p(\exp(-n))$, which is the optimal rate according to the central limit theorem.
\end{theorem}
Assuming $\pi_+>\pi_-$, we can calculate the class prior probability by $\pi_+=\frac{\sqrt{2(\pi_+^2+\pi_-^2)-1}+1}{2}$. According to Theorem~\ref{thm:class prior estimation}, we can approximate $\pi_+$ with the average of similarity-confidence and the formulation above.

\subsection{
Analysis of Noisy Class Prior and Label Noise
}\label{sec:robust}

In the preceding discussion, the class prior probability, the ground-truth similarity-confidence and the ground-truth confidence-difference were available.
However, as discussed in~\cite{sconf} and~\cite{confdiff}, these settings do not necessarily hold in realistic scenarios. In this section, we provide a theoretical analysis regarding the impact of the noise in class prior probability, similarity-confidence, and confidence-difference on SconfConfDiff classification.

Based on the analysis of~\cite{sconf} and~\cite{confdiff}, let $\bar{\mathcal{D}}_n=\{(\bm{x}_i,\bm{x}_i'),\bar{s}_i,\bar{c}_i\}_{i=1}^{n}$ denote noisy SconfConfDiff dataset. $\bar{s}_i$ and $\bar{c}_i$ are noisy similarity-confidence, and noisy confidence-difference, respectively. Furthermore, let $\bar{\pi}_+$ and $\bar{\pi}_-$ be the noisy class prior probabilities, and let $\bar{R}_{\mathrm{SCD}}$ denote the empirical risk computed using these noisy class priors together with noisy similarity-confidence, and noisy confidence-difference. We define $\bar{g}_{\mathrm{SCD}}=\arg\min_{g\in\mathcal{G}}\bar{R}_{\mathrm{SCD}}(g)$. Then, the following theorem holds.

\begin{theorem}\label{thm:robust estimation error}
Under the same assumptions as in Section~\ref{sec:bound}, for any $\delta>0$, the following inequality holds with probability at least $1-\delta$:
\begin{align}
\textstyle R&(\bar{g}_{\mathrm{SCD}})-R(g^*)\notag\\
\le& \textstyle 24L_{\ell}\mathfrak{R}_n(\mathcal{G})+(4\pi_+^2+4\pi_-^2+10)C_{\ell}\sqrt{\frac{\log2/\delta}{2n}}\notag\\
+& \textstyle \frac{4C_{\ell}}{n}\sum_{i=1}^{n}(|\bar{\pi}_+\bar{c}_i-\pi_+c_i|+|\bar{\pi}_-\bar{c}_i-\pi_-c_i|+|\bar{s}_i-s_i|)\notag\\
+& \textstyle 4C_{\ell}(|\bar{\pi}_+^2-\pi_+^2|+|\bar{\pi}_-^2-\pi_-^2|+|\bar{\pi}_+-\pi_+|).\label{robustness}
\end{align}
\end{theorem}

From this inequality, while we observe that the estimation error bound is influenced by both the mean absolute error of the noisy labels and the absolute error in the class prior probability. 
If the total noise grows slower than linearly with respect to the sample size with high probability and the class prior probability is consistently estimated, $\bar{R}_{\mathrm{SCD}}(g)$ is also consistent.

\subsection{Risk Correction Approach}\label{sec:risk correction}

Similar to Sconf learning~\cite{sconf} and ConfDiff classification~\cite{confdiff}, the risk estimator in Eq.~\eqref{unbiased risk} may take negative values, violating the non-negativity property of loss functions and potentially causing overfitting with complex models. 
To mitigate this, following ~\cite{sconf} and~\cite{confdiff}, we apply simple correction functions, such as the ReLU function $f(z)=\max(0,z)$ and the absolute value function $f(z)=|z|$. 
The corrected risk estimator for SconfConfDiff classification is then given by:
\begin{align}
\widetilde{R}_{\mathrm{SCD}}(g) = f(\widehat{A}(g)) + f(\widehat{B}(g)) + f(\widehat{C}(g)) + f(\widehat{D}(g)),\label{corrected risk}
\end{align}
where $\widehat{A}(g)=\sum_{i=1}^{n}((2\pi_+(\pi_+-c_i)+\pi_--s_i)\ell(g(\bm{x}_i),+1))/2n$, $\widehat{B}(g)=\sum_{i=1}^{n}((2\pi_-(\pi_--c_i)+\pi_+-s_i)\ell(g(\bm{x}_i'),-1))/2n$, $\widehat{C}(g)=\sum_{i=1}^{n}((2\pi_+(\pi_++c_i)+\pi_--s_i)\ell(g(\bm{x}_i'),+1))/2n$, and $\widehat{D}(g)=\sum_{i=1}^{n}((2\pi_-(\pi_-+c_i)+\pi_+-s_i)\ell(g(\bm{x}_i),-1))/2n$. 
We assume that $f(z)$ is Lipschitz continuous and denote its Lipschitz constant by $L_f$.
By the proof of Theorem~\ref{thm:unbiased risk} in Appendix~\ref{appendix:proof}, since $\mathbb{E}[\widehat{A}(g)]$, $\mathbb{E}[\widehat{B}(g)]$, $\mathbb{E}[\widehat{C}(g)]$, and $\mathbb{E}[\widehat{D}(g)]$ are non-negative, there exist non-negative constants $a,b,c,d$ such that $\mathbb{E}[\widehat{A}(g)]\ge{a}$, $\mathbb{E}[\widehat{B}(g)]\ge{b}$, $\mathbb{E}[\widehat{C}(g)]\ge{c}$, and $\mathbb{E}[\widehat{D}(g)]\ge{d}$. 
Moreover, we define $\tilde{g}_{\mathrm{SCD}}=\arg\min_{g\in\mathcal{G}}\widetilde{R}_{\mathrm{SCD}}(g)$. Then, the following theorem provides the bias and consistency of $\widetilde{R}_{\mathrm{SCD}}(g)$. 

\begin{theorem}\label{thm:consistency}
Based on the assumptions above, the risk estimator $\widetilde{R}_{\mathrm{SCD}}(g)$ decays exponentially as $n\rightarrow{\infty}$:
\begin{align}
0\le\mathbb{E}[\widetilde{R}_{\mathrm{SCD}}(g)]-R(g)\le3(L_f+1)C_{\ell}\Delta,
\end{align}
where $\Delta=\exp(-\frac{4a^2n}{(4\pi_++1)^2C_{\ell}^2})+\exp(-\frac{4b^2n}{(4\pi_-+1)^2C_{\ell}^2})+\exp(-\frac{4c^2n}{(4\pi_++1)^2C_{\ell}^2})+\exp(-\frac{4d^2n}{(4\pi_-+1)^2C_{\ell}^2})$. Furthermore, the following inequality holds with probability at least $1-\delta$:
\begin{align}
\textstyle |\widetilde{R}_{\mathrm{SCD}}(g)-R(g)|\le6L_fC_{\ell}\sqrt{\frac{\log2/\delta}{2n}}+3(L_f+1)C_{\ell}\Delta.
\end{align}
\end{theorem}

Theorem~\ref{thm:consistency} indicates that $\widetilde{R}_{\mathrm{SCD}}(g)\rightarrow{R(g)}$ with probabilistic order $\mathcal{O}_p(1/\sqrt{n})$, implying that while $\widetilde{R}_{\mathrm{SCD}}(g)$ is no longer unbiased, it remains consistent. 
Next, we derive an estimation error bound for $\tilde{g}_{\mathrm{SCD}}$.

\begin{table*}[t]
\caption{
Classification accuracy on the benchmark test set with $\pi_+=0.2, 0.5$ averaged over five random seeds, with mean and standard deviation (mean$\pm$std). The highest score among the compared methods, excluding supervised learning, is shown in bold.
}
\label{tab:benchmark_acc_0.2}
\begin{center}
\begin{small}
\begin{tabular}{clcccc}
\toprule
Class Prior & \multicolumn{1}{c}{Method} & MNIST & Kuzushiji & Fashion & CIFAR-10 \\
\midrule
&SconfConfDiff-Unbiased & 0.894 $\pm$ 0.043 & 0.706 $\pm$ 0.016 & 0.851 $\pm$ 0.085 & 0.865 $\pm$ 0.006 \\
&SconfConfDiff-ReLU & 0.975 $\pm$ 0.002 & 0.890 $\pm$ 0.010 & 0.965 $\pm$ 0.002 & 0.871 $\pm$ 0.012 \\
&SconfConfDiff-ABS & \textbf{0.978 $\pm$ 0.002} & \textbf{0.906 $\pm$ 0.002} & \textbf{0.969 $\pm$ 0.003} & \textbf{0.884 $\pm$ 0.006} \\
\cmidrule(lr){2-6}
&ConfDiff-ABS & 0.975 $\pm$ 0.001 & 0.898 $\pm$ 0.006 & 0.967 $\pm$ 0.001 & 0.868 $\pm$ 0.007 \\
$\pi_+=0.2$ &Convex($\gamma=0.2$)-ABS & 0.975 $\pm$ 0.003 & 0.904 $\pm$ 0.007 & 0.967 $\pm$ 0.002 & 0.870 $\pm$ 0.005 \\
&Convex($\gamma=0.5$)-ABS & 0.976 $\pm$ 0.004 & 0.903 $\pm$ 0.006 & 0.965 $\pm$ 0.002 & 0.869 $\pm$ 0.006 \\
&Convex($\gamma=0.8$)-ABS & 0.972 $\pm$ 0.003 & 0.901 $\pm$ 0.005 & 0.962 $\pm$ 0.004 & 0.870 $\pm$ 0.008 \\
&Sconf-ABS & 0.960 $\pm$ 0.004 & 0.881 $\pm$ 0.006 & 0.955 $\pm$ 0.006 & 0.818 $\pm$ 0.022 \\
\cmidrule(lr){2-6}
&Supervised & 0.990 $\pm$ 0.000 & 0.936 $\pm$ 0.003 & 0.979 $\pm$ 0.002 & 0.895 $\pm$ 0.003 \\
\cmidrule(lr){1-6}
&SconfConfDiff-Unbiased & \textbf{0.977 $\pm$ 0.004} & \textbf{0.901 $\pm$ 0.005} & 0.963 $\pm$ 0.002 & 0.837 $\pm$ 0.004 \\
&SconfConfDiff-ReLU & \textbf{0.977 $\pm$ 0.004} & \textbf{0.901 $\pm$ 0.005} & 0.963 $\pm$ 0.002 & 0.837 $\pm$ 0.004 \\
$\pi_+=0.5$ &SconfConfDiff-ABS & \textbf{0.977 $\pm$ 0.004} & \textbf{0.901 $\pm$ 0.005} & 0.963 $\pm$ 0.002 & 0.837 $\pm$ 0.004 \\
\cmidrule(lr){2-6}
&ConfDiff-ABS & 0.965 $\pm$ 0.001 & 0.866 $\pm$ 0.004 & \textbf{0.968 $\pm$ 0.001} & \textbf{0.842 $\pm$ 0.002} \\
\cmidrule(lr){2-6}
&Supervised & 0.987 $\pm$ 0.001 & 0.928 $\pm$ 0.002 & 0.976 $\pm$ 0.001 & 0.876 $\pm$ 0.003 \\
\bottomrule
\end{tabular}
\end{small}
\end{center}
\vskip -0.1in
\end{table*}

\begin{table*}[t]
\caption{
Classification accuracy on the UCI test set with $\pi_+=0.2, 0.5$ averaged over five random seeds, with mean and standard deviation (mean$\pm$std). The highest score among the compared methods, excluding supervised learning, is shown in bold.
}
\label{tab:uci_acc_0.2}
\begin{center}
\begin{small}
\begin{tabular}{clcccc}
\toprule
Class Prior & \multicolumn{1}{c}{Method} & Optdigits & Pendigits & Letter & PMU-UD \\
\midrule
&SconfConfDiff-Unbiased & 0.917 $\pm$ 0.067 & 0.954 $\pm$ 0.041 & 0.746 $\pm$ 0.044 & 0.841 $\pm$ 0.058 \\
&SconfConfDiff-ReLU & 0.955 $\pm$ 0.023 & 0.987 $\pm$ 0.003 & 0.934 $\pm$ 0.007 & 0.966 $\pm$ 0.007 \\
&SconfConfDiff-ABS & \textbf{0.969 $\pm$ 0.005} & \textbf{0.991 $\pm$ 0.002} & \textbf{0.951 $\pm$ 0.004} & 0.975 $\pm$ 0.005 \\
\cmidrule(lr){2-6}
&ConfDiff-ABS & 0.964 $\pm$ 0.010 & 0.988 $\pm$ 0.005 & 0.945 $\pm$ 0.004 & 0.978 $\pm$ 0.008 \\
$\pi_+=0.2$ &Convex($\gamma=0.2$)-ABS & 0.969 $\pm$ 0.009 & 0.987 $\pm$ 0.004 & 0.947 $\pm$ 0.005 & \textbf{0.979 $\pm$ 0.006} \\
&Convex($\gamma=0.5$)-ABS & 0.969 $\pm$ 0.005 & 0.987 $\pm$ 0.004 & 0.946 $\pm$ 0.004 & 0.979 $\pm$ 0.008 \\
&Convex($\gamma=0.8$)-ABS & 0.966 $\pm$ 0.006 & 0.986 $\pm$ 0.004 & 0.941 $\pm$ 0.004 & 0.976 $\pm$ 0.004 \\
&Sconf-ABS & 0.944 $\pm$ 0.014 & 0.978 $\pm$ 0.006 & 0.915 $\pm$ 0.006 & 0.954 $\pm$ 0.010 \\
\cmidrule(lr){2-6}
&Supervised & 0.989 $\pm$ 0.003 & 0.997 $\pm$ 0.001 & 0.978 $\pm$ 0.004 & 0.994 $\pm$ 0.003 \\
\cmidrule(lr){1-6}
&SconfConfDiff-Unbiased & \textbf{0.971 $\pm$ 0.008} & \textbf{0.992 $\pm$ 0.002} & \textbf{0.935 $\pm$ 0.006} & \textbf{0.986 $\pm$ 0.005} \\
&SconfConfDiff-ReLU & \textbf{0.971 $\pm$ 0.008} & \textbf{0.992 $\pm$ 0.002} & \textbf{0.935 $\pm$ 0.006} & \textbf{0.986 $\pm$ 0.005} \\
$\pi_+=0.5$ &SconfConfDiff-ABS & \textbf{0.971 $\pm$ 0.008} & \textbf{0.992 $\pm$ 0.002} & \textbf{0.935 $\pm$ 0.006} & \textbf{0.986 $\pm$ 0.005} \\
\cmidrule(lr){2-6}
&ConfDiff-ABS & 0.964 $\pm$ 0.007 & 0.989 $\pm$ 0.001 & 0.922 $\pm$ 0.007 & 0.982 $\pm$ 0.005 \\
\cmidrule(lr){2-6}
&Supervised & 0.991 $\pm$ 0.003 & 0.996 $\pm$ 0.002 & 0.976 $\pm$ 0.002 & 0.993 $\pm$ 0.002 \\
\bottomrule
\end{tabular}
\end{small}
\end{center}
\vskip -0.1in
\end{table*}

\begin{theorem}\label{thm:corrected estimation error}
Under the above assumptions, for any $\delta>0$, the following inequality holds with probability at least $1-\delta$:
\begin{align}
\textstyle R(\tilde{g}_{\mathrm{SCD}})&-R(g^*)\le6(L_f+1)C_{\ell}\Delta+12L_{\ell}\mathfrak{R}_n(\mathcal{G})\notag\\
& \textstyle +(12L_f+2\pi_+^2+2\pi_-^2+5)C_{\ell}\sqrt{\frac{\log4/\delta}{2n}}.
\end{align}
\end{theorem}

From Theorem~\ref{thm:corrected estimation error}, we can observe that $R(\tilde{g}_{\mathrm{SCD}})\rightarrow{R(g^*)}$ as $n\rightarrow\infty$, because $\mathfrak{R}_n(\mathcal{G})=O(1/\sqrt{n})\rightarrow{0}$ as $n\rightarrow\infty$ for all parametric models with a bounded norm~\cite{sugiyama2022machine}. Furthermore, from Theorem~\ref{thm:consistency}, $\Delta$ associated with the risk correction approach decays exponentially as $n\rightarrow{\infty}$. 
Thus, as a whole, it is shown
that the estimation error bound converges at the rate $\mathcal{O}_p(1/\sqrt{n})$. 
This rate is known to be typical for empirical risk minimization over a norm-bounded, finite-dimensional hypothesis class, and is considered optimal in the absence of additional smoothness or margin assumptions~\cite{optimal_param}.

\begin{figure*}[t]
\vskip 0.2in
\begin{center}
\centerline{\includegraphics[width=1.0\linewidth]{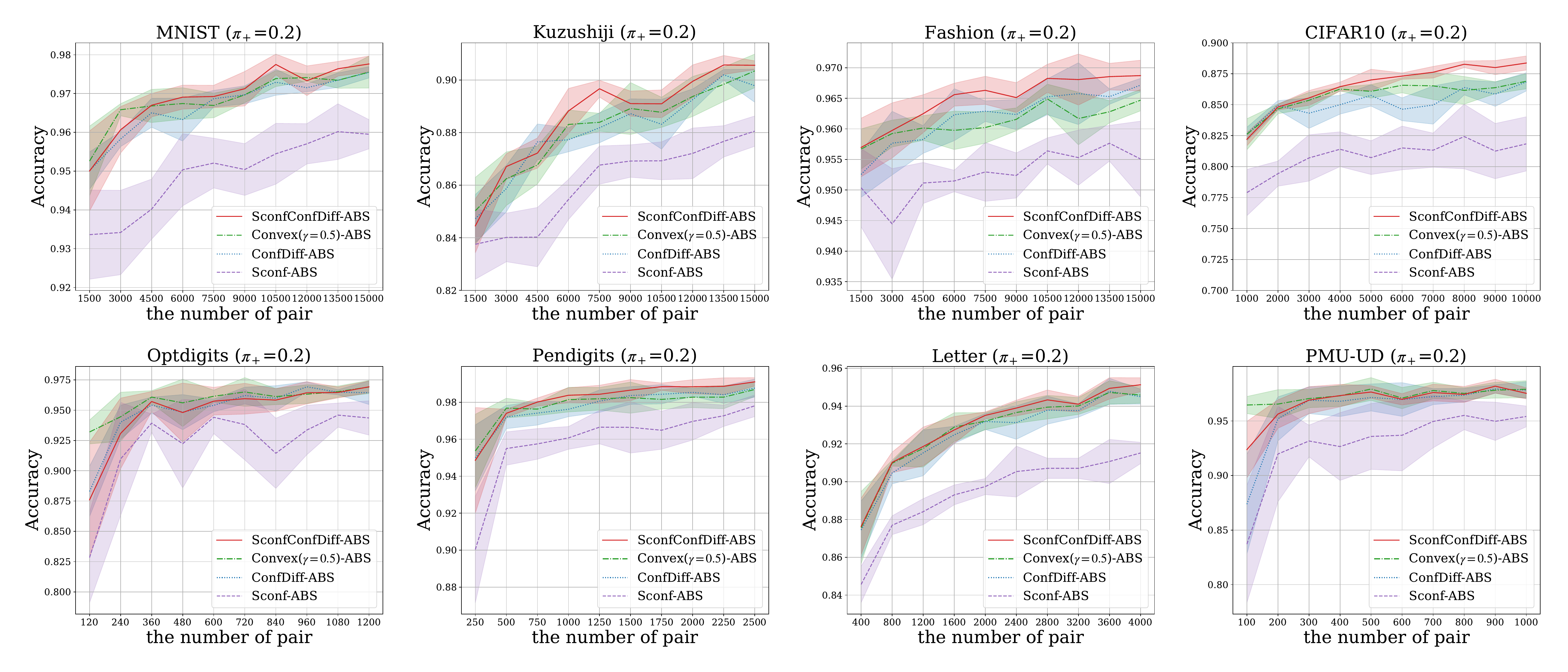}}
\caption{
Classification performance of SconfConfDiff-ABS, SconfConfDiff-Convex($\gamma=0.5$)-ABS, Sconf-ABS, and ConfDiff-ABS with varying numbers of training samples ($\pi_+=0.2$).
}
\label{fig:num_0.2}
\end{center}
\vskip -0.2in
\end{figure*}

\section{Numerical Experiments}\label{experiment}

In this section, we experimentally evaluate the superiority of the proposed method.

\subsection{Experimental Setup}\label{exp_setup}

Following the experimental setup of~\cite{confdiff}, we conducted experiments on the benchmark datasets MNIST~\cite{mnist}, Kuzushiji-MNIST~\cite{kmnist}, Fashion-MNIST~\cite{fashion}, and CIFAR-10~\cite{cifar10}. Additionally, we used four UCI datasets including Optdigits~\cite{optdigits}, Pendigits~\cite{pendigits}, Letter\cite{letter}, and PMU-UD~\cite{pmu-ud}. Since these datasets were originally designed for multi-class classification, we redesigned for binary classification. 
For CIFAR-10, ResNet-34~\cite{resnet} was used as the model architecture. For other datasets, we used a multilayer perceptron (MLP) with three hidden layers of width 300, ReLU~\cite{relu} activation functions and batch normalization~\cite{batch_norm}. We adopt the logistic loss for the binary loss function $\ell$.

We assess under various class prior probabilities ($\pi_+\in\{0.2, 0.5, 0.8\}$).
For all compared methods, we use the true class prior $\pi_+$. We repeated data sampling and training five times and recorded the mean and standard deviation of classification accuracy.

We evaluate the following variants of the proposed method: 1) SconfConfDiff-Unbiased (the unbiased risk estimator based on Eq.~\eqref{unbiased risk}), 2) SconfConfDiff-ReLU (the corrected risk estimator with the ReLU function), 3) SconfConfDiff-ABS (the corrected risk estimator with the absolute value function). Similarly, we conducted experiments on convex combination approach (unbiased) in Eq.~\eqref{unbiased risk sum}, the risk correction with the ReLU function, and the risk correction with the absolute value function. We compared the proposed approaches with Sconf learning and ConfDiff classification. Since SconfConfDiff-Convex and Sconf learning cannot be executed when $\pi_+=0.5$, we compared only ConfDiff classification in the case of $\pi_+=0.5$. 
Furthermore, we compared supervised learning using ground-truth hard labels. 

While Sconf and ConfDiff are provided by annotators, we synthetically generated them in this study to conduct comprehensive experiments. Details of the data generation process and hyperparameters are provided in Appendix~\ref{appendix_exp_setting}.

\begin{figure*}[t]
\vskip 0.2in
\begin{center}
\centerline{\includegraphics[width=1.0\linewidth]{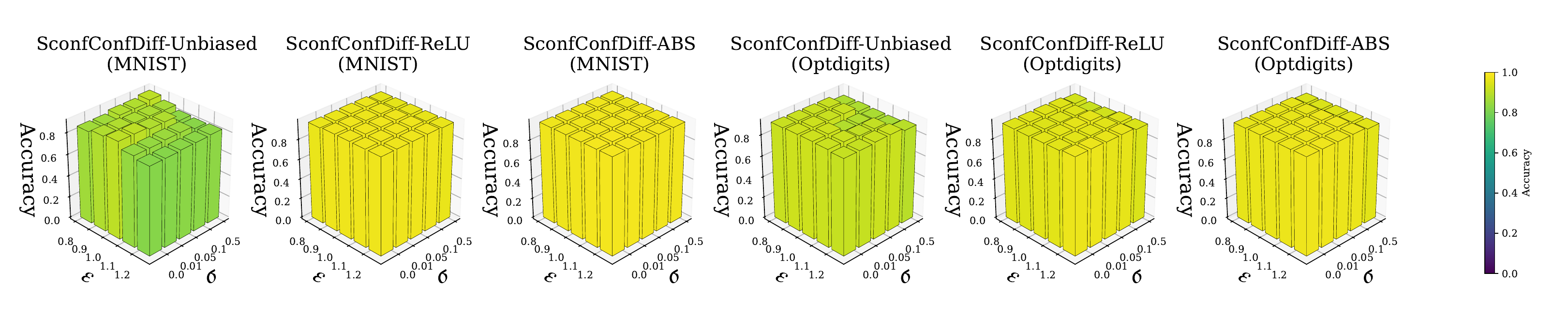}}
\caption{Classification performance on MNIST (left) and Optdigits (right) with $\pi_+=0.2$ given an noisy class prior probability, noisy similarity-confidence, and noisy confidence-difference. Both the height and color of the bars represent accuracy.}
\label{fig:robust_result}
\end{center}
\vskip -0.2in
\end{figure*}

\subsection{Experimental Results}\label{exp:acc}

Based on the experimental setting in Section~\ref{exp_setup}, we verify the proposed method derived in Section~\ref{proposed}. The experimental results demonstrated that the proposed method outperforms the comparison methods on many datasets, thereby validating its effectiveness. In the tables, SconfConfDiff-Convex is denoted as Convex($\gamma = \text{value}$), with varying values of $\gamma$.

Table~\ref{tab:benchmark_acc_0.2} shows the classification accuracy results of all compared methods on four benchmark datasets. 
When $\pi_+=0.2$, the proposed approach, either SconfConfDiff-ABS or SconfConfDiff-Convex-ABS, achieves superior performance compared to other methods. In addition, SconfConfDiff-ABS has smaller variances than SconfConfDiff-Unbiased and SconfConfDiff-ReLU. 
This result indicates that the risk correction improves learning performance of SconfConfDiff classification. 
In the case of $\pi_+=0.5$, ConfDiff-ABS achieves high accuracy on some datasets. This is likely because, when there is no difference in the class prior probability, the similarity-confidence information does not contribute significantly to the learning process.

Tables~\ref{tab:uci_acc_0.2} reports experimental results of classification accuracy for all the compared methods on four UCI datasets. 
It can be observed that for all UCI datasets with different class prior probabilities, the proposed approach, either SconfConfDiff-ABS or SconfConfDiff-Convex-ABS, achieves the best performance compared to other methods.

When $\pi_+ = 0.5$, the risk estimator assigned equal weight to the positive and negative class terms. This balance helped stabilize the overall estimate, making it less likely for the risk to take negative values. As a result, applying a risk correction function such as ReLU or the absolute value produced identical experimental results. Note that in Tables~\ref{tab:benchmark_acc_0.2} and~\ref{tab:uci_acc_0.2}, only the absolute value correction results are shown for Sconf learning, ConfDiff classification, and SconfConfDiff-Convex. Detailed results, including those for the unbiased risk estimator and 
the ReLU-based corrected risk estimator, as well as the results in the case of $\pi_+=0.8$, are provided in Appendix~\ref{appendix_exp}.

\subsection{Experimental Results with Limited Data}\label{exp:few data}

We also investigated the classification performance by varying the proportion of training data for SconfConfDiff-ABS and SconfConfDiff-Convex-ABS. For comparison, we provide the results of Sconf learning and ConfDiff classification in Fig.~\ref{fig:num_0.2}. Figure.~\ref{fig:num_0.2} shows the results for $\pi_+=0.2$, 
and additional experimental results are provided in the Appendix~\ref{appendix_exp}.
It can be observed that providing two types of weak supervision leads to improved classification performance.

\subsection{Experimental Results with Noisy Data}\label{exp:robust}

In this section, we experimentally investigate the impact of a noisy class prior probability, noisy similarity-confidence, and noisy confidence-difference on the classification performance of SconfConfDiff classification. 
Specifically, let $\epsilon$ be a real number close to 1 and define $\bar{\pi}_+=\epsilon\pi_+$ as an noisy class prior probability. we also define $\bar{s}_i=\epsilon_i's_i$ and $\bar{c}_i=\epsilon_i''c_i$ as the noisy similarity-confidence and the noisy confidence-difference, respectively, where $\epsilon_i'$ and $\epsilon_i''$ are sampled from a normal distribution ${\mathcal{N}(1,\sigma^2)}$.
Figure.~\ref{fig:robust_result} presents the classification accuracy (only mean accuracy) of SconfConfDiff classification on MNIST and Optdigits ($\pi_+=0.2$) with different $\epsilon$ and $\sigma$. From Fig.~\ref{fig:robust_result}, it can be observed that the proposed method is robust to noise of class prior probability, similarity-confidence, and confidence-difference. More experimental results (including the standard deviation of accuracy) can be found in Appendix~\ref{appendix_exp}.

\section{Conclusion}
In this paper, we proposed a novel learning framework that integrates two types of weak supervision. Specifically, we considered a setting where unlabeled data pairs are annotated with both similarity-confidence and confidence-difference, and derived two unbiased risk estimators: SconfConfDiff-Convex classification, which combines the risks from Sconf learning and ConfDiff classification via a convex formulation, and SconfConfDiff classification, which integrates the two types of supervision in a unified manner to reflect their interaction. To mitigate overfitting, we incorporate a risk correction approach that improves learning performance. Furthermore, we provided estimation error bounds for the classifier obtained by empirical risk minimization and proved its statistical consistency. Experimental results demonstrate that our method achieves performance comparable to or better than learning with a single weak supervision. Future directions include extending the proposed framework to handle other forms of weak supervision, scenarios with partially missing labels, and a learning framework that effectively leverages both ordinary supervised labels and weakly labels. Since our current method is limited to binary classification, extending it to multi-class problems is also a promising direction for future work.

\section*{Impact Statement}
This paper contributes to the theoretical foundations of machine learning by analyzing the statistical properties of learning from multiple weak supervision sources and by establishing rigorous guarantees for the proposed method.
As the contribution focuses on fundamental aspects of learning, this work does not introduce application-specific risks beyond those commonly associated with machine learning research. Any potential societal impact therefore depends on downstream applications and deployment choices.

\section*{Acknowledgements}
This work was supported in part by the Japan Society for the
 Promotion of Science through Grants-in-Aid for Scientific 
Research (C) (23K11111).

\bibliography{references}

@misc{SD-Pcomp,
      title={Learning from Similarity/Dissimilarity and Pairwise Comparison}, 
      author={Tomoya Tate and Kosuke Sugiyama and Masato Uchida},
      year={2026},
      eprint={2603.19713},
      archivePrefix={arXiv},
      primaryClass={cs.LG},
      url={https://arxiv.org/abs/2603.19713}, 
}

@incollection{McDiarmid, 
  place={Cambridge}, 
  series={London Mathematical Society Lecture Note Series}, 
  title={On the method of bounded differences}, 
  booktitle={Surveys in Combinatorics, 1989: Invited Papers at the Twelfth British Combinatorial Conference}, 
  publisher={Cambridge University Press}, 
  author={McDiarmid, Colin}, 
  editor={Siemons, J.Editor}, 
  year={1989}, 
  pages={148–-188}, 
  collection={London Mathematical Society Lecture Note Series}
}

@book{Talagrand,
  author    = {Michel Ledoux and Michel Talagrand},
  title     = {Probability in Banach Spaces: Isoperimetry and Processes},
  publisher = {Springer Berlin Heidelberg},
  year      = {1991}
}

@InProceedings{sconf,
  title = 	 {Learning from Similarity-Confidence Data},
  author =       {Cao, Yuzhou and Feng, Lei and Xu, Yitian and An, Bo and Niu, Gang and Sugiyama, Masashi},
  booktitle = 	 {Proceedings of the 38th International Conference on Machine Learning},
  pages = 	 {1272--1282},
  year = 	 {2021},
  publisher =    {PMLR}
}

@inproceedings{confdiff,
 author = {Wang, Wei and Feng, Lei and Jiang, Yuchen and Niu, Gang and Zhang, Min-Ling and Sugiyama, Masashi},
 booktitle = {Advances in Neural Information Processing Systems},
 pages = {5936--5960},
 publisher = {Curran Associates, Inc.},
 title = {Binary Classification with Confidence Difference},
 year = {2023}
}

@inproceedings{PU_analysis,
author = {Plessis, Marthinus C. du and Niu, Gang and Sugiyama, Masashi},
title = {Analysis of learning from positive and unlabeled data},
year = {2014},
booktitle = {Proceedings of the 28th International Conference on Neural Information Processing Systems - Volume 1},
pages = {703--711}
}

@inproceedings{PU-learning,
author = {Elkan, Charles and Noto, Keith},
title = {Learning classifiers from only positive and unlabeled data},
year = {2008},
publisher = {Association for Computing Machinery},
booktitle = {Proceedings of the 14th ACM SIGKDD International Conference on Knowledge Discovery and Data Mining},
pages = {213--220}
}

@inproceedings{PU_NN_risk,
 author = {Kiryo, Ryuichi and Niu, Gang and du Plessis, Marthinus C and Sugiyama, Masashi},
 booktitle = {Advances in Neural Information Processing Systems},
 pages = {1--30},
 publisher = {Curran Associates, Inc.},
 title = {Positive-Unlabeled Learning with Non-Negative Risk Estimator},
 year = {2017}
}

@InProceedings{PU_convex,
  title = 	 {Convex Formulation for Learning from Positive and Unlabeled Data},
  author = 	 {Plessis, Marthinus Du and Niu, Gang and Sugiyama, Masashi},
  booktitle = 	 {Proceedings of the 32nd International Conference on Machine Learning},
  pages = 	 {1386--1394},
  year = 	 {2015},
  publisher =    {PMLR}
}

@InProceedings{PNU-learning,
  title = 	 {Semi-Supervised Classification Based on Classification from Positive and Unlabeled Data},
  author =       {Tomoya Sakai and Marthinus Christoffel du Plessis and Gang Niu and Masashi Sugiyama},
  booktitle = 	 {Proceedings of the 34th International Conference on Machine Learning},
  pages = 	 {2998--3006},
  year = 	 {2017},
  publisher =    {PMLR}
}

@InProceedings{SU,
  title = 	 {Classification from Pairwise Similarity and Unlabeled Data},
  author =       {Bao, Han and Niu, Gang and Sugiyama, Masashi},
  booktitle = 	 {Proceedings of the 35th International Conference on Machine Learning},
  pages = 	 {452--461},
  year = 	 {2018},
  publisher =    {PMLR}
}

@ARTICLE{SD,
  author={Shimada, Takuya and Bao, Han and Sato, Issei and Sugiyama, Masashi},
  journal={Neural Computation}, 
  title={Classification From Pairwise Similarities/Dissimilarities and Unlabeled Data via Empirical Risk Minimization}, 
  year={2021},
  volume={33},
  number={5},
  pages={1234-1268}
}

@InProceedings{pcomp,
  title = 	 {Pointwise Binary Classification with Pairwise Confidence Comparisons},
  author =       {Feng, Lei and Shu, Senlin and Lu, Nan and Han, Bo and Xu, Miao and Niu, Gang and An, Bo and Sugiyama, Masashi},
  booktitle = 	 {Proceedings of the 38th International Conference on Machine Learning},
  pages = 	 {3252--3262},
  year = 	 {2021},
  publisher =    {PMLR}
}

@ARTICLE{mnist,
  author={Lecun, Y. and Bottou, L. and Bengio, Y. and Haffner, P.},
  journal={Proceedings of the IEEE}, 
  title={Gradient-based learning applied to document recognition}, 
  year={1998},
  volume={86},
  number={11},
  pages={2278-2324}
}

@misc{kmnist,
author = {Clanuwat, Tarin and Bober-Irizar, Mikel and Kitamoto, Asanobu and Lamb, Alex and Kazuaki, Yamamoto and Ha, David},
year = {2018},
pages = {1--8},
title = {Deep Learning for Classical Japanese Literature},
doi = {10.20676/00000341}
}

@misc{fashion,
      title={Fashion-MNIST: a Novel Image Dataset for Benchmarking Machine Learning Algorithms}, 
      author={Han Xiao and Kashif Rasul and Roland Vollgraf},
      year={2017},
      eprint={1708.07747},
      archivePrefix={arXiv},
      primaryClass={cs.LG},
      url={https://arxiv.org/abs/1708.07747}, 
}

@techreport{cifar10,
  author={Alex Krizhevsky},
  title       = {Learning Multiple Layers of Features from Tiny Images},
  institution = {University of Toronto},
  year        = {2009}
}

@INPROCEEDINGS{resnet,
  author={He, Kaiming and Zhang, Xiangyu and Ren, Shaoqing and Sun, Jian},
  booktitle={2016 IEEE Conference on Computer Vision and Pattern Recognition}, 
  title={Deep Residual Learning for Image Recognition}, 
  year={2016},
  pages={770-778}
}

@inproceedings{relu,
author = {Nair, Vinod and Hinton, Geoffrey E.},
title = {Rectified linear units improve restricted boltzmann machines},
year = {2010},
publisher = {Omnipress},
booktitle = {Proceedings of the 27th International Conference on International Conference on Machine Learning},
pages = {807-814}
}

@InProceedings{batch_norm,
  title = 	 {Batch Normalization: Accelerating Deep Network Training by Reducing Internal Covariate Shift},
  author = 	 {Ioffe, Sergey and Szegedy, Christian},
  booktitle = 	 {Proceedings of the 32nd International Conference on Machine Learning},
  pages = 	 {448--456},
  year = 	 {2015},
  publisher =    {PMLR}
}

@inproceedings{PyTorch,
 author = {Paszke, Adam and Gross, Sam and Massa, Francisco and Lerer, Adam and Bradbury, James and Chanan, Gregory and Killeen, Trevor and Lin, Zeming and Gimelshein, Natalia and Antiga, Luca and Desmaison, Alban and Kopf, Andreas and Yang, Edward and DeVito, Zachary and Raison, Martin and Tejani, Alykhan and Chilamkurthy, Sasank and Steiner, Benoit and Fang, Lu and Bai, Junjie and Chintala, Soumith},
 booktitle = {Advances in Neural Information Processing Systems},
 pages = {1--32},
 publisher = {Curran Associates, Inc.},
 title = {PyTorch: An Imperative Style, High-Performance Deep Learning Library},
 year = {2019}
}

@misc{Adam,
      title={Adam: A Method for Stochastic Optimization}, 
      author={Diederik P. Kingma and Jimmy Ba},
      year={2017},
      eprint={1412.6980},
      archivePrefix={arXiv},
      primaryClass={cs.LG},
      url={https://arxiv.org/abs/1412.6980}, 
}

@book{sugiyama2022machine,
  author    = {Sugiyama, Masashi and Bao, Han and Ishida, Takashi and Lu, Nan and Sakai, Taiji and Niu, Gang},
  title     = {Machine Learning from Weak Supervision: An Empirical Risk Minimization Approach},
  year      = {2022},
  publisher = {MIT Press}
}

@inproceedings{partial,
 author = {Feng, Lei and Lv, Jiaqi and Han, Bo and Xu, Miao and Niu, Gang and Geng, Xin and An, Bo and Sugiyama, Masashi},
 booktitle = {Advances in Neural Information Processing Systems},
 pages = {10948--10960},
 publisher = {Curran Associates, Inc.},
 title = {Provably Consistent Partial-Label Learning},
 year = {2020}
}

@inproceedings{UU,
title={On the Minimal Supervision for Training Any Binary Classifier from Only Unlabeled Data},
author={Nan Lu and Gang Niu and Aditya K. Menon and Masashi Sugiyama},
booktitle={International Conference on Learning Representations},
pages = {1--18},
year={2019}
}

@InProceedings{multi-complementary,
  title = 	 {Learning with Multiple Complementary Labels},
  author =       {Feng, Lei and Kaneko, Takuo and Han, Bo and Niu, Gang and An, Bo and Sugiyama, Masashi},
  booktitle = 	 {Proceedings of the 37th International Conference on Machine Learning},
  pages = 	 {3072--3081},
  year = 	 {2020},
  publisher =    {PMLR}
}

@inproceedings{complementary,
 author = {Ishida, Takashi and Niu, Gang and Hu, Weihua and Sugiyama, Masashi},
 booktitle = {Advances in Neural Information Processing Systems},
 pages = {1--30},
 publisher = {Curran Associates, Inc.},
 title = {Learning from Complementary Labels},
 year = {2017}
}

@InProceedings{Rademacher_bound,
  title = 	 {Size-Independent  Sample Complexity of Neural Networks},
  author =       {Golowich, Noah and Rakhlin, Alexander and Shamir, Ohad},
  booktitle = 	 {Proceedings of the 31st  Conference On Learning Theory},
  pages = 	 {297--299},
  year = 	 {2018},
  publisher =    {PMLR}
}

@ARTICLE{optimal_param,
  author={Mendelson, Shahar},
  journal={IEEE Transactions on Information Theory}, 
  title={Lower Bounds for the Empirical Minimization Algorithm}, 
  year={2008},
  volume={54},
  number={8},
  pages={3797-3803}
}

@misc{optdigits,
  author       = {Alpaydin, E. and Kaynak, C.},
  title        = {{Optical Recognition of Handwritten Digits}},
  year         = {1998},
  howpublished = {UCI Machine Learning Repository}
}

@misc{pendigits,
  author       = {Alpaydin, E. and Alimoglu, Fevzi.},
  title        = {{Pen-Based Recognition of Handwritten Digits}},
  year         = {1996},
  howpublished = {UCI Machine Learning Repository}
}

@misc{letter,
  author       = {Slate, David},
  title        = {{Letter Recognition}},
  year         = {1991},
  howpublished = {UCI Machine Learning Repository}
}

@misc{pmu-ud,
  author       = {Latif, Ghazanfar},
  title        = {{PMU-UD}},
  year         = {2018},
  howpublished = {UCI Machine Learning Repository}
}

@book{mohri,
author = {Mohri, Mehryar and Rostamizadeh, Afshin and Talwalkar, Ameet},
title = {Foundations of Machine Learning},
year = {2012},
publisher = {The MIT Press}
}
\bibliographystyle{anonymousconference}

\newpage
\appendix
\onecolumn

\section{Proof of Section~\ref{proposed}}\label{appendix:proof}
\subsection{Proof of Theorem~\ref{thm:unbiased risk}}
\begin{lemma}[cf. Lemma 3 in~\cite{confdiff}]\label{lemma:confdiff}
The following equations hold:
\begin{align}
&\mathbb{E}_{p(\bm{x},\bm{x}')}[(\pi_+-c(\bm{x},\bm{x}'))\ell(g(\bm{x}),+1)]=\pi_+\mathbb{E}_{p(\bm{x} \mid y=+1)}[\ell(g(\bm{x}),+1)],\label{confdiff-1}\\
&\mathbb{E}_{p(\bm{x},\bm{x}')}[(\pi_-+c(\bm{x},\bm{x}'))\ell(g(\bm{x}),-1)]=\pi_-\mathbb{E}_{p(\bm{x} \mid y=-1)}[\ell(g(\bm{x}),-1)],\label{confdiff-2}\\
&\mathbb{E}_{p(\bm{x},\bm{x}')}[(\pi_++c(\bm{x},\bm{x}'))\ell(g(\bm{x}'),+1)]=\pi_+\mathbb{E}_{p(\bm{x}' \mid y=+1)}[\ell(g(\bm{x}'),+1)],\label{confdiff-3}\\
&\mathbb{E}_{p(\bm{x},\bm{x}')}[(\pi_--c(\bm{x},\bm{x}'))\ell(g(\bm{x}'),-1)]=\pi_-\mathbb{E}_{p(\bm{x}' \mid y=-1)}[\ell(g(\bm{x}'),-1)].\label{confdiff-4}
\end{align}
\end{lemma}
\begin{lemma}[cf. Lemma 2 in~\cite{sconf}]\label{lemma:sconf}
The following equations hold:
\begin{align}
&\mathbb{E}_{p(\bm{x},\bm{x}')}[s(\bm{x},\bm{x}')\ell(g(\bm{x}),+1)]-\frac{\pi_-}{\pi_+}\mathbb{E}_{p(\bm{x},\bm{x}')}[(1-s(\bm{x},\bm{x}'))\ell(g(\bm{x}),+1)]\notag\\
&=(\pi_+-\pi_-)\mathbb{E}_{p(\bm{x} \mid y=+1)}[\ell(g(\bm{x}),+1)].\label{sconf-1}
\end{align}
\end{lemma}
\begin{lemma}\label{lemma of thm1}
The following equations hold:
\begin{align}
\pi_+\mathbb{E}_{p(\bm{x} \mid y=+1)}[\ell(g(\bm{x}),+1)]=&\mathbb{E}_{p(\bm{x},\bm{x}')}[(2\pi_+(\pi_+-c(\bm{x},\bm{x}'))+\pi_--s(\bm{x},\bm{x}'))\ell(g(\bm{x}),+1)],\\
\pi_-\mathbb{E}_{p(\bm{x} \mid y=-1)}[\ell(g(\bm{x}),-1)]=&\mathbb{E}_{p(\bm{x},\bm{x}')}[(2\pi_-(\pi_-+c(\bm{x},\bm{x}'))+\pi_+-s(\bm{x},\bm{x}'))\ell(g(\bm{x}),-1)],\\
\pi_+\mathbb{E}_{p(\bm{x}' \mid y=+1)}[\ell(g(\bm{x}'),+1)]=&\mathbb{E}_{p(\bm{x},\bm{x}')}[(2\pi_+(\pi_++c(\bm{x},\bm{x}'))+\pi_--s(\bm{x},\bm{x}'))\ell(g(\bm{x}'),+1)],\\
\pi_-\mathbb{E}_{p(\bm{x}' \mid y=-1)}[\ell(g(\bm{x}'),-1)]=&\mathbb{E}_{p(\bm{x},\bm{x}')}[(2\pi_-(\pi_--c(\bm{x},\bm{x}'))+\pi_+-s(\bm{x},\bm{x}'))\ell(g(\bm{x}'),-1)].
\end{align}
\end{lemma}
\begin{proof}
Based on Lemma~\ref{lemma:sconf}, 
\begin{align}
\mathbb{E}&_{p(\bm{x},\bm{x}')}[s(\bm{x},\bm{x}')\ell(g(\bm{x}),+1)]-\frac{\pi_-}{\pi_+}\mathbb{E}_{p(\bm{x},\bm{x}')}[(1-s(\bm{x},\bm{x}'))\ell(g(\bm{x}),+1)]\notag\\
=&2\pi_+\mathbb{E}_{p(\bm{x} \mid y=+1)}[\ell(g(\bm{x}),+1)]-\mathbb{E}_{p(\bm{x} \mid y=+1)}[\ell(g(\bm{x}),+1)]\notag\\
=&2\mathbb{E}_{p(\bm{x},\bm{x}')}[(\pi_+-c(\bm{x},\bm{x}'))\ell(g(\bm{x}),+1)]-\mathbb{E}_{p(\bm{x} \mid y=+1)}[\ell(g(\bm{x}),+1)].\label{a-1}
\end{align}
By multiplying both sides of Eq.~\eqref{a-1} by $\pi_+$ and solving for $\pi_+\mathbb{E}_{p(\bm{x} \mid y=+1)}[\ell(g(\bm{x}),+1)]$, we obtain
\begin{align}
\pi&_+\mathbb{E}_{p(\bm{x} \mid y=+1)}[\ell(g(\bm{x}),+1)]\notag\\
=&2\pi_+\mathbb{E}_{p(\bm{x},\bm{x}')}[(\pi_+-c(\bm{x},\bm{x}'))\ell(g(\bm{x}),+1)]-\pi_+\mathbb{E}_{p(\bm{x},\bm{x}')}[s(\bm{x},\bm{x}')\ell(g(\bm{x}),+1)]\notag\\
&\hspace{1.5cm}+\pi_-\mathbb{E}_{p(\bm{x},\bm{x}')}[(1-s(\bm{x},\bm{x}'))\ell(g(\bm{x}),+1)]\notag\\
=&\mathbb{E}_{p(\bm{x},\bm{x}')}[2\pi_+(\pi_+-c(\bm{x},\bm{x}'))\ell(g(\bm{x}),+1)-\pi_+s(\bm{x},\bm{x}')\ell(g(\bm{x}),+1)\notag\\
&\hspace{1.5cm}+\pi_-(1-s(\bm{x},\bm{x}'))\ell(g(\bm{x}),+1)]\notag\\
=&\mathbb{E}_{p(\bm{x},\bm{x}')}[(2\pi_+(\pi_+-c(\bm{x},\bm{x}'))+\pi_--s(\bm{x},\bm{x}'))\ell(g(\bm{x}),+1)].\notag
\end{align}
Similarly, 
\begin{align}
\mathbb{E}&_{p(\bm{x},\bm{x}')}[s(\bm{x},\bm{x}')\ell(g(\bm{x}),-1)]-\frac{\pi_+}{\pi_-}\mathbb{E}_{p(\bm{x},\bm{x}')}[(1-s(\bm{x},\bm{x}'))\ell(g(\bm{x}),-1)]\notag\\
=&(\pi_--\pi_+)\mathbb{E}_{p(\bm{x} \mid y=-1)}[\ell(g(\bm{x}),+1)]\notag\\
=&2\pi_-\mathbb{E}_{p(\bm{x} \mid y=-1)}[\ell(g(\bm{x}),-1)]-\mathbb{E}_{p(\bm{x} \mid y=-1)}[\ell(g(\bm{x}),-1)]\notag\\
=&2\mathbb{E}_{p(\bm{x},\bm{x}')}[(\pi_-+c(\bm{x},\bm{x}'))\ell(g(\bm{x}),-1)]-\mathbb{E}_{p(\bm{x} \mid y=-1)}[\ell(g(\bm{x}),-1)].\label{a-3}
\end{align}
By multiplying both sides of Eq.~\eqref{a-3} by $\pi_-$ and solving for $\pi_-\mathbb{E}_{p(\bm{x} \mid y=-1)}[\ell(g(\bm{x}),-1)]$, we obtain
\begin{align}
\pi&_-\mathbb{E}_{p(\bm{x} \mid y=-1)}[\ell(g(\bm{x}),-1)]\notag\\
=&2\pi_-\mathbb{E}_{p(\bm{x},\bm{x}')}[(\pi_-+c(\bm{x},\bm{x}'))\ell(g(\bm{x}),-1)]-\pi_-\mathbb{E}_{p(\bm{x},\bm{x}')}[s(\bm{x},\bm{x}')\ell(g(\bm{x}),-1)]\notag\\
&\hspace{1.5cm}+\pi_+\mathbb{E}_{p(\bm{x},\bm{x}')}[(1-s(\bm{x},\bm{x}'))\ell(g(\bm{x}),-1)]\notag\\
=&\mathbb{E}_{p(\bm{x},\bm{x}')}[2\pi_-(\pi_-+c(\bm{x},\bm{x}'))\ell(g(\bm{x}),-1)-\pi_-s(\bm{x},\bm{x}')\ell(g(\bm{x}),-1)\notag\\
&\hspace{1.5cm}+\pi_+(1-s(\bm{x},\bm{x}'))\ell(g(\bm{x}),-1)]\notag\\
=&\mathbb{E}_{p(\bm{x},\bm{x}')}[(2\pi_-(\pi_-+c(\bm{x},\bm{x}'))+\pi_+-s(\bm{x},\bm{x}'))\ell(g(\bm{x}),-1)].\notag
\end{align}
Similarly for $\bm{x}'$, 
\begin{align}
\mathbb{E}&_{p(\bm{x},\bm{x}')}[s(\bm{x},\bm{x}')\ell(g(\bm{x}'),+1)]-\frac{\pi_-}{\pi_+}\mathbb{E}_{p(\bm{x},\bm{x}')}[(1-s(\bm{x},\bm{x}'))\ell(g(\bm{x}'),+1)]\notag\\
=&(\pi_+-\pi_-)\mathbb{E}_{p(\bm{x}' \mid y=+1)}[\ell(g(\bm{x}'),+1)]\notag\\
=&2\pi_+\mathbb{E}_{p(\bm{x}' \mid y=+1)}[\ell(g(\bm{x}'),+1)]-\mathbb{E}_{p(\bm{x}' \mid y=+1)}[\ell(g(\bm{x}'),+1)]\notag\\
=&2\mathbb{E}_{p(\bm{x},\bm{x}')}[(\pi_++c(\bm{x},\bm{x}'))\ell(g(\bm{x}'),+1)]-\mathbb{E}_{p(\bm{x}' \mid y=+1)}[\ell(g(\bm{x}'),+1)].\label{a-4}
\end{align}
By multiplying both sides of Eq.~\eqref{a-4} by $\pi_+$ and solving for $\pi_+\mathbb{E}_{p(\bm{x}' \mid y=+1)}[\ell(g(\bm{x}'),+1)]$, we obtain
\begin{align}
\pi&_+\mathbb{E}_{p(\bm{x}' \mid y=+1)}[\ell(g(\bm{x}'),+1)]\notag\\
=&2\pi_+\mathbb{E}_{p(\bm{x},\bm{x}')}[(\pi_++c(\bm{x},\bm{x}'))\ell(g(\bm{x}'),+1)]-\pi_+\mathbb{E}_{p(\bm{x},\bm{x}')}[s(\bm{x},\bm{x}')\ell(g(\bm{x}'),+1)]\notag\\
&\hspace{1.5cm}+\pi_-\mathbb{E}_{p(\bm{x},\bm{x}')}[(1-s(\bm{x},\bm{x}'))\ell(g(\bm{x}'),+1)]\notag\\
=&\mathbb{E}_{p(\bm{x},\bm{x}')}[2\pi_+(\pi_++c(\bm{x},\bm{x}'))\ell(g(\bm{x}'),+1)-\pi_+s(\bm{x},\bm{x}')\ell(g(\bm{x}'),+1)\notag\\
&\hspace{1.5cm}+\pi_-(1-s(\bm{x},\bm{x}'))\ell(g(\bm{x}'),+1)]\notag\\
=&\mathbb{E}_{p(\bm{x},\bm{x}')}[(2\pi_+(\pi_++c(\bm{x},\bm{x}'))+\pi_--s(\bm{x},\bm{x}'))\ell(g(\bm{x}'),+1)].\notag
\end{align}
Similarly, 
\begin{align}
\mathbb{E}&_{p(\bm{x},\bm{x}')}[s(\bm{x},\bm{x}')\ell(g(\bm{x}'),-1)]-\frac{\pi_+}{\pi_-}\mathbb{E}_{p(\bm{x},\bm{x}')}[(1-s(\bm{x},\bm{x}'))\ell(g(\bm{x}'),-1)]\notag\\
=&(\pi_--\pi_+)\mathbb{E}_{p(\bm{x}' \mid y=-1)}[\ell(g(\bm{x}'),-1)]\notag\\
=&2\pi_-\mathbb{E}_{p(\bm{x}' \mid y=-1)}[\ell(g(\bm{x}'),-1)]-\mathbb{E}_{p(\bm{x}' \mid y=-1)}[\ell(g(\bm{x}'),-1)]\notag\\
=&2\mathbb{E}_{p(\bm{x},\bm{x}')}[(\pi_--c(\bm{x},\bm{x}'))\ell(g(\bm{x}'),-1)]-\mathbb{E}_{p(\bm{x}' \mid y=-1)}[\ell(g(\bm{x}'),-1)].\label{a-6}
\end{align}
By multiplying both sides of Eq.~\eqref{a-6} by $\pi_-$ and solving for $\pi_-\mathbb{E}_{p(\bm{x}' \mid y=-1)}[\ell(g(\bm{x}'),-1)]$, we obtain
\begin{align}
\pi&_-\mathbb{E}_{p(\bm{x}' \mid y=-1)}[\ell(g(\bm{x}'),-1)]\notag\\
=&2\pi_-\mathbb{E}_{p(\bm{x},\bm{x}')}[(\pi_--c(\bm{x},\bm{x}'))\ell(g(\bm{x}'),-1)]-\pi_-\mathbb{E}_{p(\bm{x},\bm{x}')}[s(\bm{x},\bm{x}')\ell(g(\bm{x}'),-1)]\notag\\
&\hspace{1.5cm}+\pi_+\mathbb{E}_{p(\bm{x},\bm{x}')}[(1-s(\bm{x},\bm{x}'))\ell(g(\bm{x}'),-1)]\notag\\
=&\mathbb{E}_{p(\bm{x},\bm{x}')}[2\pi_-(\pi_--c(\bm{x},\bm{x}'))\ell(g(\bm{x}'),-1)-\pi_-s(\bm{x},\bm{x}')\ell(g(\bm{x}'),-1)\notag\\
&\hspace{1.5cm}+\pi_+(1-s(\bm{x},\bm{x}'))\ell(g(\bm{x}'),-1)]\notag\\
=&\mathbb{E}_{p(\bm{x},\bm{x}')}[(2\pi_-(\pi_--c(\bm{x},\bm{x}'))+\pi_+-s(\bm{x},\bm{x}'))\ell(g(\bm{x}'),-1)].\notag
\end{align}
\end{proof}
Then, we prove Theorem~\ref{thm:unbiased risk}.
\begin{proof}
Firstly, 
\begin{align*}
\mathbb{E}_{p(\bm{x} \mid y=+1)}[\ell(g(\bm{x}),+1)]=\mathbb{E}_{p(\bm{x}' \mid y=+1)}[\ell(g(\bm{x}'),+1)], \\
\mathbb{E}_{p(\bm{x} \mid y=-1)}[\ell(g(\bm{x}),-1)]=\mathbb{E}_{p(\bm{x}' \mid y=-1)}[\ell(g(\bm{x}'),-1)].
\end{align*}
Based on Lemma~\ref{lemma of thm1}, 
\begin{align*}
\mathbb{E}&_{p(\bm{x},\bm{x}')}[(2\pi_+(\pi_+-c(\bm{x},\bm{x}'))+\pi_--s(\bm{x},\bm{x}'))\ell(g(\bm{x}),+1)\\
&\hspace{1cm}+(2\pi_-(\pi_-+c(\bm{x},\bm{x}'))+\pi_+-s(\bm{x},\bm{x}'))\ell(g(\bm{x}),-1)\\
&\hspace{1cm}+(2\pi_+(\pi_++c(\bm{x},\bm{x}'))+\pi_--s(\bm{x},\bm{x}'))\ell(g(\bm{x}'),+1)\\
&\hspace{1cm}+(2\pi_-(\pi_--c(\bm{x},\bm{x}'))+\pi_+-s(\bm{x},\bm{x}'))\ell(g(\bm{x}'),-1)]\\
=&2\pi_+\mathbb{E}_{p(\bm{x} \mid y=+1)}[\ell(g(\bm{x}),+1)]+2\pi_-\mathbb{E}_{p(\bm{x} \mid y=-1)}[\ell(g(\bm{x}),-1)]\\
=&2R(g).
\end{align*}
By dividing both sides by 2, we can obtain Theorem~\ref{thm:unbiased risk}.
\end{proof}

\subsection{Proof of Theorem~\ref{thm:minimum variance}}
We first prove Lemma~\ref{lemma:unbiased risk lambda}:
\begin{proof}
Based on Lemma~\ref{lemma of thm1}, the following equations hold:
\begin{align*}
&\mathbb{E}_{p(\bm{x},\bm{x}')}[\mathcal{L}(\bm{x},\bm{x}')]\\
=&\mathbb{E}_{p(\bm{x},\bm{x}')}[(2\pi_+(\pi_+-c(\bm{x},\bm{x}'))+\pi_--s(\bm{x},\bm{x}'))\ell(g(\bm{x}),+1)\\
&\hspace{1.5cm}+(2\pi_-(\pi_--c(\bm{x},\bm{x}'))+\pi_+-s(\bm{x},\bm{x}'))\ell(g(\bm{x}'),-1)]\\
=&\pi_+\mathbb{E}_{p(\bm{x} \mid y=+1)}[\ell(g(\bm{x}),+1)]+\pi_-\mathbb{E}_{p(\bm{x}' \mid y=-1)}[\ell(g(\bm{x}'),-1)]\\
=&\pi_+\mathbb{E}_{p(\bm{x} \mid y=+1)}[\ell(g(\bm{x}),+1)]+\pi_-\mathbb{E}_{p(\bm{x} \mid y=-1)}[\ell(g(\bm{x}),-1)]\\
=&R(g),
\end{align*}
and
\begin{align*}
\mathbb{E}&_{p(\bm{x},\bm{x}')}[\mathcal{L}(\bm{x}',\bm{x})]\\
=&\mathbb{E}_{p(\bm{x},\bm{x}')}[(2\pi_+(\pi_++c(\bm{x},\bm{x}'))+\pi_--s(\bm{x},\bm{x}'))\ell(g(\bm{x}'),+1)\\
&\hspace{1.5cm}+(2\pi_-(\pi_-+c(\bm{x},\bm{x}'))+\pi_+-s(\bm{x},\bm{x}'))\ell(g(\bm{x}),-1)]\\
=&\pi_+\mathbb{E}_{p(\bm{x}' \mid y=+1)}[\ell(g(\bm{x}'),+1)]+\pi_-\mathbb{E}_{p(\bm{x} \mid y=-1)}[\ell(g(\bm{x}),-1)]\\
=&\pi_+\mathbb{E}_{p(\bm{x} \mid y=+1)}[\ell(g(\bm{x}),+1)]+\pi_-\mathbb{E}_{p(\bm{x} \mid y=-1)}[\ell(g(\bm{x}),-1)]\\
=&R(g).
\end{align*}
Therefore, for an arbitrary weight $\lambda\in[0,1]$, 
\begin{align*}
R(g)=&\lambda{R(g)}+(1-\lambda)R(g)\\
=&\lambda\mathbb{E}_{p(\bm{x},\bm{x}')}[\mathcal{L}(\bm{x},\bm{x}')]+(1-\lambda)\mathbb{E}_{p(\bm{x},\bm{x}')}[\mathcal{L}(\bm{x}',\bm{x})].
\end{align*}
This indicates that 
\begin{align*}
\frac{1}{n}\sum_{i=1}^{n}(\lambda\mathcal{L}(\bm{x}_i,\bm{x}'_i)+(1-\lambda)\mathcal{L}(\bm{x}_i',\bm{x}_i))
\end{align*}
is also an unbiased risk estimator of the classification risk.
\end{proof}
We now prove Theorem~\ref{thm:minimum variance}. We introduce the following notations:
\[S(g;\lambda)=\frac{1}{n}\sum_{i=1}^{n}(\lambda\mathcal{L}(\bm{x}_i,\bm{x}_i')+(1-\lambda)\mathcal{L}(\bm{x}_i',\bm{x}_i)),\]
\[\mu_1\triangleq\mathbb{E}_{p(\bm{x},\bm{x}')}\left[\left(\frac{1}{n}\sum_{i=1}^{n}\mathcal{L}(\bm{x}_i,\bm{x}_i')\right)^2\right]=\mathbb{E}_{p(\bm{x},\bm{x}')}\left[\left(\frac{1}{n}\sum_{i=1}^{n}\mathcal{L}(\bm{x}_i',\bm{x}_i)\right)^2\right],\]
\[\mu_2\triangleq\mathbb{E}_{p(\bm{x},\bm{x}')}\left[\frac{1}{n^2}\sum_{i=1}^{n}\mathcal{L}(\bm{x}_i,\bm{x}_i')\sum_{i=1}^{n}\mathcal{L}(\bm{x}_i',\bm{x}_i)\right].\]
Based on Lemma~\ref{lemma:unbiased risk lambda}, 
\[\mathbb{E}_{p(\bm{x},\bm{x}')}[S(g;\lambda)]=R(g).\]
\begin{proof}
\begin{align*}
\mathrm{Var}(S(g;\lambda))=&\mathbb{E}_{p(\bm{x},\bm{x}')}[(S(g;\lambda)-R(g))^2]\\
=&\mathbb{E}_{p(\bm{x},\bm{x}')}[S(g;\lambda)^2]-R(g)^2\\
=&\lambda^2\mathbb{E}_{p(\bm{x},\bm{x}')}\left[\left(\frac{1}{n}\sum_{i=1}^{n}\mathcal{L}(\bm{x}_i,\bm{x}_i')\right)^2\right]+(1-\lambda)^2\mathbb{E}_{p(\bm{x},\bm{x}')}\left[\left(\frac{1}{n}\sum_{i=1}^{n}\mathcal{L}(\bm{x}_i',\bm{x}_i)\right)^2\right]\\
&+2\lambda(1-\lambda)\mathbb{E}_{p(\bm{x},\bm{x}')}\left[\frac{1}{n^2}\sum_{i=1}^{n}\mathcal{L}(\bm{x}_i,\bm{x}_i')\sum_{i=1}^{n}\mathcal{L}(\bm{x}_i',\bm{x}_i)\right]-R(g)^2\\
=&\mu_1\lambda^2+\mu_1(1-\lambda)^2+2\mu_2\lambda(1-\lambda)-R(g)^2\\
=&(2\mu_1-2\mu_2)\lambda^2-(2\mu_1-2\mu_2)\lambda+\mu_1-R(g)^2\\
=&(2\mu_1-2\mu_2)(\lambda-\frac{1}{2})^2+\frac{1}{2}(\mu_1+\mu_2)-R(g)^2.
\end{align*}
Besides, it can be observed that
\begin{align*}
2\mu_1-2\mu_2=&\mathbb{E}_{p(\bm{x},\bm{x}')}\left[\left(\frac{1}{n}\sum_{i=1}^{n}\mathcal{L}(\bm{x}_i,\bm{x}_i')\right)^2\right]+\mathbb{E}_{p(\bm{x},\bm{x}')}\left[\left(\frac{1}{n}\sum_{i=1}^{n}\mathcal{L}(\bm{x}_i',\bm{x}_i)\right)^2\right]\\
&\hspace{1cm}-2\mathbb{E}_{p(\bm{x},\bm{x}')}\left[\frac{1}{n^2}\sum_{i=1}^{n}\mathcal{L}(\bm{x}_i,\bm{x}_i')\sum_{i=1}^{n}\mathcal{L}(\bm{x}_i',\bm{x}_i)\right]\\
=&\mathbb{E}_{p(\bm{x},\bm{x}')}\left[\left(\frac{1}{n}\sum_{i=1}^{n}(\mathcal{L}(\bm{x}_i,\bm{x}_i')-\mathcal{L}(\bm{x}_i',\bm{x}_i))\right)^2\right]\ge{0}.
\end{align*}
Therefore, the value of $\lambda$ that minimizes $\mathrm{Var}(S(g;\lambda))$ is $\lambda=1/2$.
\end{proof}

\subsection{Proof of Theorem~\ref{thm:estimation error}}\label{def:rademacher}
For convenience, we make the following notation:
\begin{align*}
\mathcal{L}_{\mathrm{SCD}}(g;\bm{x}_i,\bm{x}_i')=\frac{1}{2}\{(&2\pi_+(\pi_+-c(\bm{x}_i,\bm{x}_i'))+\pi_--s(\bm{x}_i,\bm{x}_i'))\ell(g(\bm{x}_i),+1)\\
+(&2\pi_-(\pi_--c(\bm{x}_i,\bm{x}_i'))+\pi_+-s(\bm{x}_i,\bm{x}_i'))\ell(g(\bm{x}_i'),-1)\\
+(&2\pi_+(\pi_++c(\bm{x}_i,\bm{x}_i'))+\pi_--s(\bm{x}_i,\bm{x}_i'))\ell(g(\bm{x}_i'),+1)\\
+(&2\pi_-(\pi_-+c(\bm{x}_i,\bm{x}_i'))+\pi_+-s(\bm{x}_i,\bm{x}_i'))\ell(g(\bm{x}_i),-1)\}.
\end{align*}
\begin{definition}[Rademacher complexity]
Let $\mathcal{X}_n=\{\bm{x}_1, \cdots, \bm{x}_n\}$ be $n$ independent random variables drawn from a distribution with probability density $p(\bm{x})$, $\mathcal{G}=\{g:\mathcal{X}\rightarrow\mathbb{R}\}$ be a class of measurable functions, and  $\bm{\sigma}=(\sigma_1, \cdots, \sigma_n)$ be independent Rademacher variables taking values in $\{-1, +1\}$. 
Then, the (expected) Rademacher complexity of $\mathcal{G}$ is defined as
\begin{align}
\mathfrak{R}_n(\mathcal{G})=\mathbb{E}_{\bm{x}_1,\cdots,\bm{x}_n}\mathbb{E}_{\bm{\sigma}}\left[\sup_{g\in\mathcal{G}}\frac{1}{n}\sum_{i=1}^{n}\sigma_ig(\bm{x}_i)\right].\label{rademacher}
\end{align}
\end{definition}
\begin{lemma}\label{lemma:upper rademacher}
\[\bar{\mathfrak{R}}_n(\mathcal{L_{\mathrm{SCD}}\circ{\mathcal{G}}})\le3L_{\ell}\mathfrak{R_n(\mathcal{G})},\]
where $\mathcal{L}_{\mathrm{SCD}}\circ\mathcal{G}=\{\mathcal{L}_{\mathrm{SCD}}\circ{g}\}$ and $\bar{\mathfrak{R}}_n(\cdot)$ is the Rademacher complexity over SconfConfDiff data pairs $\mathcal{D}_n$ of size $n$.
\end{lemma}
\begin{proof}
\begin{align*}
\bar{\mathfrak{R}}_n(\mathcal{L_{\mathrm{SCD}}\circ{\mathcal{G}}})=\mathbb{E}_{\mathcal{D}_n}&\mathbb{E}_{\bm{\sigma}}\left[\underset{g\in\mathcal{G}}{\sup}\frac{1}{n}\sum_{i=1}^{n}\sigma_i\mathcal{L}_{\mathrm{SCD}}(g;\bm{x}_i,\bm{x}_i')\right]\\
=\mathbb{E}_{\mathcal{D}_n}\mathbb{E}_{\bm{\sigma}}\biggl[\underset{g\in\mathcal{G}}{\sup}\frac{1}{2n}\sum_{i=1}^{n}\{(&2\pi_+(\pi_+-c(\bm{x}_i,\bm{x}_i'))+\pi_--s(\bm{x}_i,\bm{x}_i'))\ell(g(\bm{x}_i),+1)\\
+(&2\pi_-(\pi_--c(\bm{x}_i,\bm{x}_i'))+\pi_+-s(\bm{x}_i,\bm{x}_i'))\ell(g(\bm{x}_i'),-1)\\
+(&2\pi_+(\pi_++c(\bm{x}_i,\bm{x}_i'))+\pi_--s(\bm{x}_i,\bm{x}_i'))\ell(g(\bm{x}_i'),+1)\\
+(&2\pi_-(\pi_-+c(\bm{x}_i,\bm{x}_i'))+\pi_+-s(\bm{x}_i,\bm{x}_i'))\ell(g(\bm{x}_i),-1)\}\biggr].
\end{align*}
Then, we can induce that
\begin{align*}
&\lVert\nabla\mathcal{L}_{\mathrm{SCD}}(g;\bm{x},\bm{x}')\rVert_2\\
=&\biggl\lVert\nabla(\frac{(2\pi_+(\pi_+-c(\bm{x},\bm{x}'))+\pi_--s(\bm{x},\bm{x}'))\ell(g(\bm{x}),+1)}{2}\\
&+\frac{(2\pi_-(\pi_--c(\bm{x},\bm{x}'))+\pi_+-s(\bm{x},\bm{x}'))\ell(g(\bm{x}'),-1)}{2}\\
&+\frac{(2\pi_+(\pi_++c(\bm{x},\bm{x}'))+\pi_--s(\bm{x},\bm{x}'))\ell(g(\bm{x}'),+1)}{2}\\
&+\frac{(2\pi_-(\pi_-+c(\bm{x},\bm{x}'))+\pi_+-s(\bm{x},\bm{x}'))\ell(g(\bm{x}),-1)}{2})\biggr\rVert_2\\
\le&\biggl\lVert\nabla(\frac{(2\pi_+(\pi_+-c(\bm{x},\bm{x}'))+\pi_--s(\bm{x},\bm{x}'))\ell(g(\bm{x}),+1)}{2})\biggr\rVert_2\\
&+\biggl\lVert\nabla(\frac{(2\pi_-(\pi_--c(\bm{x},\bm{x}'))+\pi_+-s(\bm{x},\bm{x}'))\ell(g(\bm{x}'),-1)}{2})\biggr\rVert_2\\
&+\biggl\lVert\nabla(\frac{(2\pi_+(\pi_++c(\bm{x},\bm{x}'))+\pi_--s(\bm{x},\bm{x}'))\ell(g(\bm{x}'),+1)}{2})\biggr\rVert_2\\
&+\biggl\lVert\nabla(\frac{(2\pi_-(\pi_-+c(\bm{x},\bm{x}'))+\pi_+-s(\bm{x},\bm{x}'))\ell(g(\bm{x}),-1)}{2})\biggr\rVert_2\\
\le&\frac{|2\pi_+(\pi_+-c(\bm{x},\bm{x}'))+\pi_--s(\bm{x},\bm{x}')|L_{\ell}}{2}+\frac{|2\pi_-(\pi_--c(\bm{x},\bm{x}'))+\pi_+-s(\bm{x},\bm{x}')|L_{\ell}}{2}\\
&+\frac{|2\pi_+(\pi_++c(\bm{x},\bm{x}'))+\pi_--s(\bm{x},\bm{x}')|L_{\ell}}{2}+\frac{|2\pi_-(\pi_-+c(\bm{x},\bm{x}'))+\pi_+-s(\bm{x},\bm{x}')|L_{\ell}}{2}\\
\le&(\frac{|2\pi_+(\pi_+-c(\bm{x},\bm{x}'))|+|\pi_--s(\bm{x},\bm{x}')|}{2}+\frac{|2\pi_-(\pi_--c(\bm{x},\bm{x}'))|+|\pi_+-s(\bm{x},\bm{x}')|}{2}\\
&+\frac{|2\pi_+(\pi_++c(\bm{x},\bm{x}'))|+|\pi_--s(\bm{x},\bm{x}')|}{2}+\frac{|2\pi_-(\pi_-+c(\bm{x},\bm{x}'))|+|\pi_+-s(\bm{x},\bm{x}')|}{2})L_{\ell}.\\
\end{align*}
Suppose $\pi_+\ge\pi_-$. \\
(i)When $\hspace{1mm}c(\bm{x},\bm{x}')\in[-1,-\pi_+)$,  
\[
\frac{|2\pi_+(\pi_+-c(\bm{x},\bm{x}'))|+|\pi_--s(\bm{x},\bm{x}')|}{2}=
\left\{
{\setlength{\arraycolsep}{1pt}
\begin{array}{ll}
\frac{2\pi_+(\pi_+-c(\bm{x},\bm{x}'))+\pi_--s(\bm{x},\bm{x}')}{2}, & s(\bm{x},\bm{x}') \in [0,\pi_-)\\
\frac{2\pi_+(\pi_+-c(\bm{x},\bm{x}'))+s(\bm{x},\bm{x}')-\pi_-}{2}, & s(\bm{x},\bm{x}') \in [\pi_-,\pi_+)\\
\frac{2\pi_+(\pi_+-c(\bm{x},\bm{x}'))+s(\bm{x},\bm{x}')-\pi_-}{2}, & s(\bm{x},\bm{x}') \in [\pi_+,1]\\
\end{array}}
\right.
\]
\[
\frac{|2\pi_-(\pi_--c(\bm{x},\bm{x}'))|+|\pi_+-s(\bm{x},\bm{x}')|}{2}=
\left\{
{\setlength{\arraycolsep}{1pt}
\begin{array}{ll}
\frac{2\pi_-(\pi_--c(\bm{x},\bm{x}'))+\pi_+-s(\bm{x},\bm{x}')}{2}, & s(\bm{x},\bm{x}') \in [0,\pi_-)\\
\frac{2\pi_-(\pi_--c(\bm{x},\bm{x}'))+\pi_+-s(\bm{x},\bm{x}')}{2}, & s(\bm{x},\bm{x}') \in [\pi_-,\pi_+)\\
\frac{2\pi_-(\pi_--c(\bm{x},\bm{x}'))+s(\bm{x},\bm{x}') - \pi_+}{2}, & s(\bm{x},\bm{x}') \in [\pi_+,1]\\
\end{array}}
\right.
\]
\[
\frac{|2\pi_+(\pi_++c(\bm{x},\bm{x}'))|+|\pi_--s(\bm{x},\bm{x}')|}{2}=
\left\{
{\setlength{\arraycolsep}{1pt}
\begin{array}{ll}
\frac{-2\pi_+(\pi_++c(\bm{x},\bm{x}'))+\pi_--s(\bm{x},\bm{x}')}{2}, & s(\bm{x},\bm{x}') \in [0,\pi_-)\\
\frac{-2\pi_+(\pi_++c(\bm{x},\bm{x}'))+s(\bm{x},\bm{x}')-\pi_-}{2}, & s(\bm{x},\bm{x}') \in [\pi_-,\pi_+)\\
\frac{-2\pi_+(\pi_++c(\bm{x},\bm{x}'))+s(\bm{x},\bm{x}')-\pi_-}{2}, & s(\bm{x},\bm{x}') \in [\pi_+,1]\\
\end{array}}
\right.
\]
\[
\frac{|2\pi_-(\pi_-+c(\bm{x},\bm{x}'))|+|\pi_+-s(\bm{x},\bm{x}')|}{2}=
\left\{
{\setlength{\arraycolsep}{1pt}
\begin{array}{ll}
\frac{-2\pi_-(\pi_-+c(\bm{x},\bm{x}'))+\pi_+-s(\bm{x},\bm{x}')}{2}, & s(\bm{x},\bm{x}') \in [0,\pi_-)\\
\frac{-2\pi_-(\pi_-+c(\bm{x},\bm{x}'))+\pi_+-s(\bm{x},\bm{x}')}{2}, & s(\bm{x},\bm{x}') \in [\pi_-,\pi_+)\\
\frac{-2\pi_-(\pi_-+c(\bm{x},\bm{x}'))+s(\bm{x},\bm{x}') - \pi_+}{2}, & s(\bm{x},\bm{x}') \in [\pi_+,1]\\
\end{array}}
\right.
\]
Therefore, 
\[
\lVert\nabla\mathcal{L}_{\mathrm{SCD}}(g;\bm{x},\bm{x}')\rVert_2\le
\left\{
\begin{array}{ll}
(1-2(s(\bm{x},\bm{x}')+c(\bm{x},\bm{x}')))L_{\ell}, & s(\bm{x},\bm{x}') \in [0,\pi_-)\\
(\pi_+-\pi_--2c(\bm{x},\bm{x}')))L_{\ell}, & s(\bm{x},\bm{x}') \in [\pi_-,\pi_+)\\
(2(s(\bm{x},\bm{x}')-c(\bm{x},\bm{x}'))-1)L_{\ell}, & s(\bm{x},\bm{x}') \in [\pi_+,1]\\
\end{array}
\right.
\]
(ii)When $\hspace{1mm}c(\bm{x},\bm{x}')\in[-\pi_+,-\pi_-)$, 
\[
\frac{|2\pi_+(\pi_+-c(\bm{x},\bm{x}'))|+|\pi_--s(\bm{x},\bm{x}')|}{2}=
\left\{
{\setlength{\arraycolsep}{1pt}
\begin{array}{ll}
\frac{2\pi_+(\pi_+-c(\bm{x},\bm{x}'))+\pi_--s(\bm{x},\bm{x}')}{2}, & s(\bm{x},\bm{x}') \in [0,\pi_-)\\
\frac{2\pi_+(\pi_+-c(\bm{x},\bm{x}'))+s(\bm{x},\bm{x}')-\pi_-}{2}, & s(\bm{x},\bm{x}') \in [\pi_-,\pi_+)\\
\frac{2\pi_+(\pi_+-c(\bm{x},\bm{x}'))+s(\bm{x},\bm{x}')-\pi_-}{2}, & s(\bm{x},\bm{x}') \in [\pi_+,1]\\
\end{array}}
\right.
\]
\[
\frac{|2\pi_-(\pi_--c(\bm{x},\bm{x}'))|+|\pi_+-s(\bm{x},\bm{x}')|}{2}=
\left\{
{\setlength{\arraycolsep}{1pt}
\begin{array}{ll}
\frac{2\pi_-(\pi_--c(\bm{x},\bm{x}'))+\pi_+-s(\bm{x},\bm{x}')}{2}, & s(\bm{x},\bm{x}') \in [0,\pi_-)\\
\frac{2\pi_-(\pi_--c(\bm{x},\bm{x}'))+\pi_+-s(\bm{x},\bm{x}')}{2}, & s(\bm{x},\bm{x}') \in [\pi_-,\pi_+)\\
\frac{2\pi_-(\pi_--c(\bm{x},\bm{x}'))+s(\bm{x},\bm{x}') - \pi_+}{2}, & s(\bm{x},\bm{x}') \in [\pi_+,1]\\
\end{array}}
\right.
\]
\[
\frac{|2\pi_+(\pi_++c(\bm{x},\bm{x}'))|+|\pi_--s(\bm{x},\bm{x}')|}{2}=
\left\{
{\setlength{\arraycolsep}{1pt}
\begin{array}{ll}
\frac{2\pi_+(\pi_++c(\bm{x},\bm{x}'))+\pi_--s(\bm{x},\bm{x}')}{2}, & s(\bm{x},\bm{x}') \in [0,\pi_-)\\
\frac{2\pi_+(\pi_++c(\bm{x},\bm{x}'))+s(\bm{x},\bm{x}')-\pi_-}{2}, & s(\bm{x},\bm{x}') \in [\pi_-,\pi_+)\\
\frac{2\pi_+(\pi_++c(\bm{x},\bm{x}'))+s(\bm{x},\bm{x}')-\pi_-}{2}, & s(\bm{x},\bm{x}') \in [\pi_+,1]\\
\end{array}}
\right.
\]
\[
\frac{|2\pi_-(\pi_-+c(\bm{x},\bm{x}'))|+|\pi_+-s(\bm{x},\bm{x}')|}{2}=
\left\{
{\setlength{\arraycolsep}{1pt}
\begin{array}{ll}
\frac{-2\pi_-(\pi_-+c(\bm{x},\bm{x}')+\pi_+-s(\bm{x},\bm{x}')}{2}, & s(\bm{x},\bm{x}') \in [0,\pi_-)\\
\frac{-2\pi_-(\pi_-+c(\bm{x},\bm{x}')+\pi_+-s(\bm{x},\bm{x}')}{2}, & s(\bm{x},\bm{x}') \in [\pi_-,\pi_+)\\
\frac{-2\pi_-(\pi_-+c(\bm{x},\bm{x}')+s(\bm{x},\bm{x}') - \pi_+}{2}, & s(\bm{x},\bm{x}') \in [\pi_+,1]\\
\end{array}}
\right.
\]
Therefore,
\begin{align*}
\lVert&\nabla\mathcal{L}_{\mathrm{SCD}}(g;\bm{x},\bm{x}')\rVert_2\\
&\le\left\{
{\setlength{\arraycolsep}{3pt}
\begin{array}{ll}
(2\pi_+^2+1-2s(\bm{x},\bm{x}')+(\pi_+-\pi_--1)c(\bm{x},\bm{x}')))L_{\ell}, & s(\bm{x},\bm{x}') \in [0,\pi_-)\\
(2\pi_+^2+\pi_+-\pi_-+(\pi_+-\pi_--1)c(\bm{x},\bm{x}')))L_{\ell}, & s(\bm{x},\bm{x}') \in [\pi_-,\pi_+)\\
(2\pi_+^2-1+2s(\bm{x},\bm{x}')+(\pi_+-\pi_--1)c(\bm{x},\bm{x}')))L_{\ell}, & s(\bm{x},\bm{x}') \in [\pi_+,1]\\
\end{array}}
\right.
\end{align*}
(iii)When $\hspace{1mm}c(\bm{x},\bm{x}')\in[-\pi_-,\pi_-)$, 
\[
\frac{|2\pi_+(\pi_+-c(\bm{x},\bm{x}'))|+|\pi_--s(\bm{x},\bm{x}')|}{2}=
\left\{
{\setlength{\arraycolsep}{1pt}
\begin{array}{ll}
\frac{2\pi_+(\pi_+-c(\bm{x},\bm{x}'))+\pi_--s(\bm{x},\bm{x}')}{2}, & s(\bm{x},\bm{x}') \in [0,\pi_-)\\
\frac{2\pi_+(\pi_+-c(\bm{x},\bm{x}'))+s(\bm{x},\bm{x}')-\pi_-}{2}, & s(\bm{x},\bm{x}') \in [\pi_-,\pi_+)\\
\frac{2\pi_+(\pi_+-c(\bm{x},\bm{x}'))+s(\bm{x},\bm{x}')-\pi_-)}{2}, & s(\bm{x},\bm{x}') \in [\pi_+,1]\\
\end{array}}
\right.
\]
\[
\frac{|2\pi_-(\pi_--c(\bm{x},\bm{x}'))|+|\pi_+-s(\bm{x},\bm{x}')|}{2}=
\left\{
{\setlength{\arraycolsep}{1pt}
\begin{array}{ll}
\frac{2\pi_-(\pi_--c(\bm{x},\bm{x}'))+\pi_+-s(\bm{x},\bm{x}')}{2}, & s(\bm{x},\bm{x}') \in [0,\pi_-)\\
\frac{2\pi_-(\pi_--c(\bm{x},\bm{x}'))+\pi_+-s(\bm{x},\bm{x}')}{2}, & s(\bm{x},\bm{x}') \in [\pi_-,\pi_+)\\
\frac{2\pi_-(\pi_--c(\bm{x},\bm{x}'))+s(\bm{x},\bm{x}') - \pi_+}{2}, & s(\bm{x},\bm{x}') \in [\pi_+,1]\\
\end{array}}
\right.
\]
\[
\frac{|2\pi_+(\pi_++c(\bm{x},\bm{x}'))|+|\pi_--s(\bm{x},\bm{x}')|}{2}=
\left\{
{\setlength{\arraycolsep}{1pt}
\begin{array}{ll}
\frac{2\pi_+(\pi_++c(\bm{x},\bm{x}'))+\pi_--s(\bm{x},\bm{x}')}{2}, & s(\bm{x},\bm{x}') \in [0,\pi_-)\\
\frac{2\pi_+(\pi_++c(\bm{x},\bm{x}'))+s(\bm{x},\bm{x}')-\pi_-}{2}, & s(\bm{x},\bm{x}') \in [\pi_-,\pi_+)\\
\frac{2\pi_+(\pi_++c(\bm{x},\bm{x}'))+s(\bm{x},\bm{x}')-\pi_-}{2}, & s(\bm{x},\bm{x}') \in [\pi_+,1]\\
\end{array}}
\right.
\]
\[
\frac{|2\pi_-(\pi_-+c(\bm{x},\bm{x}'))|+|\pi_+-s(\bm{x},\bm{x}')|}{2}=
\left\{
{\setlength{\arraycolsep}{1pt}
\begin{array}{ll}
\frac{2\pi_-(\pi_-+c(\bm{x},\bm{x}'))+\pi_+-s(\bm{x},\bm{x}')}{2}, & s(\bm{x},\bm{x}') \in [0,\pi_-)\\
\frac{2\pi_-(\pi_-+c(\bm{x},\bm{x}'))+\pi_+-s(\bm{x},\bm{x}')}{2}, & s(\bm{x},\bm{x}') \in [\pi_-,\pi_+)\\
\frac{2\pi_-(\pi_-+c(\bm{x},\bm{x}'))+s(\bm{x},\bm{x}') - \pi_+}{2}, & s(\bm{x},\bm{x}') \in [\pi_+,1]\\
\end{array}}
\right.
\]
Therefore,  
\[
\lVert\nabla\mathcal{L}_{\mathrm{SCD}}(g;\bm{x},\bm{x}')\rVert_2\le
\left\{
\begin{array}{ll}
(2(\pi_+^2+\pi_-^2)+1-2s(\bm{x},\bm{x}'))L_{\ell}, & s(\bm{x},\bm{x}') \in [0,\pi_-)\\
(2(\pi_+^2+\pi_-^2)+\pi_+-\pi_-)L_{\ell}, & s(\bm{x},\bm{x}') \in [\pi_-,\pi_+)\\
(2(\pi_+^2+\pi_-^2)+2s(\bm{x},\bm{x}')-1)L_{\ell}, & s(\bm{x},\bm{x}') \in [\pi_+,1]\\
\end{array}
\right.
\]
(iv)When $\hspace{1mm}c(\bm{x},\bm{x}')\in[\pi_-,\pi_+)$, 
\[
\frac{|2\pi_+(\pi_+-c(\bm{x},\bm{x}'))|+|\pi_--s(\bm{x},\bm{x}')|}{2}=
\left\{
{\setlength{\arraycolsep}{1pt}
\begin{array}{ll}
\frac{2\pi_+(\pi_+-c(\bm{x},\bm{x}'))+\pi_--s(\bm{x},\bm{x}')}{2}, & s(\bm{x},\bm{x}') \in [0,\pi_-)\\
\frac{2\pi_+(\pi_+-c(\bm{x},\bm{x}'))+s(\bm{x},\bm{x}')-\pi_-}{2}, & s(\bm{x},\bm{x}') \in [\pi_-,\pi_+)\\
\frac{2\pi_+(\pi_+-c(\bm{x},\bm{x}'))+s(\bm{x},\bm{x}')-\pi_-}{2}, & s(\bm{x},\bm{x}') \in [\pi_+,1]\\
\end{array}}
\right.
\]
\[
\frac{|2\pi_-(\pi_--c(\bm{x},\bm{x}'))|+|\pi_+-s(\bm{x},\bm{x}')|}{2}=
\left\{
{\setlength{\arraycolsep}{1pt}
\begin{array}{ll}
\frac{-2\pi_-(\pi_--c(\bm{x},\bm{x}'))+\pi_+-s(\bm{x},\bm{x}')}{2}, & s(\bm{x},\bm{x}') \in [0,\pi_-)\\
\frac{-2\pi_-(\pi_--c(\bm{x},\bm{x}'))+\pi_+-s(\bm{x},\bm{x}')}{2}, & s(\bm{x},\bm{x}') \in [\pi_-,\pi_+)\\
\frac{-2\pi_-(\pi_--c(\bm{x},\bm{x}'))+s(\bm{x},\bm{x}') - \pi_+}{2}, & s(\bm{x},\bm{x}') \in [\pi_+,1]\\
\end{array}}
\right.
\]
\[
\frac{|2\pi_+(\pi_++c(\bm{x},\bm{x}'))|+|\pi_--s(\bm{x},\bm{x}')|}{2}=
\left\{
{\setlength{\arraycolsep}{1pt}
\begin{array}{ll}
\frac{2\pi_+(\pi_++c(\bm{x},\bm{x}'))+\pi_--s(\bm{x},\bm{x}')}{2}, & s(\bm{x},\bm{x}') \in [0,\pi_-)\\
\frac{2\pi_+(\pi_++c(\bm{x},\bm{x}'))+s(\bm{x},\bm{x}')-\pi_-}{2}, & s(\bm{x},\bm{x}') \in [\pi_-,\pi_+)\\
\frac{2\pi_+(\pi_++c(\bm{x},\bm{x}'))+s(\bm{x},\bm{x}')-\pi_-}{2}, & s(\bm{x},\bm{x}') \in [\pi_+,1]\\
\end{array}}
\right.
\]
\[
\frac{|2\pi_-(\pi_-+c(\bm{x},\bm{x}'))|+|\pi_+-s(\bm{x},\bm{x}')|}{2}=
\left\{
{\setlength{\arraycolsep}{1pt}
\begin{array}{ll}
\frac{2\pi_-(\pi_-+c(\bm{x},\bm{x}'))+\pi_+-s(\bm{x},\bm{x}')}{2}, & s(\bm{x},\bm{x}') \in [0,\pi_-)\\
\frac{2\pi_-(\pi_-+c(\bm{x},\bm{x}'))+\pi_+-s(\bm{x},\bm{x}')}{2}, & s(\bm{x},\bm{x}') \in [\pi_-,\pi_+)\\
\frac{2\pi_-(\pi_-+c(\bm{x},\bm{x}'))+s(\bm{x},\bm{x}') - \pi_+}{2}, & s(\bm{x},\bm{x}') \in [\pi_+,1]\\
\end{array}}
\right.
\]
Therefore, 
\[
\lVert\nabla\mathcal{L}_{\mathrm{SCD}}(g;\bm{x},\bm{x}')\rVert_2\le
\left\{
\begin{array}{ll}
(2\pi_+^2+1-2(s(\bm{x},\bm{x}')-\pi_-c(\bm{x},\bm{x}')))L_{\ell}, & s(\bm{x},\bm{x}') \in [0,\pi_-)\\
(2(\pi_+^2+\pi_-c(\bm{x},\bm{x}')))+\pi_+-\pi_-)L_{\ell}, & s(\bm{x},\bm{x}') \in [\pi_-,\pi_+)\\
(2\pi_+^2-1+2(s(\bm{x},\bm{x}')+\pi_-c(\bm{x},\bm{x}')))L_{\ell}, & s(\bm{x},\bm{x}') \in [\pi_+,1]\\
\end{array}
\right.
\]
(v)When $\hspace{1mm}c(\bm{x},\bm{x}')\in[\pi_+,1]$,  
\[
\frac{|2\pi_+(\pi_+-c(\bm{x},\bm{x}'))|+|\pi_--s(\bm{x},\bm{x}')|}{2}=
\left\{
{\setlength{\arraycolsep}{1pt}
\begin{array}{ll}
\frac{-2\pi_+(\pi_+-c(\bm{x},\bm{x}'))+\pi_--s(\bm{x},\bm{x}')}{2}, & s(\bm{x},\bm{x}') \in [0,\pi_-)\\
\frac{-2\pi_+(\pi_+-c(\bm{x},\bm{x}'))+s(\bm{x},\bm{x}')-\pi_-}{2}, & s(\bm{x},\bm{x}') \in [\pi_-,\pi_+)\\
\frac{-2\pi_+(\pi_+-c(\bm{x},\bm{x}'))+s(\bm{x},\bm{x}')-\pi_-}{2}, & s(\bm{x},\bm{x}') \in [\pi_+,1]\\
\end{array}}
\right.
\]
\[
\frac{|2\pi_-(\pi_--c(\bm{x},\bm{x}'))|+|\pi_+-s(\bm{x},\bm{x}')|}{2}=
\left\{
{\setlength{\arraycolsep}{1pt}
\begin{array}{ll}
\frac{-2\pi_-(\pi_--c(\bm{x},\bm{x}'))+\pi_+-s(\bm{x},\bm{x}')}{2}, & s(\bm{x},\bm{x}') \in [0,\pi_-)\\
\frac{-2\pi_-(\pi_--c(\bm{x},\bm{x}'))+\pi_+-s(\bm{x},\bm{x}')}{2}, & s(\bm{x},\bm{x}') \in [\pi_-,\pi_+)\\
\frac{-2\pi_-(\pi_--c(\bm{x},\bm{x}'))+s(\bm{x},\bm{x}') - \pi_+}{2}, & s(\bm{x},\bm{x}') \in [\pi_+,1]\\
\end{array}}
\right.
\]
\[
\frac{|2\pi_+(\pi_++c(\bm{x},\bm{x}'))|+|\pi_--s(\bm{x},\bm{x}')|}{2}=
\left\{
{\setlength{\arraycolsep}{1pt}
\begin{array}{ll}
\frac{2\pi_+(\pi_++c(\bm{x},\bm{x}'))+\pi_--s(\bm{x},\bm{x}')}{2}, & s(\bm{x},\bm{x}') \in [0,\pi_-)\\
\frac{2\pi_+(\pi_++c(\bm{x},\bm{x}'))+s(\bm{x},\bm{x}')-\pi_-}{2}, & s(\bm{x},\bm{x}') \in [\pi_-,\pi_+)\\
\frac{2\pi_+(\pi_++c(\bm{x},\bm{x}'))+s(\bm{x},\bm{x}')-\pi_-}{2}, & s(\bm{x},\bm{x}') \in [\pi_+,1]\\
\end{array}}
\right.
\]
\[
\frac{|2\pi_-(\pi_-+c(\bm{x},\bm{x}'))|+|\pi_+-s(\bm{x},\bm{x}')|}{2}=
\left\{
{\setlength{\arraycolsep}{1pt}
\begin{array}{ll}
\frac{2\pi_-(\pi_-+c(\bm{x},\bm{x}'))+\pi_+-s(\bm{x},\bm{x}')}{2}, & s(\bm{x},\bm{x}') \in [0,\pi_-)\\
\frac{2\pi_-(\pi_-+c(\bm{x},\bm{x}'))+\pi_+-s(\bm{x},\bm{x}')}{2}, & s(\bm{x},\bm{x}') \in [\pi_-,\pi_+)\\
\frac{2\pi_-(\pi_-+c(\bm{x},\bm{x}'))+s(\bm{x},\bm{x}') - \pi_+}{2}, & s(\bm{x},\bm{x}') \in [\pi_+,1]\\
\end{array}}
\right.
\]
Therefore, 
\[
\lVert\nabla\mathcal{L}_{\mathrm{SCD}}(g;\bm{x},\bm{x}')\rVert_2\le
\left\{
\begin{array}{ll}
(1-2(s(\bm{x},\bm{x}')-c(\bm{x},\bm{x}')))L_{\ell}, & s(\bm{x},\bm{x}') \in [0,\pi_-)\\
(2c(\bm{x},\bm{x}')+\pi_+-\pi_-)L_{\ell}, & s(\bm{x},\bm{x}') \in [\pi_-,\pi_+)\\
(2(s(\bm{x},\bm{x}')+c(\bm{x},\bm{x}'))-1)L_{\ell}, & s(\bm{x},\bm{x}') \in [\pi_+,1]\\
\end{array}
\right.
\]
As a result, 
\begin{align}
\lVert\nabla\mathcal{L}_{\mathrm{SCD}}(g;\bm{x},\bm{x}')\rVert_2\le{3L_{\ell}},\label{bound of L_SCD}
\end{align}
and according to Talagrand's Lemma~\cite{Talagrand},
\begin{align*}
\bar{\mathfrak{R}}_n(\mathcal{L_{\mathrm{SCD}}\circ{\mathcal{G}}})&\le3L_{\ell}\mathbb{E}_{\mathcal{D}_n}\mathbb{E}_{\bm{\sigma}}\left[\underset{g\in\mathcal{G}}{\sup}\frac{1}{n}\sum_{i=1}^{n}\sigma_ig(\bm{x}_i)\right]\\
=&3L_{\ell}\mathfrak{R_n(\mathcal{G})}.
\end{align*}
\end{proof}
\begin{lemma}\label{lemma:sup unbiased}
The following inequality holds with probability at least $1-\delta$:
\[\underset{g\in\mathcal{G}}{\sup}|R(g)-\widehat{R}_{\mathrm{SCD}}(g)|\le6L_{\ell}\mathfrak{R}_n(\mathcal{G})+\left(\pi_+^2+\pi_-^2+\frac{5}{2}\right)C_{\ell}\sqrt{\frac{\log2/\delta}{2n}}\]
\end{lemma}
\begin{proof}
The following inequalities hold for $\mathcal{L}_{\mathrm{SCD}}$.
\begin{align*}
\mathcal{L}_{\mathrm{SCD}}(g;\bm{x},\bm{x}')\le&\frac{1}{2n}((2\pi_+^2-2\pi_+c_i+\pi_-)C_{\ell}+(2\pi_-^2-2\pi_-c_i+\pi_+)C_{\ell}\\
&+(2\pi_+^2+2\pi_+c_i+\pi_-)C_{\ell}+(2\pi_-^2+2\pi_-c_i+\pi_+)C_{\ell})\\
=&\frac{1}{2n}(4(\pi_+^2+\pi_-^2)+2)C_{\ell}\\
=&\frac{1}{n}(2\pi_+^2+2\pi_-^2+1)C_{\ell},
\end{align*}
and
\begin{align*}
\mathcal{L}_{\mathrm{SCD}}(g;\bm{x},\bm{x}')>&\frac{1}{2n}(2\pi_+^2-2\pi_+c_i+\pi_--1+2\pi_-^2-2\pi_-c_i+\pi_+-1)C_{\ell}\\
=&\frac{1}{2n}(2(\pi_+^2+\pi_-^2)-2c_i-1)C_{\ell}\\
=&\frac{1}{n}\left(\pi_+^2+\pi_-^2-\frac{3}{2}\right).
\end{align*}
Thus, the following inequality holds: 
\[\frac{1}{n}\left(\pi_+^2+\pi_-^2-\frac{3}{2}\right)<\mathcal{L}_{\mathrm{SCD}}(g;\bm{x},\bm{x}')\le\frac{1}{n}(2\pi_+^2+2\pi_-^2+1)C_{\ell}\]

Here, we introduce $\Phi=\sup_{g\in\mathcal{G}}(R(g)-\widehat{R}_{\mathrm{SCD}}(g)), \dot{\Phi}=\sup_{g\in\mathcal{G}}(R(g)-\widehat{\dot{R}}_{\mathrm{SCD}}(g))$, where $\widehat{R}_{\mathrm{SCD}}(g)$ and $\widehat{\dot{R}}_{\mathrm{SCD}}(g)$ are the empirical risk over two SconfConfDiff datasets with exactly one different pair $\{(\bm{x}_i,\bm{x}_i'),s_i, c_i\}$ and $\{(\dot{\bm{x}}_i,\dot{\bm{x}}_i'),\dot{s}_i, \dot{c}_i\}$.
Then, based on uniform law of large numbers, 
\begin{align*}
\dot{\Phi}-\Phi&\le\underset{g\in\mathcal{G}}{\sup}(\widehat{R}_{\mathrm{SCD}}(g)-\widehat{\dot{R}}_{\mathrm{SCD}}(g))\\
&\le\underset{g\in\mathcal{G}}{\sup}\left(\frac{\mathcal{L}_{\mathrm{SCD}}(g;\bm{x}_i,\bm{x}_i')-\mathcal{L}_{\mathrm{SCD}}(g;\dot{\bm{x}}_i,\dot{\bm{x}}_i')}{n}\right)\\
&\le\frac{(\pi_+^2+\pi_-^2+\frac{5}{2})C_{\ell}}{n}.
\end{align*}
By applying McDiarmid's inequality~\cite{McDiarmid}, the following inequality holds with probability at least $1-\delta/2$:
\[\underset{g\in\mathcal{G}}{\sup}(R(g)-\widehat{R}_{\mathrm{SCD}}(g))\le\mathbb{E}_{\mathcal{D}_n}[\underset{g\in\mathcal{G}}{\sup}(R(g)-\widehat{R}_{\mathrm{SCD}}(g))]+\left(\pi_+^2+\pi_-^2+\frac{5}{2}\right)C_{\ell}\sqrt{\frac{\log2/\delta}{2n}}.\]
Furthermore, we can bound $\mathbb{E}_{\mathcal{D}_n}[\underset{g\in\mathcal{G}}{\sup}(R(g)-\widehat{R}_{\mathrm{SCD}}(g))]$ with Rademacher complexity.
\[\mathbb{E}_{\mathcal{D}_n}[\underset{g\in\mathcal{G}}{\sup}(R(g)-\widehat{R}_{\mathrm{SCD}}(g))]\le2\bar{\mathfrak{R}}_n(\mathcal{L}_{\mathrm{SCD}}\circ\mathcal{G})\le6L_{\ell}\mathfrak{R}_n(\mathcal{G}).\]
As a result, The following inequality holds with probability at least $1-\delta$:
\[\underset{g\in\mathcal{G}}{\sup}|R(g)-\widehat{R}_{\mathrm{SCD}}(g)|\le6L_{\ell}\mathfrak{R}_n(\mathcal{G})+\left(\pi_+^2+\pi_-^2+\frac{5}{2}\right)C_{\ell}\sqrt{\frac{\log2/\delta}{2n}}.\]
\end{proof}
Finally, the proof of Theorem~\ref{thm:estimation error} is given as follows.
\begin{proof}
\begin{align*}
R(\hat{g}_{\mathrm{SCD}})-R(g^*)=&(R(\hat{g}_{\mathrm{SCD}})-\widehat{R}(\hat{g}_{\mathrm{SCD}}))+(\widehat{R}(\hat{g}_{\mathrm{SCD}})-\widehat{R}_{\mathrm{SCD}}(g^*))+(\widehat{R}_{\mathrm{SCD}}(g^*)-R(g^*))\\
&\le(R(\hat{g}_{\mathrm{SCD}})-\widehat{R}(\hat{g}_{\mathrm{SCD}}))+(\widehat{R}_{\mathrm{SCD}}(g^*)-R(g^*))\\
&\le|R(\hat{g}_{\mathrm{SCD}})-\widehat{R}(\hat{g}_{\mathrm{SCD}})|+|\widehat{R}_{\mathrm{SCD}}(g^*)-R(g^*)|\\
&\le2\underset{g\in\mathcal{G}}{\sup}|R(g)-\widehat{R}_{\mathrm{SCD}}(g)|\\
&\le12L_{\ell}\mathfrak{R}_n(\mathcal{G})+(2\pi_+^2+2\pi_-^2+5)C_{\ell}\sqrt{\frac{\log2/\delta}{2n}}.
\end{align*}
The first inequality is deduced from the definition $\hat{g}_{\mathrm{SCD}}$.
\end{proof}

\subsection{Proof of Theorem~\ref{thm:robust estimation error}}
\begin{proof}
\begin{align}
|&\bar{R}_{\mathrm{SCD}}(g)-\widehat{R}_{\mathrm{SCD}}(g)|\notag\\
=&\frac{1}{2n}\biggl|\sum_{i=1}^{n}((2(\bar{\pi}_+^2-\pi_+^2)-2(\bar{\pi}_+\bar{c}_i-\pi_+c_i)+(\bar{\pi}_--\pi_-)-(\bar{s}_i-s_i))\ell(g(\bm{x}_i),+1)\notag\\
&\hspace{1cm}+(2(\bar{\pi}_-^2-\pi_-^2)-2(\bar{\pi}_-\bar{c}_i-\pi_-c_i)+(\bar{\pi}_+-\pi_+)-(\bar{s}_i-s_i))\ell(g(\bm{x}_i'),-1)\notag\\
&\hspace{1cm}+(2(\bar{\pi}_+^2-\pi_+^2)+2(\bar{\pi}_+\bar{c}_i-\pi_+c_i)+(\bar{\pi}_--\pi_-)-(\bar{s}_i-s_i))\ell(g(\bm{x}_i'),+1)\notag\\
&\hspace{1cm}+(2(\bar{\pi}_-^2-\pi_-^2)+2(\bar{\pi}_-\bar{c}_i-\pi_-c_i)+(\bar{\pi}_+-\pi_+)-(\bar{s}_i-s_i))\ell(g(\bm{x}_i),-1))\biggr|\notag\\
\le&\frac{C_{\ell}}{2n}\sum_{i=1}^{n}(4|\bar{\pi}_+^2-\pi_+^2|+4|\bar{\pi}_-^2-\pi_-^2|+4|\bar{\pi}_+\bar{c}_i-\pi_+c_i|\notag\\
&\hspace{1cm}+4|\bar{\pi}_-\bar{c}_i-\pi_-c_i|+4|\bar{\pi}_+-\pi_+|+4|\bar{s}_i-s_i|)\notag\\
=&\frac{2C_{\ell}}{n}\sum_{i=1}^{n}(|\bar{\pi}_+\bar{c}_i-\pi_+c_i|+|\bar{\pi}_-\bar{c}_i-\pi_-c_i|+|\bar{s}_i-s_i|)\notag\\
&\hspace{1cm}+2C_{\ell}(|\bar{\pi}_+^2-\pi_+^2|+|\bar{\pi}_-^2-\pi_-^2|+|\bar{\pi}_+-\pi_+|).\notag
\end{align}
Then, the following inequality can be derived.
\begin{align*}
R&(\bar{g}_{\mathrm{SCD}})-R(g^*)\\
=&(R(\bar{g}_{\mathrm{SCD}})-\widehat{R}_{\mathrm{SCD}}(\bar{g}_{\mathrm{SCD}}))+(\widehat{R}_{\mathrm{SCD}}(\bar{g}_{\mathrm{SCD}})-\bar{R}_{\mathrm{SCD}}(\bar{g}_{\mathrm{SCD}}))+(\bar{R}_{\mathrm{SCD}}(\bar{g}_{\mathrm{SCD}})-\bar{R}_{\mathrm{SCD}}(\hat{g}_{\mathrm{SCD}}))\\
&+(\bar{R}_{\mathrm{SCD}}(\hat{g}_{\mathrm{SCD}})-\widehat{R}_{\mathrm{SCD}}(\hat{g}_{\mathrm{SCD}}))+(\widehat{R}_{\mathrm{SCD}}(\hat{g}_{\mathrm{SCD}})-R_{\mathrm{SCD}}(\hat{g}_{\mathrm{SCD}}))+(R_{\mathrm{SCD}}(\hat{g}_{\mathrm{SCD}})-R(g^*))\\
\le&2\underset{g\in\mathcal{G}}{\sup}|R(g)-\widehat{R}_{\mathrm{SCD}}(g)|+2\underset{g\in\mathcal{G}}{\sup}|\bar{R}_{\mathrm{SCD}}(g)-\widehat{R}_{\mathrm{SCD}}(g)|+(R_{\mathrm{SCD}}(\hat{g}_{\mathrm{SCD}})-R(g^*))\\
\le&4\underset{g\in\mathcal{G}}{\sup}|R(g)-\widehat{R}_{\mathrm{SCD}}(g)|+2\underset{g\in\mathcal{G}}{\sup}|\bar{R}_{\mathrm{SCD}}(g)-\widehat{R}_{\mathrm{SCD}}(g)|\\
\le&24L_{\ell}\mathfrak{R}_n(\mathcal{G})+(4\pi_+^2+4\pi_-^2+10)C_{\ell}\sqrt{\frac{\log2/\delta}{2n}}\\
&+\frac{4C_{\ell}}{n}\sum_{i=1}^{n}(|\bar{\pi}_+\bar{c}_i-\pi_+c_i|+|\bar{\pi}_-\bar{c}_i-\pi_-c_i|+|\bar{s}_i-s_i|)\\
&+4C_{\ell}(|\bar{\pi}_+^2-\pi_+^2|+|\bar{\pi}_-^2-\pi_-^2|+|\bar{\pi}_+-\pi_+|).
\end{align*}
The second inequality is derived based on the proof of Theorem~\ref{thm:estimation error}.
\end{proof}

\subsection{Proof of Theorem~\ref{thm:consistency}}
Following the analysis of~\cite{confdiff}, we define $\mathfrak{D}_n^+(g)=\{\mathcal{D}_n|\widehat{A}(g)\ge{0}\hspace{1mm}\cap\hspace{1mm}\widehat{B}(g)\ge{0}\hspace{1mm}\cap\hspace{1mm}\widehat{C}(g)\ge{0}\hspace{1mm}\cap\hspace{1mm}\widehat{D}(g)\ge{0}\}$, $\mathfrak{D}_n^-(g)=\{\mathcal{D}_n|\widehat{A}(g)\le{0}\hspace{1mm}\cup\hspace{1mm}\widehat{B}(g)\le{0}\hspace{1mm}\cup\hspace{1mm}\widehat{C}(g)\le{0}\hspace{1mm}\cup\hspace{1mm}\widehat{D}(g)\le{0}\}$. 
To prove Theorem~\ref{thm:consistency}, we firstly present the following lemma.
\begin{lemma}\label{lemma:upper exp}
For the probability measure of $\mathfrak{D}_n^-(g)$, the following inequality holds:
\begin{align}
\mathbb{P}(\mathfrak{D}_n^-(g))\le&\exp\left(-\frac{4a^2n}{(4\pi_++1)^2C_{\ell}^2}\right)+\exp\left(-\frac{4b^2n}{(4\pi_-+1)^2C_{\ell}^2}\right)\notag\\
&+\exp\left(-\frac{4c^2n}{(4\pi_++1)^2C_{\ell}^2}\right)+\exp\left(-\frac{4d^2n}{(4\pi_-+1)^2C_{\ell}^2}\right).
\end{align}
\end{lemma}

\begin{proof}
We observe that
\begin{align*}
p(\mathcal{D}_n)=&p(\bm{x}_1,\bm{x}_1')\cdots{p(\bm{x}_n,\bm{x}_n')}\\
=&p(\bm{x}_1)\cdots{p(\bm{x}_n)p(\bm{x}_1')}\cdots{p(\bm{x}_n')}.
\end{align*}
Therefore, the probability measure of $\mathfrak{D}_n^-(g)$ is 
\begin{align*}
\mathbb{P}(\mathfrak{D}_n^-(g))=&\int_{\mathcal{D}_n\in\mathfrak{D}_n^-(g)}p(\mathcal{D}_n)\hspace{1mm}\mathrm{d}\mathcal{D}_n\\
=&\int_{\mathcal{D}_n\in\mathfrak{D}_n^-(g)}p(\mathcal{D}_n)\hspace{1mm}\mathrm{d}\bm{x}_1\cdots{\mathrm{d}\bm{x}_n}\mathrm{d}\bm{x}_1'\cdots{\mathrm{d}\bm{x}_n'}.
\end{align*}
When one SconfConfDiff data pair in $\mathcal{D}_n$ is replaced, the changes in $\widehat{A}(g)$ and $\widehat{C}(g)$ are bounded above by $(4\pi_++1)C_{\ell}/2n$, and the changes in $\widehat{B}(g)$ and $\widehat{D}(g)$ are bounded above by $(4\pi_-+1)C_{\ell}/2n$. From McDiarmid's inequality~\cite{McDiarmid}, the following inequalities hold.
\[\mathbb{P}(\mathbb{E}[\widehat{A}(g)]-\widehat{A}(g)\ge{a})\le\exp\left(-\frac{4a^2n}{(4\pi_++1)^2C_{\ell}^2}\right),\]
\[\mathbb{P}(\mathbb{E}[\widehat{B}(g)]-\widehat{B}(g)\ge{b})\le\exp\left(-\frac{4b^2n}{(4\pi_-+1)^2C_{\ell}^2}\right),\]
\[\mathbb{P}(\mathbb{E}[\widehat{C}(g)]-\widehat{C}(g)\ge{c})\le\exp\left(-\frac{4c^2n}{(4\pi_++1)^2C_{\ell}^2}\right),\]
\[\mathbb{P}(\mathbb{E}[\widehat{D}(g)]-\widehat{D}(g)\ge{d})\le\exp\left(-\frac{4d^2n}{(4\pi_-+1)^2C_{\ell}^2}\right).\]
Furthermore, 
\begin{align*}
\mathbb{P}(\mathfrak{D}_n^-(g))\le&\mathbb{P}(\widehat{A}(g)\le{0})+\mathbb{P}(\widehat{B}(g)\le{0})+\mathbb{P}(\widehat{C}(g)\le{0})+\mathbb{P}(\widehat{D}(g)\le{0})\\
\le&\mathbb{P}(\widehat{A}(g)\le\mathbb{E}[\widehat{A}(g)]-a)+\mathbb{P}(\widehat{B}(g)\le\mathbb{E}[\widehat{B}(g)]-b)\\
&+\mathbb{P}(\widehat{C}(g)\le\mathbb{E}[\widehat{C}(g)]-c)+\mathbb{P}(\widehat{D}(g)\le\mathbb{E}[\widehat{D}(g)]-d)\\
=&\mathbb{P}(\mathbb{E}[\widehat{A}(g)]-\widehat{A}(g)\ge{a})+\mathbb{P}(\mathbb{E}[\widehat{B}(g)]-\widehat{B}(g)\ge{b})\\
&+\mathbb{P}(\mathbb{E}[\widehat{C}(g)]-\widehat{C}(g)\ge{c})+\mathbb{P}(\mathbb{E}[\widehat{D}(g)]-\widehat{D}(g)\ge{d})\\
\le&\exp\left(-\frac{4a^2n}{(4\pi_++1)^2C_{\ell}^2}\right)+\exp\left(-\frac{4b^2n}{(4\pi_-+1)^2C_{\ell}^2}\right)\\
&+\exp\left(-\frac{4c^2n}{(4\pi_++1)^2C_{\ell}^2}\right)+\exp\left(-\frac{4d^2n}{(4\pi_-+1)^2C_{\ell}^2}\right),
\end{align*}
which concludes the proof.
\end{proof}
Then, the proof of Theorem~\ref{thm:consistency} is given.
\begin{proof}
\begin{align*}
\mathbb{E}&[\widetilde{R}_{\mathrm{SCD}}(g)]-R(g)\\
=&\mathbb{E}[\widetilde{R}_{\mathrm{SCD}}(g)-\widehat{R}_{\mathrm{SCD}}(g)]\\
=&\int_{\mathcal{D}_n\in\mathfrak{D}_n^+(g)}(\widetilde{R}_{\mathrm{SCD}}(g)-\widehat{R}_{\mathrm{SCD}}(g))p(\mathcal{D}_n)\hspace{1mm}\mathrm{d}\mathcal{D}_n+\int_{\mathcal{D}_n\in\mathfrak{D}_n^-(g)}(\widetilde{R}_{\mathrm{SCD}}(g)-\widehat{R}_{\mathrm{SCD}}(g))p(\mathcal{D}_n)\hspace{1mm}\mathrm{d}\mathcal{D}_n\\
=&\int_{\mathcal{D}_n\in\mathfrak{D}_n^-(g)}(\widetilde{R}_{\mathrm{SCD}}(g)-\widehat{R}_{\mathrm{SCD}}(g))p(\mathcal{D}_n)\hspace{1mm}\mathrm{d}\mathcal{D}_n\ge{0}.
\end{align*}
In addition, 
\begin{align*}
\mathbb{E}&[\widetilde{R}_{\mathrm{SCD}}(g)]-R(g)\\
=&\int_{\mathcal{D}_n\in\mathfrak{D}_n^-(g)}(\widetilde{R}_{\mathrm{SCD}}(g)-\widehat{R}_{\mathrm{SCD}}(g))p(\mathcal{D}_n)\hspace{1mm}\mathrm{d}\mathcal{D}_n\\
\le&\underset{\mathcal{D}_n\in\mathfrak{D}_n^-(g)}{\sup}(\widetilde{R}_{\mathrm{SCD}}(g)-\widehat{R}_{\mathrm{SCD}}(g))\int_{\mathcal{D}_n\in\mathfrak{D}_n^-(g)}p(\mathcal{D}_n)\hspace{1mm}\mathrm{d}\mathcal{D}_n\\
=&\underset{\mathcal{D}_n\in\mathfrak{D}_n^-(g)}{\sup}(\widetilde{R}_{\mathrm{SCD}}(g)-\widehat{R}_{\mathrm{SCD}}(g))\mathbb{P}(\mathfrak{D}_n^-(g))\\
=&\underset{\mathcal{D}_n\in\mathfrak{D}_n^-(g)}{\sup}(f(\widehat{A}(g))+f(\widehat{B}(g))+f(\widehat{C}(g))+f(\widehat{D}(g))\\
&\hspace{2cm}-\widehat{A}(g)-\widehat{B}(g)-\widehat{C}(g)-\widehat{D}(g))\mathbb{P}(\mathfrak{D}_n^-(g))\\
\le&\underset{\mathcal{D}_n\in\mathfrak{D}_n^-(g)}{\sup}(L_f|\widehat{A}(g)|+L_f|\widehat{B}(g)|+L_f|\widehat{C}(g)|+L_f|\widehat{D}(g)|\\
&\hspace{2cm}+|\widehat{A}(g)|+|\widehat{B}(g)|+|\widehat{C}(g)|+|\widehat{D}(g)|)\mathbb{P}(\mathfrak{D}_n^-(g))\\
=&\underset{\mathcal{D}_n\in\mathfrak{D}_n^-(g)}{\sup}\frac{L_f+1}{2n}\biggl\{\left|\sum_{i=1}^{n}(2\pi_+(\pi_+-c_i)+\pi_--s_i)\ell(g(\bm{x}_i),+1)\right|\\
&\hspace{2cm}+\left|\sum_{i=1}^{n}(2\pi_-(\pi_--c_i)+\pi_+-s_i)\ell(g(\bm{x}_i'),-1)\right|\\
&\hspace{2cm}+\left|\sum_{i=1}^{n}(2\pi_+(\pi_++c_i)+\pi_--s_i)\ell(g(\bm{x}_i'),+1)\right|\\
&\hspace{2cm}+\left|\sum_{i=1}^{n}(2\pi_-(\pi_-+c_i)+\pi_+-s_i)\ell(g(\bm{x}_i),-1)\right|\biggr\}\mathbb{P}(\mathfrak{D}_n^-(g))\\
\le&\underset{\mathcal{D}_n\in\mathfrak{D}_n^-(g)}{\sup}\frac{L_f+1}{2n}\biggl\{\sum_{i=1}^{n}|(2\pi_+(\pi_+-c_i)+\pi_--s_i)\ell(g(\bm{x}_i),+1)|\\
&\hspace{2cm}+\sum_{i=1}^{n}|(2\pi_-(\pi_--c_i)+\pi_+-s_i)\ell(g(\bm{x}_i'),-1)|\\
&\hspace{2cm}+\sum_{i=1}^{n}|(2\pi_+(\pi_++c_i)+\pi_--s_i)\ell(g(\bm{x}_i'),+1)|\\
&\hspace{2cm}+\sum_{i=1}^{n}|(2\pi_-(\pi_-+c_i)+\pi_+-s_i)\ell(g(\bm{x}_i),-1)|\biggr\}\mathbb{P}(\mathfrak{D}_n^-(g))\\
=&\underset{\mathcal{D}_n\in\mathfrak{D}_n^-(g)}{\sup}\frac{L_f+1}{2n}\biggl\{\sum_{i=1}^{n}|((2\pi_+(\pi_+-c_i))+\pi_--s_i)\ell(g(\bm{x}_i),+1)|\\
&\hspace{2cm}+|((2\pi_-(\pi_--c_i))+\pi_+-s_i)\ell(g(\bm{x}_i'),-1)|\\
&\hspace{2cm}+|(2\pi_+(\pi_++c_i)+\pi_--s_i)\ell(g(\bm{x}_i'),+1)|\\
&\hspace{2cm}+|(2\pi_-(\pi_-+c_i)+\pi_+-s_i)\ell(g(\bm{x}_i),-1)|\biggr\}\mathbb{P}(\mathfrak{D}_n^-(g))\\
\le&\underset{\mathcal{D}_n\in\mathfrak{D}_n^-(g)}{\sup}\frac{(L_f+1)C_{\ell}}{2n}\biggl\{\sum_{i=1}^{n}|2\pi_+(\pi_+-c_i)+\pi_--s_i|+|2\pi_-(\pi_--c_i)+\pi_+-s_i|\\
&\hspace{1cm}+|2\pi_+(\pi_++c_i)+\pi_--s_i|+|2\pi_-(\pi_-+c_i)+\pi_+-s_i|\biggr\}\mathbb{P}(\mathfrak{D}_n^-(g)).\\
\end{align*}
Then, based on the derivation of Eq.~\ref{bound of L_SCD}, the following inequality holds:
\begin{align*}
|&2\pi_+(\pi_+-c_i)+\pi_--s_i|+|2\pi_-(\pi_--c_i)+\pi_+-s_i|\\
&+|2\pi_+(\pi_++c_i)+\pi_--s_i|+|2\pi_-(\pi_-+c_i)+\pi_+-s_i|\le{6}.
\end{align*}
Thus, we have
\[\mathbb{E}[\widetilde{R}_{\mathrm{SCD}}(g)]-R(g)\le3(L_f+1)C_{\ell}\Delta.\]
We now derive the upper bound of $|\widetilde{R}_{\mathrm{SCD}}(g)-\mathbb{E}[\widetilde{R}_{\mathrm{SCD}}(g)]|$.
Let 
\[\widetilde{R}_{\mathrm{SCD}}(g)=f(\widehat{A}(g))+f(\widehat{B}(g))+f(\widehat{C}(g))+f(\widehat{D}(g)),\]
\[\widetilde{\dot{R}}_{\mathrm{SCD}}(g)=f(\widehat{\dot{A}}(g))+f(\widehat{\dot{B}}(g))+f(\widehat{\dot{C}}(g))+f(\widehat{\dot{D}}(g))\]
denote the empirical risks when one SconfConfDiff data pair in $\mathcal{D}_n$ is replaced. 
Then, the following inequality holds.
\begin{align*}
&|\widetilde{R}_{\mathrm{SCD}}(g)-\widetilde{\dot{R}}_{\mathrm{SCD}}(g)|\\
&=|f(\widehat{A}(g))-f(\widehat{\dot{A}}(g))+f(\widehat{B}(g))-f(\widehat{\dot{B}}(g))\\
&\hspace{1cm}+f(\widehat{C}(g))-f(\widehat{\dot{C}}(g))+f(\widehat{D}(g))-f(\widehat{\dot{D}}(g))|\\
&\le|f(\widehat{A}(g))-f(\widehat{\dot{A}}(g))|+|f(\widehat{B}(g))-f(\widehat{\dot{B}}(g))|\\
&\hspace{1cm}+|f(\widehat{C}(g))-f(\widehat{\dot{C}}(g))|+|f(\widehat{D}(g))-f(\widehat{\dot{D}}(g))|\\
&\le{L_f}|\widehat{A}(g)-\widehat{\dot{A}}(g)|+L_f|\widehat{B}(g)-\widehat{\dot{B}}(g)|+L_f|\widehat{C}(g)-\widehat{\dot{C}}(g)|+L_f|\widehat{D}(g)-\widehat{\dot{D}}(g)|\\
&\le{L_f}\biggl\{\frac{(4\pi_++1)C_{\ell}}{2n}+\frac{(4\pi_-+1)C_{\ell}}{2n}+\frac{(4\pi_++1)C_{\ell}}{2n}+\frac{(4\pi_-+1)C_{\ell}}{2n}\biggr\}\\
&\le\frac{6L_fC_{\ell}}{n},
\end{align*}
where the second inequality is derived based on the Lipschitz continuity of the risk correction function. Therefore, the change of $\widetilde{R}_{\mathrm{SCD}}(g)$ is bounded above by $6L_fC_{\ell}/n$. From McDiarmid's inequality~\cite{McDiarmid}, the following inequalities hold at least $1-\delta/2$.
\[\widetilde{R}_{\mathrm{SCD}}(g)-\mathbb{E}[\widetilde{R}_{\mathrm{SCD}}(g)]\le6L_fC_{\ell}\sqrt{\frac{\log2/\delta}{2n}},\]
\[\mathbb{E}[\widetilde{R}_{\mathrm{SCD}}(g)]-\widetilde{R}_{\mathrm{SCD}}(g)\le6L_fC_{\ell}\sqrt{\frac{\log2/\delta}{2n}}.\]
Therefore, we can derive the following inequality with probability at least $1-\delta$:
\[|\widetilde{R}_{\mathrm{SCD}}(g)-\mathbb{E}[\widetilde{R}_{\mathrm{SCD}}(g)]|\le6L_fC_{\ell}\sqrt{\frac{\log2/\delta}{2n}}.\]
Finally, 
\begin{align*}
|\widetilde{R}_{\mathrm{SCD}}(g)-R(g)|=&|\widetilde{R}_{\mathrm{SCD}}(g)-\mathbb{E}[\widetilde{R}_{\mathrm{SCD}}(g)]+\mathbb{E}[\widetilde{R}_{\mathrm{SCD}}(g)]-R(g)|\\
\le&|\widetilde{R}_{\mathrm{SCD}}(g)-\mathbb{E}[\widetilde{R}_{\mathrm{SCD}}(g)]|+|\mathbb{E}[\widetilde{R}_{\mathrm{SCD}}(g)]-R(g)|\\
=&|\widetilde{R}_{\mathrm{SCD}}(g)-\mathbb{E}[\widetilde{R}_{\mathrm{SCD}}(g)]|+\mathbb{E}[\widetilde{R}_{\mathrm{SCD}}(g)]-R(g)\\
\le&6L_fC_{\ell}\sqrt{\frac{\log2/\delta}{2n}}+3(L_f+1)C_{\ell}\Delta,
\end{align*}
with probability at least $1-\delta$.
\end{proof}

\subsection{Proof of Theorem~\ref{thm:corrected estimation error}}
\begin{proof}
By combining the probabilistic inequalities established in Theorem~\ref{thm:consistency} and Theorem~\ref{thm:estimation error}, each of which holds with probability at least $1-\delta/2$, we obtain the following inequality, which holds with probability at least $1-\delta$.
\begin{align*}
R(\tilde{g}_{\mathrm{SCD}})-R(g^*)=&(R(\tilde{g}_{\mathrm{SCD}})-\widetilde{R}_{\mathrm{SCD}}(\tilde{g}_{\mathrm{SCD}}))+(\widetilde{R}_{\mathrm{SCD}}(\tilde{g}_{\mathrm{SCD}})-\widetilde{R}_{\mathrm{SCD}}(\hat{g}_{\mathrm{SCD}}))\\
&+(\widetilde{R}_{\mathrm{SCD}}(\hat{g}_{\mathrm{SCD}})-R(\hat{g}_{\mathrm{SCD}}))+(R(\hat{g}_{\mathrm{SCD}})-R(g^*))\\
\le&|R(\tilde{g}_{\mathrm{SCD}})-\widetilde{R}_{\mathrm{SCD}}(\tilde{g}_{\mathrm{SCD}})|+|\widetilde{R}_{\mathrm{SCD}}(\hat{g}_{\mathrm{SCD}})-R(\hat{g}_{\mathrm{SCD}})|\\
&+(R(\hat{g}_{\mathrm{SCD}})-R(g^*))\\
\le&(12L_f+2\pi_+^2+2\pi_-^2+5)C_{\ell}\sqrt{\frac{\log4/\delta}{2n}}+6(L_f+1)C_{\ell}\Delta+12L_{\ell}\mathfrak{R}_n(\mathcal{G})
\end{align*}
The first inequality is deduced from the definition of $\tilde{g}_{\mathrm{SCD}}$, and the proof of Theorem~\ref{thm:corrected estimation error} is completed.
\end{proof}

\section{Analysis of Risk Estimator by Convex Combination}\label{appendix_convex}

\subsection{Estimation Error Bound}
We make the same assumptions as in Section~\ref{sec:bound}, and define $\hat{g}_{\mathrm{SCD_{convex}}}=\arg\min_{g\in\mathcal{G}}\widehat{R}_{\mathrm{SCD_{convex}}}(g)$. 
\begin{theorem}\label{thm:estimation error sum}
For any $\delta>0$, the following inequality holds with probability at least $1-\delta$:
\begin{align}
&R(\hat{g}_{\mathrm{SCD_{convex}}})-R(g^*)\notag\\
&\le\frac{4\gamma+8(1-\gamma)|\pi_+-\pi_-|}{|\pi_+-\pi_-|}L_{\ell}\mathfrak{R}_n(\mathcal{G})+\frac{2\gamma+4(1-\gamma)|\pi_+-\pi_-|}{|\pi_+-\pi_-|}C_{\ell}\sqrt{\frac{\log4/\delta}{2n}},\label{error bound sum}
\end{align}
where $\gamma\in[0,1]$.
\end{theorem}
Then, the proof of Theorem~\ref{thm:estimation error sum} is given. For convenience, we make the following notation:
\begin{align*}
\mathcal{L}_{\mathrm{Sconf}}(g;\bm{x}_i,\bm{x}_i')=&\frac{(s(\bm{x},\bm{x}')-\pi_-)(\ell(g(\bm{x}),+1)+\ell(g(\bm{x}'),+1))}{2(\pi_+-\pi_-)}\\
&+\frac{(\pi_+-s(\bm{x},\bm{x}'))(\ell(g(\bm{x}),-1)+\ell(g(\bm{x}'),-1))}{2(\pi_+-\pi_-)},
\end{align*}
\begin{align*}
\mathcal{L}_{\mathrm{CD}}(g;\bm{x}_i,\bm{x}_i')=&\frac{(\pi_+-c_i)\ell(g(\bm{x}_i),+1)}{2}+\frac{(\pi_--c_i)\ell(g(\bm{x}_i'),-1)}{2}\\
&+\frac{(\pi_++c_i)\ell(g(\bm{x}_i'),+1)}{2}+\frac{(\pi_-+c_i)\ell(g(\bm{x}_i),-1)}{2}.
\end{align*}
\begin{lemma}\label{lemma:sup unbiased sum}
The following inequality holds with probability at least $1-\delta$:
\begin{align*}
&\underset{g\in\mathcal{G}}{\sup}|R(g)-\widehat{R}_{\mathrm{SCD_{convex}}}(g)|\\
&\le\frac{2\gamma+4(1-\gamma)|\pi_+-\pi_-|}{|\pi_+-\pi_-|}L_{\ell}\mathfrak{R}_n(\mathcal{G})+\frac{\gamma+2(1-\gamma)|\pi_+-\pi_-|}{|\pi_+-\pi_-|}C_{\ell}\sqrt{\frac{\log4/\delta}{2n}}.
\end{align*}
\end{lemma}
\begin{proof}
With probability at least $1-\delta$, the following inequality holds (cf. Lemma 7 in~\cite{sconf}, Lemma 5 in~\cite{confdiff}).
\begin{align*}
|R(g)-\widehat{R}_{\mathrm{SCD_{convex}}}(g)|=&|\gamma(R(g)-\widehat{R}_{\mathrm{Sconf}}(g))+(1-\gamma)(R(g)-\widehat{R}_{\mathrm{CD}}(g))|\\
\le&\gamma|R(g)-\widehat{R}_{\mathrm{Sconf}}(g)|+(1-\gamma)|R(g)-\widehat{R}_{\mathrm{CD}}(g)|\\
\le&\frac{2\gamma{L}_{\ell}}{|\pi_+-\pi_-|}\mathfrak{R}_n(\mathcal{G})+\frac{\gamma{C}_{\ell}}{|\pi_+-\pi_-|}\sqrt{\frac{\log4/\delta}{2n}}\\
&+4(1-\gamma)L_{\ell}\mathfrak{R}_n(\mathcal{G})+2(1-\gamma)C_{\ell}\sqrt{\frac{\log4/\delta}{2n}}\\
\le&\frac{2\gamma+4(1-\gamma)|\pi_+-\pi_-|}{|\pi_+-\pi_-|}L_{\ell}\mathfrak{R}_n(\mathcal{G})\\
&+\frac{\gamma+2(1-\gamma)|\pi_+-\pi_-|}{|\pi_+-\pi_-|}C_{\ell}\sqrt{\frac{\log4/\delta}{2n}}.
\end{align*}
The second inequality is obtained from the property that the probabilistic inequalities for the estimation error bounds of $\widehat{R}_{\mathrm{Sconf}}(g)$ and $\widehat{R}_{\mathrm{CD}}(g)$ each hold with probability at least $1-\delta/2$.
\end{proof}
Then, the proof of Theorem~\ref{thm:estimation error sum} is provided as follows.
\begin{proof}
With probability at least $1-\delta$, the following inequality holds.
\begin{align*}
R&(\hat{g}_{\mathrm{SCD_{convex}}})-R(g^*)\\
=&(R(\hat{g}_{\mathrm{SCD_{convex}}})-\widehat{R}(\hat{g}_{\mathrm{SCD_{convex}}}))+(\widehat{R}(\hat{g}_{\mathrm{SCD_{convex}}})-\widehat{R}_{\mathrm{SCD_{convex}}}(g^*))\\
&+(\widehat{R}_{\mathrm{SCD_{convex}}}(g^*)-R(g^*))\\
\le&(R(\hat{g}_{\mathrm{SCD_{convex}}})-\widehat{R}(\hat{g}_{\mathrm{SCD_{convex}}}))+(\widehat{R}_{\mathrm{SCD_{convex}}}(g^*)-R(g^*))\\
\le&|R(\hat{g}_{\mathrm{SCD_{convex}}})-\widehat{R}(\hat{g}_{\mathrm{SCD_{convex}}})|+|\widehat{R}_{\mathrm{SCD_{convex}}}(g^*)-R(g^*)|\\
\le&2\underset{g\in\mathcal{G}}{\sup}|R(g)-\widehat{R}_{\mathrm{SCD_{convex}}}(g)|\\
\le&\frac{4\gamma+8(1-\gamma)|\pi_+-\pi_-|}{|\pi_+-\pi_-|}L_{\ell}\mathfrak{R}_n(\mathcal{G})+\frac{2\gamma+4(1-\gamma)|\pi_+-\pi_-|}{|\pi_+-\pi_-|}C_{\ell}\sqrt{\frac{\log4/\delta}{2n}}.
\end{align*}
The first inequality is deduced from the definition of $\hat{g}_{\mathrm{SCD_{convex}}}$.
\end{proof}

\subsection{Robustness of The Convex Risk Estimator}
Similar to Section~\ref{sec:robust}, this section provides an analysis of noisy similarity-confidence and noisy confidence-difference on the learning procedure. 
While the robustness of convex combinations has not been analyzed under the noisy class prior probability in the empirical risk of Sconf learning, and such analysis tends to be intractable, we limit our analysis here to the impact of label noise.

Let $\bar{R}_{\mathrm{SCD_{convex}}}$ denote the empirical risk calculated from $\bar{\mathcal{D}}_n=\{(\bm{x}_i,\bm{x}_i'),\bar{s}_i,\bar{c}_i\}_{i=1}^{n}$, and let $\bar{g}_{\mathrm{SCD_{convex}}}=\arg\min_{g\in\mathcal{G}}\bar{R}_{\mathrm{SCD_{convex}}}(g)$ be the minimizer of $\bar{R}_{\mathrm{SCD_{convex}}}$. 
Under this setting, the estimation error bound is given by the following theorem.
\begin{theorem}\label{thm:robust estimation error sum}
Based on the assumption above, for any $\delta>0$, the following inequality holds with probability at least $1-\delta$:
\begin{align}
R&(\bar{g}_{\mathrm{SCD_{convex}}})-R(g^*)\notag\\
\le&\frac{8\gamma+16(1-\gamma)|\pi_+-\pi_-|}{|\pi_+-\pi_-|}L_{\ell}\mathfrak{R}_n(\mathcal{G})+\frac{4\gamma+8(1-\gamma)|\pi_+-\pi_-|}{|\pi_+-\pi_-|}C_{\ell}\sqrt{\frac{\log4/\delta}{2n}}\notag\\
&+\frac{2\gamma{C}_{\ell}\sum_{i=1}^{n}|\bar{s}_i-s_i|}{n|\pi_+-\pi_-|}+\frac{4(1-\gamma){C}_{\ell}\sum_{i=1}^{n}|\bar{c}_i-c_i|}{n}.\label{robustness sum}
\end{align}
\end{theorem}
We now prove Theorem~\ref{thm:robust estimation error sum} (cf. Theorem 6 in~\cite{sconf}, Theorem 4 in~\cite{confdiff}).
\begin{proof}
The following inequality holds:
\begin{align*}
|\bar{R}_{\mathrm{SCD_{convex}}}(g)-\widehat{R}_{\mathrm{SCD_{convex}}}(g)|=&|\gamma(\bar{R}_{\mathrm{Sconf}}(g)-\widehat{R}_{\mathrm{Sconf}}(g))+(1-\gamma)(\bar{R}_{\mathrm{CD}}(g)-\widehat{R}_{\mathrm{CD}}(g))|\\
\le&\gamma|\bar{R}_{\mathrm{Sconf}}(g)-\widehat{R}_{\mathrm{Sconf}}(g)|+(1-\gamma)|\bar{R}_{\mathrm{CD}}(g)-\widehat{R}_{\mathrm{CD}}(g)|\\
\le&\frac{\gamma{C}_{\ell}\sum_{i=1}^{n}|\bar{s}_i-s_i|}{n|\pi_+-\pi_-|}+\frac{2(1-\gamma){C}_{\ell}\sum_{i=1}^{n}|\bar{c}_i-c_i|}{n}.
\end{align*}
Thus, with probability at least $1-\delta$, the following inequality can be derived.
\begin{align*}
R&(\bar{g}_{\mathrm{SCD_{convex}}})-R(g^*)\\
=&(R(\bar{g}_{\mathrm{SCD_{convex}}})-\widehat{R}_{\mathrm{SCD_{convex}}}(\bar{g}_{\mathrm{SCD_{convex}}}))\\
&+(\widehat{R}_{\mathrm{SCD_{convex}}}(\bar{g}_{\mathrm{SCD_{convex}}})-\bar{R}_{\mathrm{SCD_{convex}}}(\bar{g}_{\mathrm{SCD_{convex}}}))\\
&+(\bar{R}_{\mathrm{SCD_{convex}}}(\bar{g}_{\mathrm{SCD_{convex}}})-\bar{R}_{\mathrm{SCD_{convex}}}(\hat{g}_{\mathrm{SCD_{convex}}}))\\
&+(\bar{R}_{\mathrm{SCD_{convex}}}(\hat{g}_{\mathrm{SCD_{convex}}})-\widehat{R}_{\mathrm{SCD_{convex}}}(\hat{g}_{\mathrm{SCD_{convex}}}))\\
&+(\widehat{R}_{\mathrm{SCD_{convex}}}(\hat{g}_{\mathrm{SCD_{convex}}})-R_{\mathrm{SCD_{convex}}}(\hat{g}_{\mathrm{SCD_{convex}}}))\\
&+(R_{\mathrm{SCD_{convex}}}(\hat{g}_{\mathrm{SCD_{convex}}})-R(g^*))\\
\le&2\underset{g\in\mathcal{G}}{\sup}|R(g)-\widehat{R}_{\mathrm{SCD_{convex}}}(g)|+2\underset{g\in\mathcal{G}}{\sup}|\bar{R}_{\mathrm{SCD_{convex}}}(g)-\widehat{R}_{\mathrm{SCD_{convex}}}(g)|\\
&+(R_{\mathrm{SCD_{convex}}}(\hat{g}_{\mathrm{SCD_{convex}}})-R(g^*))\\
\le&4\underset{g\in\mathcal{G}}{\sup}|R(g)-\widehat{R}_{\mathrm{SCD_{convex}}}(g)|+2\underset{g\in\mathcal{G}}{\sup}|\bar{R}_{\mathrm{SCD_{convex}}}(g)-\widehat{R}_{\mathrm{SCD_{convex}}}(g)|\\
\le&\frac{8\gamma+16(1-\gamma)|\pi_+-\pi_-|}{|\pi_+-\pi_-|}L_{\ell}\mathfrak{R}_n(\mathcal{G})+\frac{4\gamma+8(1-\gamma)|\pi_+-\pi_-|}{|\pi_+-\pi_-|}C_{\ell}\sqrt{\frac{\log4/\delta}{2n}}\\
&+\frac{2\gamma{C}_{\ell}\sum_{i=1}^{n}|\bar{s}_i-s_i|}{n|\pi_+-\pi_-|}+\frac{4(1-\gamma){C}_{\ell}\sum_{i=1}^{n}|\bar{c}_i-c_i|}{n}.
\end{align*}
The second and third inequalities are derived from the proof of Theorem~\ref{thm:estimation error sum}.
\end{proof}

\subsection{The Correction of The Risk Estimator}
Similarly to Section~\ref{sec:risk correction}, the risk estimator in Eq.~\eqref{unbiased risk sum} may take negative values and therefore does not satisfy the non-negativity of the loss function. To overcome this difficulty, risk correction approach with risk correction functions is proposed, such as the ReLU function $f(z)=\max(0,z)$ and the absolute value function $f(z)=|z|$. 
Consequently, the corrected risk estimator for SconfConfDiff classification by convex combination is given as follows:
\begin{align}
\widetilde{R}_{\mathrm{SCD_{convex}}}(g)=\gamma\widetilde{R}_{\mathrm{Sconf}}(g)+(1-\gamma)\widetilde{R}_{\mathrm{CD}}(g),\label{corrected risk sum}
\end{align}
where
\begin{align*}
\widetilde{R}_{\mathrm{Sconf}}(g)=&f\left(\sum_{i=1}^{n}\frac{(s_i-\pi_-)(\ell(g(\bm{x}_i),+1)+\ell(g(\bm{x}_i'),+1))}{2n(\pi_+-\pi_-)}\right)\\
&+f\left(\sum_{i=1}^{n}\frac{(\pi_+-s_i)(\ell(g(\bm{x}_i),-1)+\ell(g(\bm{x}_i'),-1))}{2n(\pi_+-\pi_-)}\right),
\end{align*}
\begin{align*}
\widetilde{R}_{\mathrm{CD}}(g)=\frac{1}{2n}\biggl\{&f\left(\sum_{i=1}^{n}(\pi_+-c_i)\ell(g(\bm{x}_i),+1)\right)+f\left(\sum_{i=1}^{n}(\pi_--c_i)\ell(g(\bm{x}_i'),-1)\right)\\
+&f\left(\sum_{i=1}^{n}(\pi_++c_i)\ell(g(\bm{x}_i'),+1)\right)+f\left(\sum_{i=1}^{n}(\pi_-+c_i)\ell(g(\bm{x}_i),-1)\right)\biggr\}.
\end{align*}
Here, it is assumed that
\[\mathbb{E}\left[\sum_{i=1}^{n}\frac{(s_i-\pi_-)(\ell(g(\bm{x}_i),+1)+\ell(g(\bm{x}_i'),+1))}{2n(\pi_+-\pi_-)}\right]\ge\alpha,\]
\[\mathbb{E}\left[\sum_{i=1}^{n}\frac{(\pi_+-s_i)(\ell(g(\bm{x}_i),-1)+\ell(g(\bm{x}_i'),-1))}{2n(\pi_+-\pi_-)}\right]\ge\beta,\]
\[\mathbb{E}\left[\sum_{i=1}^{n}\frac{(\pi_+-c_i)\ell(g(\bm{x}_i),+1)}{2n}\right]\ge{a},\hspace{5mm}\mathbb{E}\left[\sum_{i=1}^{n}\frac{(\pi_--c_i)\ell(g(\bm{x}_i'),-1)}{2n}\right]\ge{b}, \]
\[\mathbb{E}\left[\sum_{i=1}^{n}\frac{(\pi_++c_i)\ell(g(\bm{x}_i'),+1)}{2n}\right]\ge{c},\hspace{5mm}\mathbb{E}\left[\sum_{i=1}^{n}\frac{(\pi_-+c_i)\ell(g(\bm{x}_i),-1)}{2n}\right]\ge{d}.\]
We define $\tilde{g}_{\mathrm{SCD_{convex}}}=\arg\min_{g\in\mathcal{G}}\widetilde{R}_{\mathrm{SCD_{convex}}}(g)$.
\begin{theorem}\label{thm:corrected estimation error sum}
Based on the assumption above, for any $\delta>0$, the following inequality holds with probability at least $1-\delta$:
\begin{align}
R&(\tilde{g}_{\mathrm{SCD_{convex}}})-R(g^*)\notag\\
\le&\frac{2\gamma+4(1-\gamma)|\pi_+-\pi_-|}{|\pi_+-\pi_-|}(L_f+1)C_{\ell}\sqrt{\frac{\log8/\delta}{2n}}\notag\\
&+\frac{2\gamma\Delta_{\mathrm{Sconf}}+4(1-\gamma)\Delta_{\mathrm{CD}}|\pi_+-\pi_-|}{|\pi_+-\pi_-|}(L_f+1)C_{\ell}+\frac{4\gamma+8(1-\gamma)|\pi_+-\pi_-|}{|\pi_+-\pi_-|}L_{\ell}\mathfrak{R}_n(\mathcal{G}),
\end{align}
where
\[\Delta_{\mathrm{Sconf}}=\exp\left(-\frac{(\pi_+-\pi_-)^2n}{2C_{\ell}^2}\right)(\exp(\alpha^2)+\exp(\beta^2)),\]
\[\Delta_{\mathrm{CD}}=\exp\left(-\frac{2a^2n}{C_{\ell}^2}\right)+\exp\left(-\frac{2b^2n}{C_{\ell}^2}\right)+\exp\left(-\frac{2c^2n}{C_{\ell}^2}\right)+\exp\left(-\frac{2d^2n}{C_{\ell}^2}\right).\]
\end{theorem}
Then, we prove Theorem~\ref{thm:corrected estimation error sum} (cf. Theorem 7 in~\cite{sconf}, Theorem 5 in~\cite{confdiff}). 
\begin{proof}
The following inequality holds with the probability at least $1-\delta/2$.
\begin{align}
|&\widetilde{R}_{\mathrm{SCD_{convex}}}(g)-R(g)|\notag \\
=&|\gamma(\widetilde{R}_{\mathrm{Sconf}}(g)-R(g))+(1-\gamma)(\widetilde{R}_{\mathrm{CD}}(g)-R(g))|\notag \\
\le&\gamma|\widetilde{R}_{\mathrm{Sconf}}(g)-R(g)|+(1-\gamma)|\widetilde{R}_{\mathrm{CD}}(g)-R(g)|\notag \\
\le&\gamma\left(\frac{L_{f}C_{\ell}}{|\pi_+-\pi_-|}\sqrt{\frac{\log8/\delta}{2n}}+\frac{(L_f+1)C_{\ell}}{|\pi_+-\pi_-|}\Delta_{\mathrm{Sconf}}\right)\notag \\
&+(1-\gamma)\left(2C_{\ell}L_f\sqrt{\frac{\log8/\delta}{2n}}+2(L_f+1)C_{\ell}\Delta_{\mathrm{CD}}\right)\notag \\
=&\frac{\gamma{L}_{f}+2(1-\gamma)L_f|\pi_+-\pi_-|}{|\pi_+-\pi_-|}C_{\ell}\sqrt{\frac{\log8/\delta}{2n}}+\frac{\gamma\Delta_{\mathrm{Sconf}}+2(1-\gamma)\Delta_{\mathrm{CD}}}{|\pi_+-\pi_-|}(L_f+1)C_{\ell}.\label{robust gen error sum}
\end{align}
The second inequality is obtained from the property that the probabilistic inequalities for the estimation error bounds of $\widetilde{R}_{\mathrm{Sconf}}(g)$ and $\widetilde{R}_{\mathrm{CD}}(g)$ each hold with probability at least $1-\delta/2$.

Therefore, if the probabilistic inequalities in Eq.~\eqref{error bound sum} and Eq.~\eqref{robust gen error sum} respectively hold with probability at least $1-\delta/2$, then the following inequality holds with probability at least $1-\delta$.
\begin{align*}
R&(\tilde{g}_{\mathrm{SCD_{convex}}})-R(g^*)\\
=&(R(\tilde{g}_{\mathrm{SCD_{convex}}})-\widetilde{R}_{\mathrm{SCD_{convex}}}(\tilde{g}_{\mathrm{SCD_{convex}}}))\\
&+(\widetilde{R}_{\mathrm{SCD_{convex}}}(\tilde{g}_{\mathrm{SCD_{convex}}})-\widetilde{R}_{\mathrm{SCD_{convex}}}(\hat{g}_{\mathrm{SCD_{convex}}}))\\
&+(\widetilde{R}_{\mathrm{SCD_{convex}}}(\hat{g}_{\mathrm{SCD_{convex}}})-R(\hat{g}_{\mathrm{SCD_{convex}}}))+(R(\hat{g}_{\mathrm{SCD_{convex}}})-R(g^*))\\
\le&|R(\tilde{g}_{\mathrm{SCD_{convex}}})-\widetilde{R}_{\mathrm{SCD_{convex}}}(\tilde{g}_{\mathrm{SCD_{convex}}})|+|\widetilde{R}_{\mathrm{SCD_{convex}}}(\hat{g}_{\mathrm{SCD_{convex}}})-R(\hat{g}_{\mathrm{SCD_{convex}}})|\\
&+(R(\hat{g}_{\mathrm{SCD_{convex}}})-R(g^*))\\
\le&\frac{2\gamma+4(1-\gamma)|\pi_+-\pi_-|}{|\pi_+-\pi_-|}(L_f+1)C_{\ell}\sqrt{\frac{\log8/\delta}{2n}}\notag\\
&+\frac{2\gamma\Delta_{\mathrm{Sconf}}+4(1-\gamma)\Delta_{\mathrm{CD}}|\pi_+-\pi_-|}{|\pi_+-\pi_-|}(L_f+1)C_{\ell}+\frac{4\gamma+8(1-\gamma)|\pi_+-\pi_-|}{|\pi_+-\pi_-|}L_{\ell}\mathfrak{R}_n(\mathcal{G}).
\end{align*}
The first inequality is deduced from the definition of $\tilde{g}_{\mathrm{SCD_{convex}}}$.
\end{proof}

\section{Disscussions}\label{appendix:limit}

\subsection{Use of Ordinary Labeled Data}
If an ordinary labeled dataset $\mathcal{D}_\mathrm{Ord}=\{(\bm{x}_i^\mathrm{Ord},y_i^\mathrm{Ord})\}_{i=1}^{n_\mathrm{Ord}}$ is available, it can be included in the training process together with the SconfConfDiff dataset. Let $\widehat{R}_\mathrm{Ord}(g)=\frac{1}{n_\mathrm{Ord}}\sum_{i=1}^{n_\mathrm{Ord}}\ell(g(\bm{x}_i^\mathrm{Ord}),y_i^\mathrm{Ord})$ be the empirical risk estimator calculated using only the dataset $\mathcal{D}_\mathrm{Ord}$ with true class labels. By setting $\lambda\widehat{R}_\mathrm{SCD}(g)+(1-\lambda)\widehat{R}_\mathrm{Ord}(g)$ where $\lambda\in[0,1]$, an unbiased risk estimator can be computed from both ordinary labeled data and SconfConfDiff data.

\subsection{Limitations}
This study focuses on binary classification problems and is not readily extendable to multi-class classification. To handle multi-class tasks, both similarity-confidence and confidence-difference must be redefined to accommodate multiclass settings. Furthermore, while this work utilizes two weak supervision: similarity-confidence and confidence-difference, it is important future work to explore a unified learning framework that can incorporate a broader variety of weak labels. Our current formulation assumes that each data pair is annotated with both similarity-confidence and confidence-difference. However, the proposed method does not simply extend to cases where either label is missing. Developing a robust learning framework for such incomplete settings is an important direction for future work. This paper derives two methods, SconfConfDiff-Convex Classification and SconfConfDiff classification, but at this stage, there is no clear criterion to determine which method is more appropriate depending on the situation. For practical implementation in the future, it is necessary to clarify the conditions under which each method should be applied and the differences in their effectiveness, and to establish guidelines for selecting the appropriate method.

\section{Details of Experimental Settings}\label{appendix_exp_setting}
In this section, the details of experimental datasets and hyperparameters are provided. 
The following experimental settings are based on those used in ~\cite{confdiff}.

\subsection{Generation of Similarity-Confidence and Confidence-Difference}

Although similarity-confidence and confidence-difference are originally assigned by annotators, in this paper we generate them synthetically to simplify the experiments.
First, we train a probabilistic classifier based on logistic regression using the ground-truth labeled data and the same model architecture.
Next, we randomly sample pairs of unlabeled data points and input them into the probabilistic classifier to obtain class posterior probabilities.
Then, following the definition in Section~\ref{sconf setting} and Section~\ref{confdiff setting}, we generate similarity-confidence and confidence-difference for each sampled pair and construct SconfConfDiff dataset.
Details of the procedure for generating SconfConfDiff dataset used in the experiment are provided in Algorithm~\ref{alg:generate sconfconfdiff}.

\begin{algorithm*}[tb]
\caption{Generation of Synthetic SconfConfDiff Data Pairs for Experiments}
\label{alg:generate sconfconfdiff}
\begin{algorithmic}
\STATE {\bfseries Input:} train dataset $\mathcal{D}_{\text{train}}$, size $2n$, annotation model $f_{\text{prob}}$
\STATE Train $f_{\text{prob}}: \mathcal{X} \to [0, 1]$ on $\mathcal{D}_{\text{train}}$ (using ground-truth labels) to estimate $p(y=+1 \mid \bm{x})$.
\STATE $\mathcal{D}_{\text{SCD}} \leftarrow \emptyset$
\FOR{$i=1$ to $n$}
\STATE Sample $\bm{x}_i, \bm{x}_i' \in \mathcal{D}_{\text{train}}$.
\STATE Calculate the positive confidence for the first instance:
$r_i \leftarrow f_{\text{prob}}(\bm{x}_i) = p(y_i=+1 \mid \bm{x}_i)$
\STATE Calculate the positive confidence for the second instance:
$r_i' \leftarrow f_{\text{prob}}(\bm{x}_i') = p(y_i'=+1 \mid \bm{x}_i')$
\STATE Compute the similarity-confidence $s_i$:
$s_i \leftarrow r_ir_i' +(1 - r_i)(1 - r_i')$
\STATE Compute the confidence-difference $c_i$:
$c_i \leftarrow r_i' - r_i$
\STATE $\mathcal{D}_{\text{SCD}} \leftarrow \mathcal{D}_{\text{SCD}} \cup \{((\bm{x}_i, \bm{x}_i'), s_i, c_i)\}$
\ENDFOR
\STATE {\bfseries Output:}
SconfConfDiff dataset $\mathcal{D}_{\text{SCD}}$
\end{algorithmic}
\end{algorithm*}

\begin{table}
\caption{Characteristics of experimental datasets.}
\label{tab:details of dataset}
\centering
\begin{tabular}{cccccc}
\toprule
\textbf{Dataset} & \textbf{\# Train} & \textbf{\# Test} & \textbf{\# Features} & \textbf{\# Class Labels} & \textbf{Model} \\
\midrule
\textbf{MNIST} & 60,000 & 10,000 & 784 & 10 & MLP \\
\textbf{Kuzushiji} & 60,000 & 10,000 & 784 & 10 & MLP \\
\textbf{Fashion} & 60,000 & 10,000 & 784 & 10 & MLP \\
\textbf{CIFAR-10} & 50,000 & 10,000 & 3,072 & 10 & ResNet-34 \\
\midrule
\textbf{Optdigits} & 4,495 & 1,125 & 64 & 10 & MLP \\
\textbf{Pendigits} & 8,793 & 2,199 & 16 & 10 & MLP \\
\textbf{Letter} & 16,000 & 4,000 & 16 & 10 & MLP \\
\textbf{PMU-UD} & 4,144 & 1,036 & 784 & 26 & MLP \\
\bottomrule
\end{tabular}
\end{table}

\subsection{Details of Experimental Datasets}
Table~\ref{tab:details of dataset} summarizes detailed statistics along with the corresponding model architectures. The basic information regarding the datasets, their sources, and data splits is as follows.

\begin{itemize}[left=1pt]
\item MNIST~\cite{mnist}:  
a grayscale handwritten digit classification dataset, and the label space is $\{0,1,2,3,4,5,6,7,8,9\}$.
To formulate a binary classification task, $\{0,2,4,6,8\}$ are treated as positive class, and $\{1,3,5,7,9\}$ are treated as negative class. 15,000 data pairs are sampled for use in our experiments. The dataset is available at \href{http://yann.lecun.com/exdb/mnist/}{http://yann.lecun.com/exdb/mnist/}.

\item Kuzushiji-MNIST~\cite{kmnist}: a grayscale handwritten Japanese character classification dataset, and the label space is \{`o', `ki', `su', `tsu', `na', `ha', `ma', `ya', `re', `wo'\}.
To formulate a binary classification task, we treated the characters \{`o', `su', `na', `ma', `re'\} as the positive class and \{`ki', `tsu', `ha', `ya', `wo'\} as the negative class. 15,000 data pairs are sampled  for use in our experiments. The dataset is available at \href{https://github.com/rois-codh/kmnist}{https://github.com/rois-codh/kmnist}.

\item Fashion-MNIST~\cite{fashion}: a grayscale image dataset for fashion item classification, and the label space is \{`T-shirt', `trouser', `pullover', `dress', `sandal', `coat', `shirt', `sneaker', `bag', `ankle boot'\}.
To formulate a binary classification task, we treated \{`T-shirt', `pullover`, `coat', `shirt', `bag'\} as the positive class and \{`trouser', `dress', `sandal', `sneaker', `ankle boot'\} as the negative class. 15,000 data pairs are sampled  for use in our experiments. The dataset is available at \href{https://github.com/zalandoresearch/fashion-mnist}{https://github.com/zalandoresearch/fashion-mnist}.

\item CIFAR-10~\cite{cifar10}: a color image dataset for object classification, and the label space is \{`airplane', `bird', `automobile', `cat', `deer', `dog', `frog', `horse', `ship', `truck'\}. 
To formulate a binary classification task, we treated \{`bird', `deer', `dog', `frog', `cat', `horse'\} as the positive class and \{`airplane', `automobile', `ship', `truck'\} as the negative class. 10,000 data pairs are sampled  for use in our experiments. The dataset is available at \href{https://www.cs.toronto.edu/~kriz/cifar.html}{https://www.cs.toronto.edu/~kriz/cifar.html}.

\item Optdigits~\cite{optdigits}, Pendigits~\cite{pendigits}, and PMU-UD~\cite{pmu-ud} are UCI datasets for handwritten digit recognition. The train-test splits are summarized in Table~\ref{tab:details of dataset}. 
For PMU-UD, which consists of image data, we resized each image to 28$\times$28 grayscale before training. The label space is $\{0,1,2,3,4,5,6,7,8,9\}$. To formulate a binary classification task, $\{0,2,4,6,8\}$ are treated as positive class, and $\{1,3,5,7,9\}$ are treated as negative class. We sampled 1,200, 2,500, and 1,000 data pairs from the training datasets of Optdigits, Pendigits, and PMU-UD, respectively. The datasets are available at 
\begin{itemize}
\item \href{https://archive.ics.uci.edu/dataset/80/optical+recognition+of+handwritten+digits}{https://archive.ics.uci.edu/dataset/80/optical+recognition+of+handwritten+digits} \item \href{https://archive.ics.uci.edu/dataset/81/pen+based+recognition+of+handwritten+digits}{https://archive.ics.uci.edu/dataset/81/pen+based+recognition+of+handwritten+digits}
\item \href{https://archive.ics.uci.edu/dataset/469/pmu+ud}{https://archive.ics.uci.edu/dataset/469/pmu+ud}
\end{itemize}

\item Letter~\cite{letter} is a UCI dataset for letter recognition, and the label space consists of 26 capital letters in the English alphabet. 
To formulate a binary classification task, we assigned the first 13 characters to the positive class and the remaining 13 characters to the negative class. 4,000 data pairs are sampled for use in our experiments. The dataset is available at \href{https://archive.ics.uci.edu/dataset/59/letter+recognition}{https://archive.ics.uci.edu/dataset/59/letter+recognition}.

\end{itemize}

\subsection{Details of Hyperparameters}

\begin{table}
\caption{Details of hyperparameters.}
\label{tab:details of hyperparameters}
\centering
\begin{tabular}{ccccc}
\toprule
\textbf{Dataset} & \textbf{\# Epoch} & \textbf{\# Learning Rate} & \textbf{\# Weight Decay} & \textbf{\# Batch Size}\\
\midrule
\textbf{MNIST} & 200 & 1e-3 & 1e-5 & 256 \\
\textbf{Kuzushiji} & 200 & 1e-3 & 1e-5 & 256 \\
\textbf{Fashion} & 200 & 1e-3 & 1e-5 & 256 \\
\textbf{CIFAR-10} & 200 & 5e-4 & 1e-5 & 128 \\
\midrule
\textbf{Optdigits} & 200 & 1e-3 & 1e-5 & 256 \\
\textbf{Pendigits} & 200 & 1e-3 & 1e-5 & 256 \\
\textbf{Letter} & 200 & 1e-3 & 1e-5 & 256 \\
\textbf{PMU-UD} & 200 & 1e-3 & 1e-5 & 256 \\
\bottomrule
\end{tabular}
\end{table}

Table~\ref{tab:details of hyperparameters} shows the details of hyperparameters.
The test accuracy is computed as the average over the last 10 epochs. 
The epoch size of the training probabilistic classifier to generate confidence scores is 10.
All methods were implemented in PyTorch~\cite{PyTorch} and optimized using the Adam optimizer~\cite{Adam}.
To ensure fair comparisons, we set the same hyperparameter values for all the compared approaches.

\subsection{Computational Resources}\label{server spec}

The experiments were conducted on a machine with the following specifications:
\begin{itemize}[left=1pt]
\item \textbf{CPU}: Intel(R) Xeon(R) Gold 6312U (24 cores, 48 threads, 2.40 GHz)
\item \textbf{RAM}: 1.0 TB
\item \textbf{GPU}: NVIDIA RTX A6000 (48 GB VRAM)
\item \textbf{Storage}: 1 TB SSD
\item \textbf{OS}: Rocky Linux 9.2
\item \textbf{Software}: Python 3.12.2, numpy 1.26.4, torch 2.3.0, torchvision 0.18.0, CUDA 11.8
\end{itemize}

\section{Details of Experimental Results}\label{appendix_exp}
In this section, we describe the details of the experimental results.

\begin{table}
\caption{
Overall classification accuracy on the benchmark test set with $\pi_+=0.2$ averaged over five random seeds, with mean and standard deviation (mean$\pm$std). The highest score among the compared methods, excluding supervised learning, is shown in bold.
}
\label{tab:benchmark_acc_0.2_all}
\begin{center}
\begin{small}
\begin{tabular}{lcccc}
\toprule
Method & MNIST & Kuzushiji & Fashion & CIFAR-10 \\
\midrule
SconfConfDiff-Unbiased & 0.894 $\pm$ 0.043 & 0.706 $\pm$ 0.016 & 0.851 $\pm$ 0.085 & 0.865 $\pm$ 0.006 \\
SconfConfDiff-ReLU & 0.975 $\pm$ 0.002 & 0.890 $\pm$ 0.010 & 0.965 $\pm$ 0.002 & 0.871 $\pm$ 0.012 \\
SconfConfDiff-ABS & \textbf{0.978 $\pm$ 0.002} & \textbf{0.906 $\pm$ 0.002} & \textbf{0.969 $\pm$ 0.003} & \textbf{0.884 $\pm$ 0.006} \\
\cmidrule(lr){1-5}
ConfDiff-Unbiased & 0.805 $\pm$ 0.054 & 0.655 $\pm$ 0.035 & 0.859 $\pm$ 0.061 & 0.782 $\pm$ 0.024 \\
Convex($0.2$)-Unbiased & 0.837 $\pm$ 0.094 & 0.687 $\pm$ 0.064 & 0.850 $\pm$ 0.062 & 0.807 $\pm$ 0.009 \\
Convex($0.5$)-Unbiased & 0.933 $\pm$ 0.007 & 0.833 $\pm$ 0.026 & 0.886 $\pm$ 0.021 & 0.830 $\pm$ 0.025 \\
Convex($0.8$)-Unbiased & 0.858 $\pm$ 0.079 & 0.732 $\pm$ 0.074 & 0.846 $\pm$ 0.056 & 0.806 $\pm$ 0.021 \\
Sconf-Unbiased & 0.830 $\pm$ 0.050 & 0.706 $\pm$ 0.044 & 0.862 $\pm$ 0.078 & 0.781 $\pm$ 0.014 \\
\cmidrule(lr){1-5}
ConfDiff-ReLU & 0.967 $\pm$ 0.005 & 0.861 $\pm$ 0.014 & 0.962 $\pm$ 0.003 & 0.854 $\pm$ 0.011 \\
Convex($\gamma=0.2$)-ReLU & 0.569 $\pm$ 0.189 & 0.473 $\pm$ 0.105 & 0.800 $\pm$ 0.122 & 0.813 $\pm$ 0.019 \\
Convex($\gamma=0.5$)-ReLU & 0.360 $\pm$ 0.121 & 0.319 $\pm$ 0.099 & 0.843 $\pm$ 0.097 & 0.676 $\pm$ 0.075 \\
Convex($\gamma=0.8$)-ReLU & 0.309 $\pm$ 0.096 & 0.400 $\pm$ 0.192 & 0.644 $\pm$ 0.294 & 0.787 $\pm$ 0.021 \\
Sconf-ReLU & 0.959 $\pm$ 0.006 & 0.845 $\pm$ 0.019 & 0.946 $\pm$ 0.008 & 0.838 $\pm$ 0.010 \\
\cmidrule(lr){1-5}
ConfDiff-ABS & 0.975 $\pm$ 0.001 & 0.898 $\pm$ 0.006 & 0.967 $\pm$ 0.001 & 0.868 $\pm$ 0.007 \\
Convex($\gamma=0.2$)-ABS & 0.975 $\pm$ 0.003 & 0.904 $\pm$ 0.007 & 0.967 $\pm$ 0.002 & 0.870 $\pm$ 0.005 \\
Convex($\gamma=0.5$)-ABS & 0.976 $\pm$ 0.004 & 0.903 $\pm$ 0.006 & 0.965 $\pm$ 0.002 & 0.869 $\pm$ 0.006 \\
Convex($\gamma=0.8$)-ABS & 0.972 $\pm$ 0.003 & 0.901 $\pm$ 0.005 & 0.962 $\pm$ 0.004 & 0.870 $\pm$ 0.008 \\
Sconf-ABS & 0.960 $\pm$ 0.004 & 0.881 $\pm$ 0.006 & 0.955 $\pm$ 0.006 & 0.818 $\pm$ 0.022 \\
\cmidrule(lr){1-5}
Supervised & 0.990 $\pm$ 0.000 & 0.936 $\pm$ 0.003 & 0.979 $\pm$ 0.002 & 0.895 $\pm$ 0.003 \\
\bottomrule
\end{tabular}
\end{small}
\end{center}
\end{table}

\begin{table}
\caption{
Overall classification accuracy on the UCI test set with $\pi_+=0.2$ averaged over five random seeds, with mean and standard deviation (mean$\pm$std). The highest score among the compared methods, excluding supervised learning, is shown in bold.
}
\label{tab:uci_acc_0.2_all}
\begin{center}
\begin{small}
\begin{tabular}{lcccc}
\toprule
Method & Optdigits & Pendigits & Letter & PMU-UD \\
\midrule
SconfConfDiff-Unbiased & 0.917 $\pm$ 0.067 & 0.954 $\pm$ 0.041 & 0.746 $\pm$ 0.044 & 0.841 $\pm$ 0.058 \\
SconfConfDiff-ReLU & 0.955 $\pm$ 0.023 & 0.987 $\pm$ 0.003 & 0.934 $\pm$ 0.007 & 0.966 $\pm$ 0.007 \\
SconfConfDiff-ABS & \textbf{0.969 $\pm$ 0.005} & \textbf{0.991 $\pm$ 0.002} & \textbf{0.951 $\pm$ 0.004} & 0.975 $\pm$ 0.005 \\
\cmidrule(lr){1-5}
ConfDiff-Unbiased & 0.883 $\pm$ 0.053 & 0.940 $\pm$ 0.050 & 0.765 $\pm$ 0.051 & 0.822 $\pm$ 0.057 \\
Convex($0.2$)-Unbiased & 0.949 $\pm$ 0.015 & 0.963 $\pm$ 0.028 & 0.736 $\pm$ 0.099 & 0.867 $\pm$ 0.062 \\
Convex($0.5$)-Unbiased & 0.952 $\pm$ 0.015 & 0.971 $\pm$ 0.008 & 0.736 $\pm$ 0.100 & 0.918 $\pm$ 0.047 \\
Convex($0.8$)-Unbiased & 0.884 $\pm$ 0.076 & 0.938 $\pm$ 0.066 & 0.781 $\pm$ 0.027 & 0.857 $\pm$ 0.070 \\
Sconf-Unbiased & 0.874 $\pm$ 0.056 & 0.931 $\pm$ 0.052 & 0.754 $\pm$ 0.027 & 0.853 $\pm$ 0.052 \\
\cmidrule(lr){1-5}
ConfDiff-ReLU & 0.958 $\pm$ 0.011 & 0.982 $\pm$ 0.004 & 0.930 $\pm$ 0.005 & 0.977 $\pm$ 0.010 \\
Convex($\gamma=0.2$)-ReLU & 0.723 $\pm$ 0.069 & 0.982 $\pm$ 0.004 & 0.890 $\pm$ 0.015 & 0.776 $\pm$ 0.188 \\
Convex($\gamma=0.5$)-ReLU & 0.605 $\pm$ 0.356 & 0.959 $\pm$ 0.041 & 0.853 $\pm$ 0.019 & 0.242 $\pm$ 0.080 \\
Convex($\gamma=0.8$)-ReLU & 0.688 $\pm$ 0.321 & 0.945 $\pm$ 0.053 & 0.838 $\pm$ 0.047 & 0.475 $\pm$ 0.306 \\
Sconf-ReLU & 0.931 $\pm$ 0.016 & 0.968 $\pm$ 0.011 & 0.897 $\pm$ 0.007 & 0.900 $\pm$ 0.053 \\
\cmidrule(lr){1-5}
ConfDiff-ABS & 0.964 $\pm$ 0.010 & 0.988 $\pm$ 0.005 & 0.945 $\pm$ 0.004 & 0.978 $\pm$ 0.008 \\
Convex($\gamma=0.2$)-ABS & 0.969 $\pm$ 0.009 & 0.987 $\pm$ 0.004 & 0.947 $\pm$ 0.005 & \textbf{0.979 $\pm$ 0.006} \\
Convex($\gamma=0.5$)-ABS & 0.969 $\pm$ 0.005 & 0.987 $\pm$ 0.004 & 0.946 $\pm$ 0.004 & 0.979 $\pm$ 0.008 \\
Convex($\gamma=0.8$)-ABS & 0.966 $\pm$ 0.006 & 0.986 $\pm$ 0.004 & 0.941 $\pm$ 0.004 & 0.976 $\pm$ 0.004 \\
Sconf-ABS & 0.944 $\pm$ 0.014 & 0.978 $\pm$ 0.006 & 0.915 $\pm$ 0.006 & 0.954 $\pm$ 0.010 \\
\cmidrule(lr){1-5}
Supervised & 0.989 $\pm$ 0.003 & 0.997 $\pm$ 0.001 & 0.978 $\pm$ 0.004 & 0.994 $\pm$ 0.003 \\
\bottomrule
\end{tabular}
\end{small}
\end{center}
\end{table}

\begin{table}
\caption{
Overall classification accuracy on the benchmark test set with $\pi_+=0.5$ averaged over five random seeds, with mean and standard deviation (mean$\pm$std). The highest score among the compared methods, excluding supervised learning, is shown in bold.
}
\label{tab:benchmark_acc_0.5_all}
\begin{center}
\begin{small}
\begin{tabular}{lcccc}
\toprule
Method & MNIST & Kuzushiji & Fashion & CIFAR-10 \\
\midrule
SconfConfDiff-Unbiased & \textbf{0.977 $\pm$ 0.004} & \textbf{0.901 $\pm$ 0.005} & 0.963 $\pm$ 0.002 & 0.837 $\pm$ 0.004 \\
SconfConfDiff-ReLU & \textbf{0.977 $\pm$ 0.004} & \textbf{0.901 $\pm$ 0.005} & 0.963 $\pm$ 0.002 & 0.837 $\pm$ 0.004 \\
SconfConfDiff-ABS & \textbf{0.977 $\pm$ 0.004} & \textbf{0.901 $\pm$ 0.005} & 0.963 $\pm$ 0.002 & 0.837 $\pm$ 0.004 \\
\cmidrule(lr){1-5}
ConfDiff-Unbiased & 0.931 $\pm$ 0.030 & 0.795 $\pm$ 0.041 & 0.867 $\pm$ 0.090 & 0.738 $\pm$ 0.019 \\
\cmidrule(lr){1-5}
ConfDiff-ReLU & 0.946 $\pm$ 0.008 & 0.805 $\pm$ 0.012 & 0.960 $\pm$ 0.002 & 0.837 $\pm$ 0.003 \\
\cmidrule(lr){1-5}
ConfDiff-ABS & 0.965 $\pm$ 0.001 & 0.866 $\pm$ 0.004 & \textbf{0.968 $\pm$ 0.001} & \textbf{0.842 $\pm$ 0.002} \\
\cmidrule(lr){1-5}
Supervised & 0.987 $\pm$ 0.001 & 0.928 $\pm$ 0.002 & 0.976 $\pm$ 0.001 & 0.876 $\pm$ 0.003 \\
\bottomrule
\end{tabular}
\end{small}
\end{center}
\end{table}

\begin{table}
\caption{
Overall classification accuracy on the benchmark test set with $\pi_+=0.8$ averaged over five random seeds, with mean and standard deviation (mean$\pm$std). The highest score among the compared methods, excluding supervised learning, is shown in bold.
}
\label{tab:benchmark_acc_0.8_all}
\begin{center}
\begin{small}
\begin{tabular}{lcccc}
\toprule
Method & MNIST & Kuzushiji & Fashion & CIFAR-10 \\
\midrule
SconfConfDiff-Unbiased & 0.801 $\pm$ 0.064 & 0.796 $\pm$ 0.047 & 0.854 $\pm$ 0.071 & 0.872 $\pm$ 0.008 \\
SconfConfDiff-ReLU & 0.975 $\pm$ 0.003 & 0.887 $\pm$ 0.012 & 0.969 $\pm$ 0.003 & 0.850 $\pm$ 0.009 \\
SconfConfDiff-ABS & \textbf{0.985 $\pm$ 0.002} & \textbf{0.920 $\pm$ 0.004} & \textbf{0.976 $\pm$ 0.001} & 0.858 $\pm$ 0.061 \\
\cmidrule(lr){1-5}
ConfDiff-Unbiased & 0.778 $\pm$ 0.019 & 0.719 $\pm$ 0.083 & 0.867 $\pm$ 0.031 & 0.793 $\pm$ 0.012 \\
Convex($0.2$)-Unbiased & 0.834 $\pm$ 0.080 & 0.849 $\pm$ 0.036 & 0.832 $\pm$ 0.031 & 0.813 $\pm$ 0.014 \\
Convex($0.5$)-Unbiased & 0.942 $\pm$ 0.024 & 0.829 $\pm$ 0.038 & 0.879 $\pm$ 0.064 & 0.848 $\pm$ 0.031 \\
Convex($0.8$)-Unbiased & 0.826 $\pm$ 0.086 & 0.774 $\pm$ 0.063 & 0.874 $\pm$ 0.055 & 0.815 $\pm$ 0.016 \\
Sconf-Unbiased & 0.804 $\pm$ 0.051 & 0.766 $\pm$ 0.038 & 0.867 $\pm$ 0.101 & 0.796 $\pm$ 0.017 \\
\cmidrule(lr){1-5}
ConfDiff-ReLU & 0.972 $\pm$ 0.003 & 0.879 $\pm$ 0.006 & 0.970 $\pm$ 0.001 & 0.854 $\pm$ 0.003 \\
Convex($\gamma=0.2$)-ReLU & 0.739 $\pm$ 0.257 & 0.455 $\pm$ 0.132 & 0.916 $\pm$ 0.028 & 0.795 $\pm$ 0.032 \\
Convex($\gamma=0.5$)-ReLU & 0.397 $\pm$ 0.147 & 0.266 $\pm$ 0.064 & 0.802 $\pm$ 0.135 & 0.697 $\pm$ 0.030 \\
Convex($\gamma=0.8$)-ReLU & 0.558 $\pm$ 0.277 & 0.299 $\pm$ 0.051 & 0.857 $\pm$ 0.111 & 0.795 $\pm$ 0.028 \\
Sconf-ReLU & 0.957 $\pm$ 0.019 & 0.846 $\pm$ 0.014 & 0.950 $\pm$ 0.010 & 0.836 $\pm$ 0.009 \\
\cmidrule(lr){1-5}
ConfDiff-ABS & 0.983 $\pm$ 0.001 & 0.916 $\pm$ 0.005 & 0.973 $\pm$ 0.001 & 0.860 $\pm$ 0.007 \\
Convex($\gamma=0.2$)-ABS & 0.983 $\pm$ 0.001 & 0.918 $\pm$ 0.004 & 0.973 $\pm$ 0.001 & \textbf{0.880 $\pm$ 0.005} \\
Convex($\gamma=0.5$)-ABS & 0.983 $\pm$ 0.001 & 0.917 $\pm$ 0.005 & 0.971 $\pm$ 0.001 & 0.878 $\pm$ 0.008 \\
Convex($\gamma=0.8$)-ABS & 0.982 $\pm$ 0.003 & 0.914 $\pm$ 0.005 & 0.971 $\pm$ 0.002 & 0.874 $\pm$ 0.007 \\
Sconf-ABS & 0.973 $\pm$ 0.001 & 0.889 $\pm$ 0.004 & 0.966 $\pm$ 0.003 & 0.839 $\pm$ 0.018 \\
\cmidrule(lr){1-5}
Supervised & 0.992 $\pm$ 0.000 & 0.942 $\pm$ 0.003 & 0.980 $\pm$ 0.001 & 0.898 $\pm$ 0.003 \\
\bottomrule
\end{tabular}
\end{small}
\end{center}
\end{table}

\begin{table}
\caption{
Overall classification accuracy on the UCIk test set with $\pi_+=0.5$ averaged over five random seeds, with mean and standard deviation (mean$\pm$std). The highest score among the compared methods, excluding supervised learning, is shown in bold.
}
\label{tab:uci_acc_0.5_all}
\begin{center}
\begin{small}
\begin{tabular}{lcccc}
\toprule
Method & Optdigits & Pendigits & Letter & PMU-UD \\
\midrule
SconfConfDiff-Unbiased & \textbf{0.971 $\pm$ 0.008} & \textbf{0.992 $\pm$ 0.002} & \textbf{0.935 $\pm$ 0.006} & \textbf{0.986 $\pm$ 0.005} \\
SconfConfDiff-ReLU & \textbf{0.971 $\pm$ 0.008} & \textbf{0.992 $\pm$ 0.002} & \textbf{0.935 $\pm$ 0.006} & \textbf{0.986 $\pm$ 0.005} \\
SconfConfDiff-ABS & \textbf{0.971 $\pm$ 0.008} & \textbf{0.992 $\pm$ 0.002} & \textbf{0.935 $\pm$ 0.006} & \textbf{0.986 $\pm$ 0.005} \\
\cmidrule(lr){1-5}
ConfDiff-Unbiased & 0.924 $\pm$ 0.014 & 0.945 $\pm$ 0.024 & 0.701 $\pm$ 0.079 & 0.913 $\pm$ 0.060 \\
\cmidrule(lr){1-5}
ConfDiff-ReLU & 0.922 $\pm$ 0.007 & 0.977 $\pm$ 0.004 & 0.899 $\pm$ 0.008 & 0.952 $\pm$ 0.014 \\
\cmidrule(lr){1-5}
ConfDiff-ABS & 0.964 $\pm$ 0.007 & 0.989 $\pm$ 0.001 & 0.922 $\pm$ 0.007 & 0.982 $\pm$ 0.005 \\
\cmidrule(lr){1-5}
Supervised & 0.991 $\pm$ 0.003 & 0.996 $\pm$ 0.002 & 0.976 $\pm$ 0.002 & 0.993 $\pm$ 0.002 \\
\bottomrule
\end{tabular}
\end{small}
\end{center}
\end{table}

\begin{table}
\caption{
Overall classification accuracy on the UCI test set with $\pi_+=0.8$ averaged over five random seeds, with mean and standard deviation (mean$\pm$std). The highest score among the compared methods, excluding supervised learning, is shown in bold.
}
\label{tab:uci_acc_0.8_all}
\begin{center}
\begin{small}
\begin{tabular}{lcccc}
\toprule
Method & Optdigits & Pendigits & Letter & PMU-UD \\
\midrule
SconfConfDiff-Unbiased & 0.894 $\pm$ 0.059 & 0.927 $\pm$ 0.044 & 0.789 $\pm$ 0.027 & 0.818 $\pm$ 0.050 \\
SconfConfDiff-ReLU & 0.957 $\pm$ 0.012 & 0.985 $\pm$ 0.004 & 0.935 $\pm$ 0.004 & 0.951 $\pm$ 0.025 \\
SconfConfDiff-ABS & 0.963 $\pm$ 0.008 & \textbf{0.987 $\pm$ 0.005} & \textbf{0.948 $\pm$ 0.004} & 0.975 $\pm$ 0.005 \\
\cmidrule(lr){1-5}
ConfDiff-Unbiased & 0.867 $\pm$ 0.050 & 0.881 $\pm$ 0.097 & 0.726 $\pm$ 0.044 & 0.785 $\pm$ 0.050 \\
Convex($\gamma=0.2$)-Unbiased & 0.912 $\pm$ 0.058 & 0.958 $\pm$ 0.029 & 0.743 $\pm$ 0.037 & 0.824 $\pm$ 0.077 \\
Convex($\gamma=0.5$)-Unbiased & 0.946 $\pm$ 0.010 & 0.966 $\pm$ 0.017 & 0.765 $\pm$ 0.029 & 0.857 $\pm$ 0.078 \\
Convex($\gamma=0.8$)-Unbiased & 0.912 $\pm$ 0.022 & 0.913 $\pm$ 0.043 & 0.813 $\pm$ 0.017 & 0.823 $\pm$ 0.025 \\
Sconf-Unbiased & 0.876 $\pm$ 0.047 & 0.898 $\pm$ 0.063 & 0.784 $\pm$ 0.035 & 0.827 $\pm$ 0.032 \\
\cmidrule(lr){1-5}
ConfDiff-ReLU & 0.961 $\pm$ 0.011 & 0.980 $\pm$ 0.007 & 0.926 $\pm$ 0.006 & 0.971 $\pm$ 0.007 \\
Convex($\gamma=0.2$)-ReLU & 0.602 $\pm$ 0.177 & 0.981 $\pm$ 0.004 & 0.878 $\pm$ 0.032 & 0.577 $\pm$ 0.319 \\
Convex($\gamma=0.5$)-ReLU & 0.541 $\pm$ 0.352 & 0.976 $\pm$ 0.007 & 0.847 $\pm$ 0.035 & 0.346 $\pm$ 0.249 \\
Convex($\gamma=0.8$)-ReLU & 0.494 $\pm$ 0.333 & 0.972 $\pm$ 0.008 & 0.792 $\pm$ 0.101 & 0.329 $\pm$ 0.206 \\
Sconf-ReLU & 0.902 $\pm$ 0.021 & 0.966 $\pm$ 0.005 & 0.896 $\pm$ 0.011 & 0.896 $\pm$ 0.058 \\
\cmidrule(lr){1-5}
ConfDiff-ABS & 0.963 $\pm$ 0.005 & 0.983 $\pm$ 0.004 & 0.943 $\pm$ 0.004 & 0.972 $\pm$ 0.002 \\
Convex($\gamma=0.2$)-ABS & 0.967 $\pm$ 0.004 & 0.982 $\pm$ 0.005 & 0.945 $\pm$ 0.001 & 0.973 $\pm$ 0.004 \\
Convex($\gamma=0.5$)-ABS & \textbf{0.969 $\pm$ 0.009} & 0.980 $\pm$ 0.005 & 0.942 $\pm$ 0.002 & \textbf{0.976 $\pm$ 0.005} \\
Convex($\gamma=0.8$)-ABS & 0.966 $\pm$ 0.007 & 0.977 $\pm$ 0.007 & 0.937 $\pm$ 0.004 & 0.974 $\pm$ 0.005 \\
Sconf-ABS & 0.948 $\pm$ 0.012 & 0.966 $\pm$ 0.013 & 0.915 $\pm$ 0.004 & 0.940 $\pm$ 0.034 \\
\cmidrule(lr){1-5}
Supervised & 0.990 $\pm$ 0.004 & 0.997 $\pm$ 0.001 & 0.977 $\pm$ 0.004 & 0.995 $\pm$ 0.002 \\
\bottomrule
\end{tabular}
\end{small}
\end{center}
\end{table}

\begin{figure}
    \centering
    \includegraphics[width=0.8\linewidth]{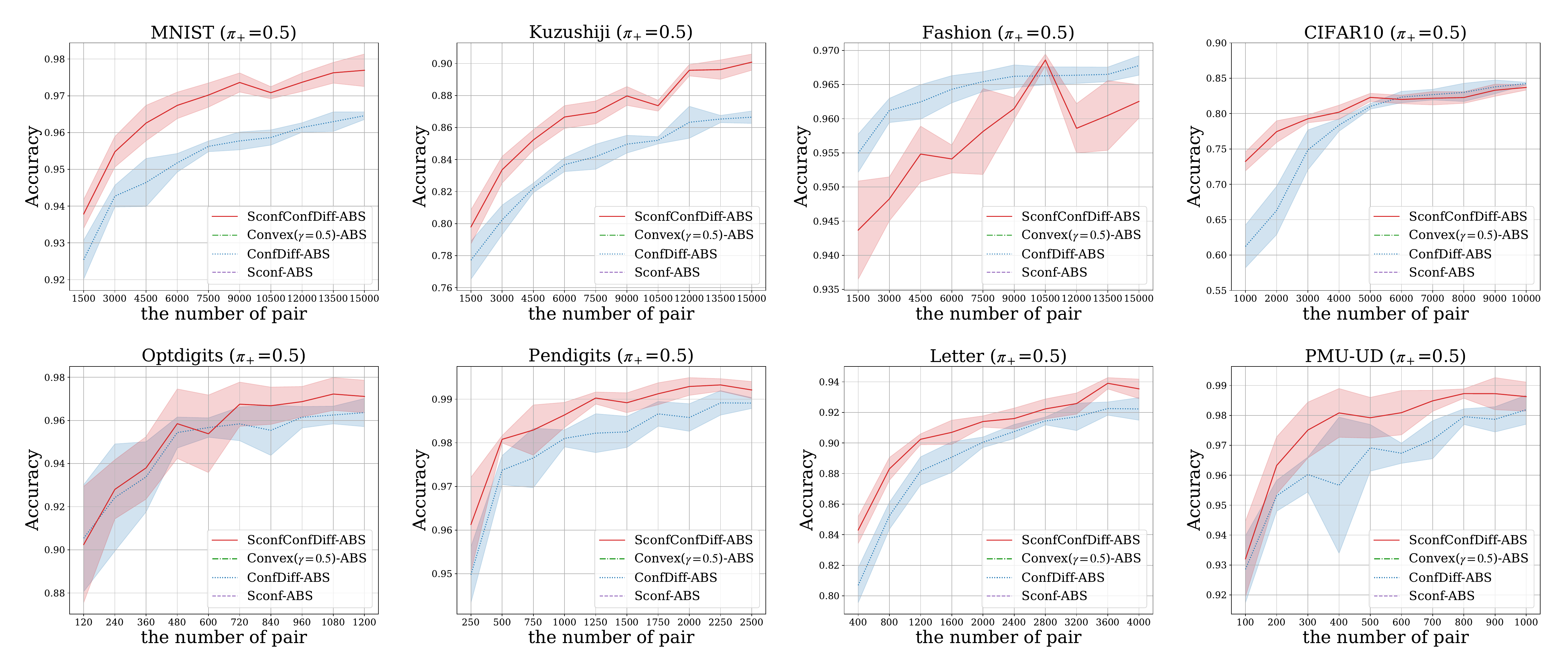}
    \caption{Classification performance of SconfConfDiff-ABS, SconfConfDiff-Convex($\gamma=0.5$)-ABS, Sconf-ABS, and ConfDiff-ABS with varying numbers of training samples ($\pi_+=0.5$).
    }
    \label{fig:num_0.5}
\end{figure}

\begin{figure}
    \centering
    \includegraphics[width=0.8\linewidth]{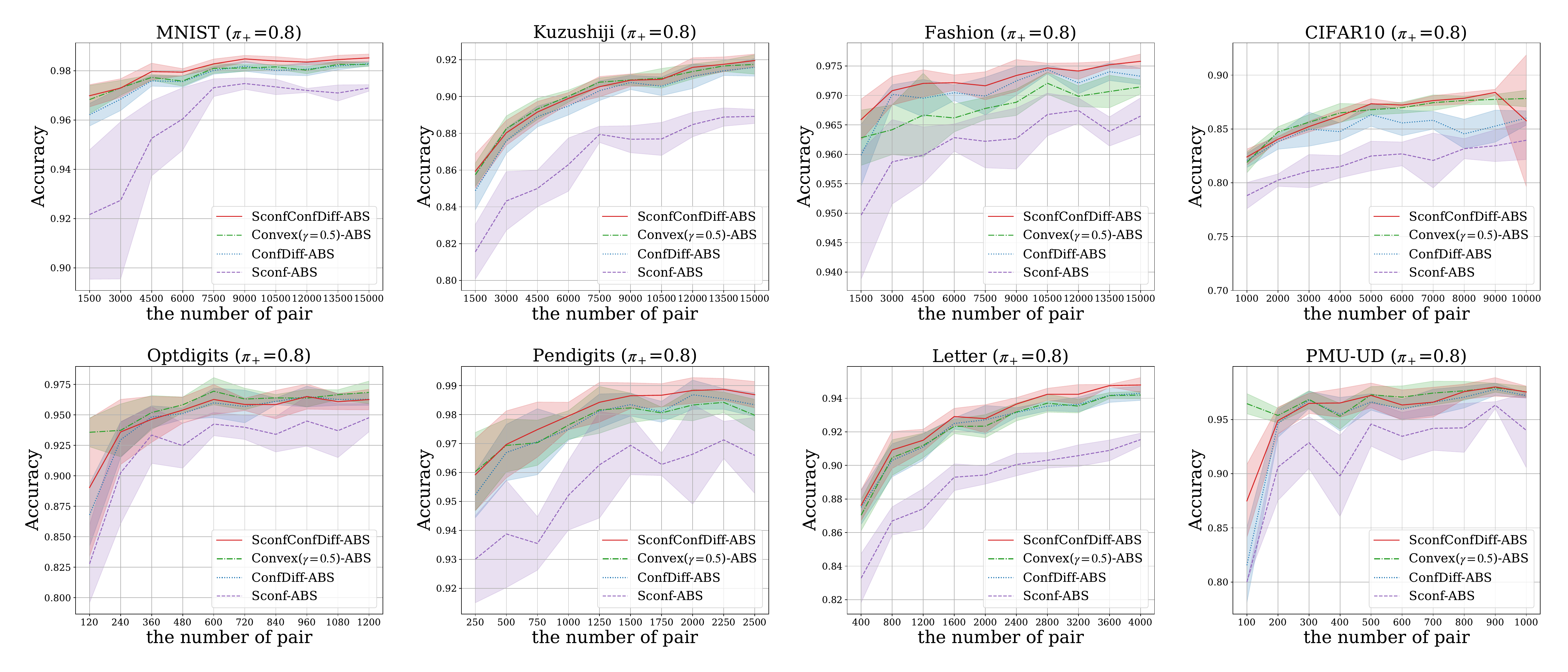}
    \caption{Classification performance of SconfConfDiff-ABS, SconfConfDiff-Convex($\gamma=0.5$)-ABS, Sconf-ABS, and ConfDiff-ABS with varying numbers of training samples ($\pi_+=0.8$).
    }
    \label{fig:num_0.8}
\end{figure}

\begin{figure}
    \centering
    \includegraphics[width=0.8\linewidth]{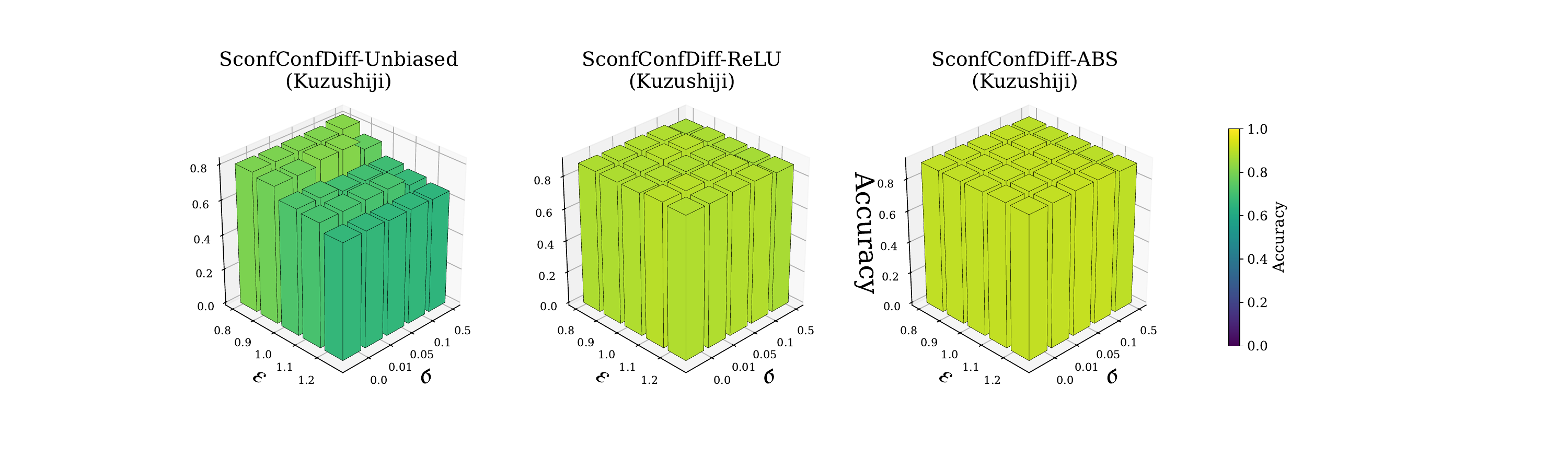}
    \caption{Classification performance on Kuzushiji-MNIST with $\pi_+=0.2$ given an noisy class prior probability, noisy similarity-confidence, and noisy confidence-difference. Both the height and color of the bars represent accuracy.}
    \label{fig:robust_kmnist}
\end{figure}

\begin{figure}
    \centering
    \includegraphics[width=0.8\linewidth]{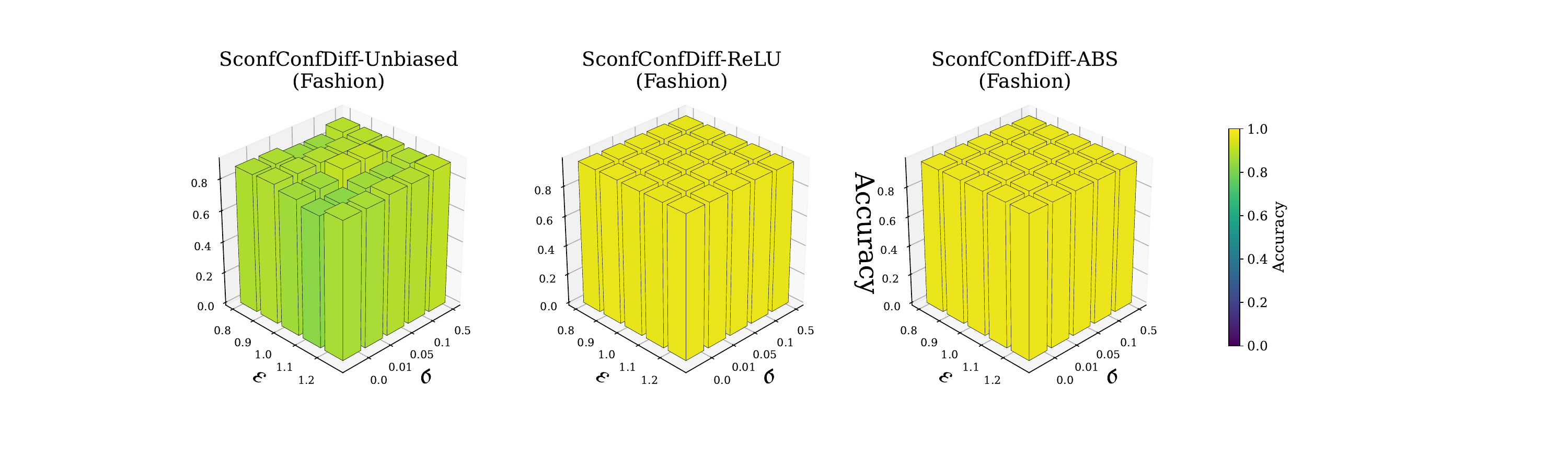}
    \caption{Classification performance on Fashion-MNIST with $\pi_+=0.2$ given an noisy class prior probability, noisy similarity-confidence, and noisy confidence-difference. Both the height and color of the bars represent accuracy.}
    \label{fig:robust_fashion}
\end{figure}

\begin{figure}
    \centering
    \includegraphics[width=0.8\linewidth]{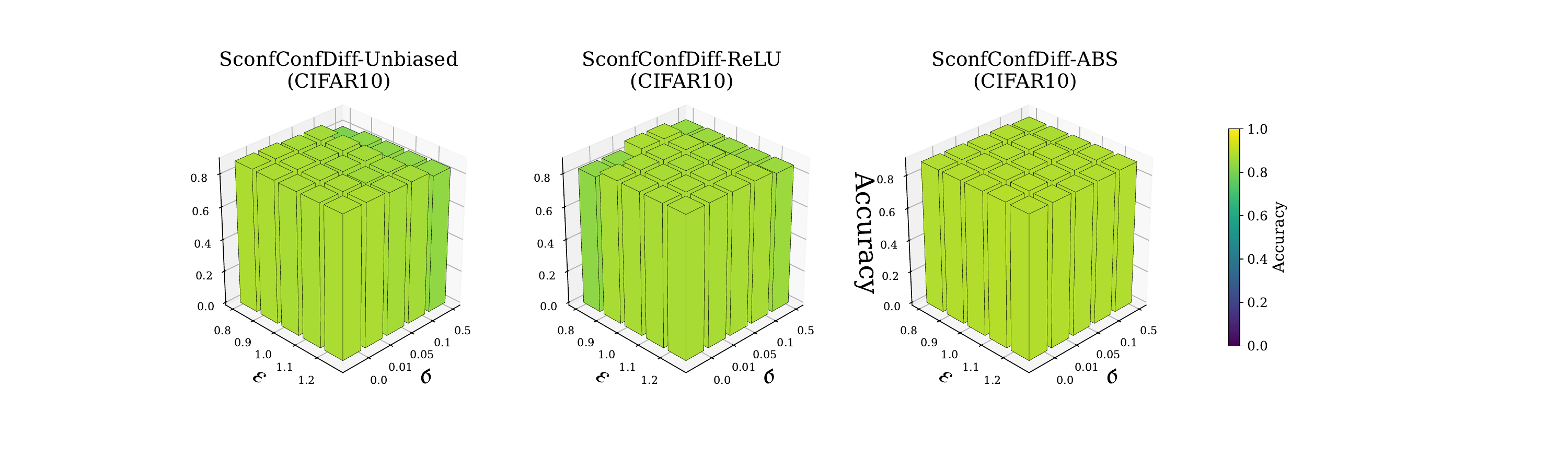}
    \caption{Classification performance on CIFAR-10 with $\pi_+=0.2$ given an noisy class prior probability, noisy similarity-confidence, and noisy confidence-difference. Both the height and color of the bars represent accuracy.}
    \label{fig:robust_CIFAR-10}
\end{figure}

\begin{figure}
    \centering
    \includegraphics[width=0.8\linewidth]{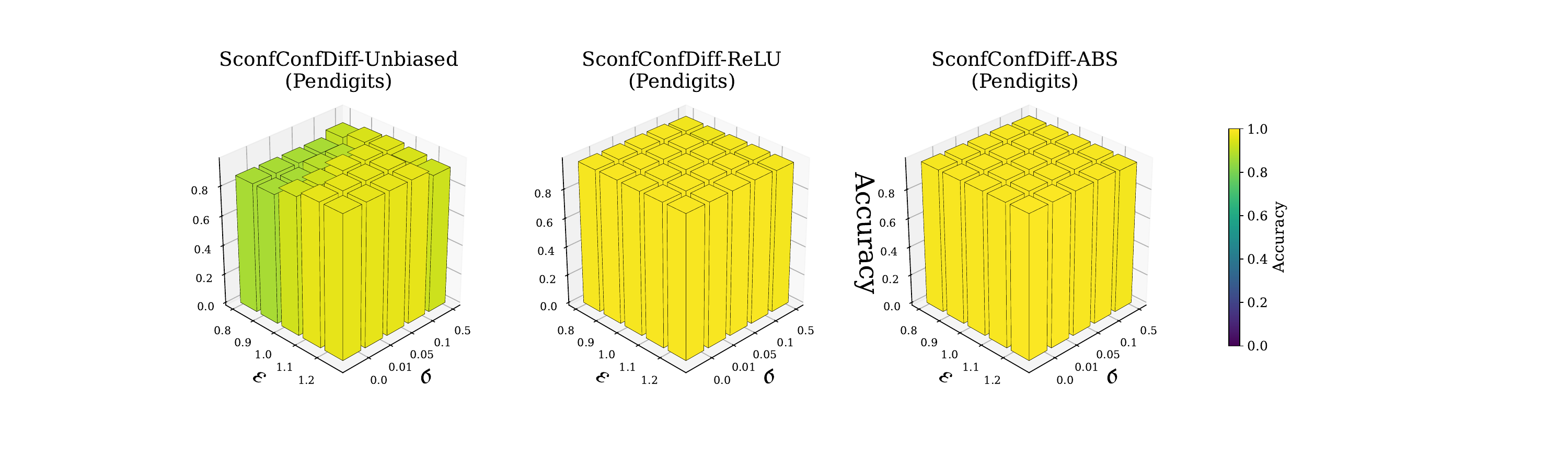}
    \caption{Classification performance on Pendigits with $\pi_+=0.2$ given an noisy class prior probability, noisy similarity-confidence, and noisy confidence-difference. Both the height and color of the bars represent accuracy.}
    \label{fig:robust_pendigits}
\end{figure}

\begin{figure}
    \centering
    \includegraphics[width=0.8\linewidth]{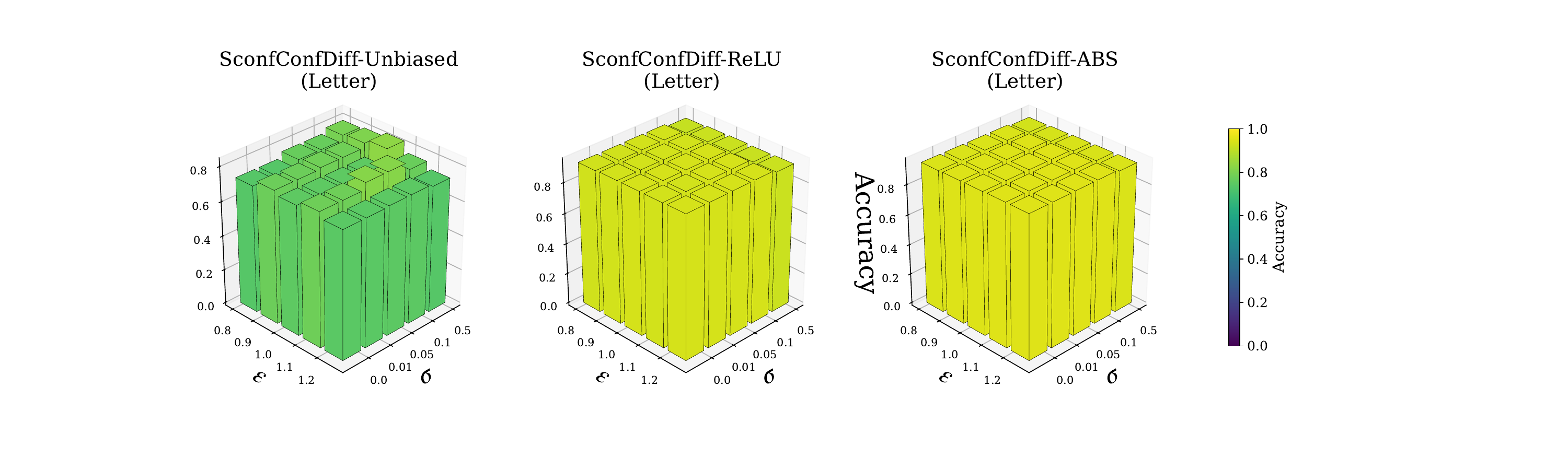}
    \caption{Classification performance on Letter with $\pi_+=0.2$ given an noisy class prior probability, noisy similarity-confidence, and noisy confidence-difference. Both the height and color of the bars represent accuracy.}
    \label{fig:robust_letter}
\end{figure}

\begin{figure}
    \centering
    \includegraphics[width=0.8\linewidth]{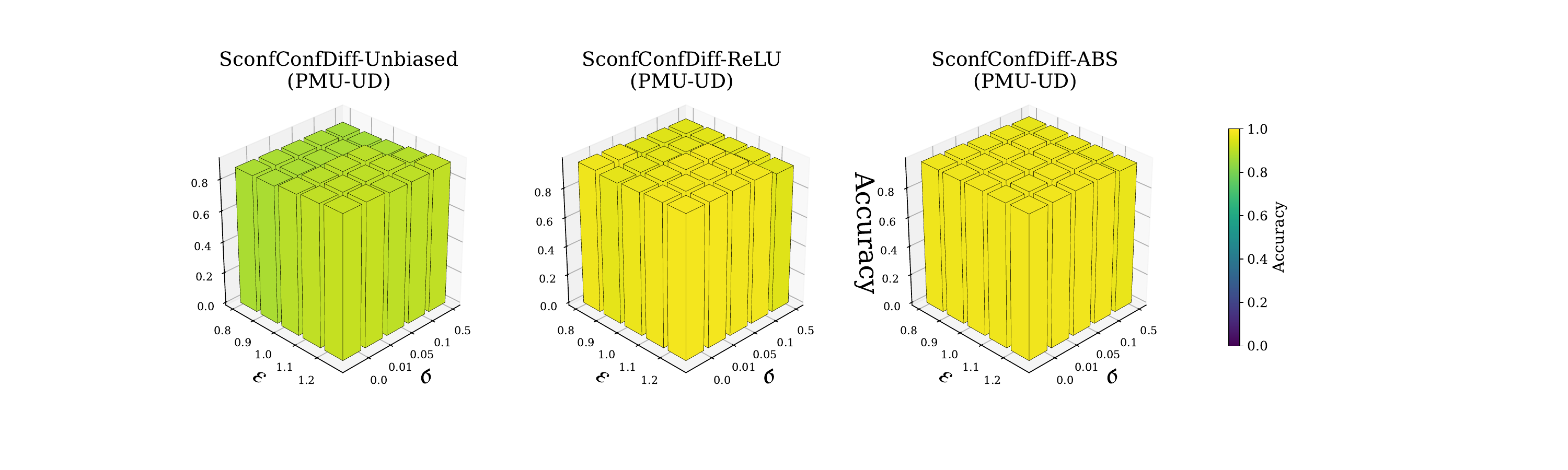}
    \caption{Classification performance on PMU-UD with $\pi_+=0.2$ given an noisy class prior probability, noisy similarity-confidence, and noisy confidence-difference. Both the height and color of the bars represent accuracy.}
    \label{fig:robust_pmu-ud}
\end{figure}

\begin{table}
\caption{The detail of accuracy (mean$\pm$std) on MNIST in Fig~\ref{fig:robust_result}.}
\label{tab:MNIST robust}
\begin{center}
\begin{small}
\begin{tabular}{lcccc}
\toprule
Dataset & ($\epsilon,\sigma$) & Unbiased & ReLU & ABS \\
\midrule
& (0.8,0.0) & 0.857 $\pm$ 0.041 & 0.966 $\pm$ 0.011 & 0.974 $\pm$ 0.003 \\
& (0.8,0.01) & 0.857 $\pm$ 0.041 & 0.966 $\pm$ 0.011 & 0.974 $\pm$ 0.003 \\
& (0.8,0.05) & 0.880 $\pm$ 0.034 & 0.970 $\pm$ 0.004 & 0.974 $\pm$ 0.002 \\
& (0.8,0.1) & 0.880 $\pm$ 0.034 & 0.970 $\pm$ 0.004 & 0.974 $\pm$ 0.002 \\
& (0.8,0.5) & 0.899 $\pm$ 0.028 & 0.959 $\pm$ 0.011 & 0.971 $\pm$ 0.003 \\
& (0.9,0.0) & 0.883 $\pm$ 0.028 & 0.975 $\pm$ 0.005 & 0.976 $\pm$ 0.003 \\
& (0.9,0.01) & 0.883 $\pm$ 0.028 & 0.975 $\pm$ 0.005 & 0.976 $\pm$ 0.003 \\
& (0.9,0.05) & 0.873 $\pm$ 0.051 & 0.969 $\pm$ 0.004 & 0.975 $\pm$ 0.004 \\
& (0.9,0.1) & 0.873 $\pm$ 0.051 & 0.969 $\pm$ 0.004 & 0.975 $\pm$ 0.004 \\
& (0.9,0.5) & 0.865 $\pm$ 0.083 & 0.961 $\pm$ 0.010 & 0.971 $\pm$ 0.001 \\
& (1.0,0.0) & 0.900 $\pm$ 0.035 & 0.975 $\pm$ 0.004 & 0.978 $\pm$ 0.003 \\
& (1.0,0.01) & 0.900 $\pm$ 0.035 & 0.975 $\pm$ 0.004 & 0.978 $\pm$ 0.003 \\
MNIST & (1.0,0.05) & 0.881 $\pm$ 0.050 & 0.973 $\pm$ 0.002 & 0.978 $\pm$ 0.002 \\
& (1.0,0.1) & 0.881 $\pm$ 0.050 & 0.973 $\pm$ 0.002 & 0.978 $\pm$ 0.002 \\
& (1.0,0.5) & 0.830 $\pm$ 0.048 & 0.963 $\pm$ 0.007 & 0.974 $\pm$ 0.002 \\
& (1.1,0.0) & 0.842 $\pm$ 0.031 & 0.974 $\pm$ 0.006 & 0.981 $\pm$ 0.003 \\
& (1.1,0.01) & 0.842 $\pm$ 0.031 & 0.974 $\pm$ 0.006 & 0.981 $\pm$ 0.003 \\
& (1.1,0.05) & 0.849 $\pm$ 0.043 & 0.974 $\pm$ 0.003 & 0.979 $\pm$ 0.003 \\
& (1.1,0.1) & 0.849 $\pm$ 0.043 & 0.974 $\pm$ 0.003 & 0.979 $\pm$ 0.003 \\
& (1.1,0.5) & 0.822 $\pm$ 0.048 & 0.962 $\pm$ 0.008 & 0.975 $\pm$ 0.004 \\
& (1.2,0.0) & 0.819 $\pm$ 0.030 & 0.975 $\pm$ 0.003 & 0.982 $\pm$ 0.001 \\
& (1.2,0.01) & 0.819 $\pm$ 0.030 & 0.975 $\pm$ 0.003 & 0.982 $\pm$ 0.001 \\
& (1.2,0.05) & 0.825 $\pm$ 0.038 & 0.974 $\pm$ 0.002 & 0.980 $\pm$ 0.002 \\
& (1.2,0.1) & 0.825 $\pm$ 0.038 & 0.974 $\pm$ 0.002 & 0.980 $\pm$ 0.002 \\
& (1.2,0.5) & 0.829 $\pm$ 0.027 & 0.965 $\pm$ 0.003 & 0.977 $\pm$ 0.003 \\
\bottomrule
\end{tabular}
\end{small}
\end{center}
\end{table}

\begin{table}
\caption{The detail of accuracy (mean$\pm$std) in Fig~\ref{fig:robust_kmnist}.}
\label{tab:KMNIST robust}
\begin{center}
\begin{small}
\begin{tabular}{lcccc}
\toprule
Dataset & ($\epsilon,\sigma$) & Unbiased & ReLU & ABS \\
\midrule
& (0.8,0.0) & 0.803 $\pm$ 0.032 & 0.879 $\pm$ 0.011 & 0.903 $\pm$ 0.005 \\
& (0.8,0.01) & 0.803 $\pm$ 0.032 & 0.879 $\pm$ 0.011 & 0.903 $\pm$ 0.005 \\
& (0.8,0.05) & 0.805 $\pm$ 0.062 & 0.882 $\pm$ 0.010 & 0.905 $\pm$ 0.006 \\
& (0.8,0.1) & 0.805 $\pm$ 0.062 & 0.882 $\pm$ 0.010 & 0.905 $\pm$ 0.006 \\
& (0.8,0.5) & 0.819 $\pm$ 0.020 & 0.864 $\pm$ 0.007 & 0.898 $\pm$ 0.009 \\
& (0.9,0.0) & 0.784 $\pm$ 0.050 & 0.878 $\pm$ 0.007 & 0.903 $\pm$ 0.006 \\
& (0.9,0.01) & 0.784 $\pm$ 0.050 & 0.878 $\pm$ 0.007 & 0.903 $\pm$ 0.006 \\
& (0.9,0.05) & 0.813 $\pm$ 0.011 & 0.888 $\pm$ 0.007 & 0.908 $\pm$ 0.004 \\
& (0.9,0.1) & 0.813 $\pm$ 0.011 & 0.888 $\pm$ 0.007 & 0.908 $\pm$ 0.004 \\
& (0.9,0.5) & 0.760 $\pm$ 0.062 & 0.873 $\pm$ 0.014 & 0.897 $\pm$ 0.005 \\
& (1.0,0.0) & 0.720 $\pm$ 0.028 & 0.879 $\pm$ 0.009 & 0.906 $\pm$ 0.004 \\
& (1.0,0.01) & 0.720 $\pm$ 0.028 & 0.879 $\pm$ 0.009 & 0.906 $\pm$ 0.004 \\
KMNIST & (1.0,0.05) & 0.698 $\pm$ 0.024 & 0.878 $\pm$ 0.013 & 0.908 $\pm$ 0.006 \\
& (1.0,0.1) & 0.698 $\pm$ 0.024 & 0.878 $\pm$ 0.013 & 0.908 $\pm$ 0.006 \\
& (1.0,0.5) & 0.690 $\pm$ 0.065 & 0.869 $\pm$ 0.014 & 0.898 $\pm$ 0.009 \\
& (1.1,0.0) & 0.707 $\pm$ 0.076 & 0.892 $\pm$ 0.009 & 0.909 $\pm$ 0.003 \\
& (1.1,0.01) & 0.707 $\pm$ 0.076 & 0.892 $\pm$ 0.009 & 0.909 $\pm$ 0.003 \\
& (1.1,0.05) & 0.709 $\pm$ 0.048 & 0.884 $\pm$ 0.005 & 0.910 $\pm$ 0.008 \\
& (1.1,0.1) & 0.709 $\pm$ 0.048 & 0.884 $\pm$ 0.005 & 0.910 $\pm$ 0.008 \\
& (1.1,0.5) & 0.672 $\pm$ 0.044 & 0.867 $\pm$ 0.008 & 0.897 $\pm$ 0.009 \\
& (1.2,0.0) & 0.663 $\pm$ 0.063 & 0.882 $\pm$ 0.013 & 0.906 $\pm$ 0.005 \\
& (1.2,0.01) & 0.663 $\pm$ 0.063 & 0.882 $\pm$ 0.013 & 0.906 $\pm$ 0.005 \\
& (1.2,0.05) & 0.657 $\pm$ 0.031 & 0.884 $\pm$ 0.009 & 0.913 $\pm$ 0.007 \\
& (1.2,0.1) & 0.657 $\pm$ 0.031 & 0.884 $\pm$ 0.009 & 0.913 $\pm$ 0.007 \\
& (1.2,0.5) & 0.652 $\pm$ 0.030 & 0.871 $\pm$ 0.008 & 0.899 $\pm$ 0.010 \\
\bottomrule
\end{tabular}
\end{small}
\end{center}
\end{table}

\begin{table}
\caption{The detail of accuracy (mean$\pm$std) in Fig~\ref{fig:robust_fashion}.}
\label{tab:Fashion robust}
\begin{center}
\begin{small}
\begin{tabular}{lcccc}
\toprule
Dataset & ($\epsilon,\sigma$) & Unbiased & ReLU & ABS \\
\midrule
& (0.8,0.0) & 0.877 $\pm$ 0.017 & 0.964 $\pm$ 0.004 & 0.967 $\pm$ 0.002 \\
& (0.8,0.01) & 0.877 $\pm$ 0.017 & 0.964 $\pm$ 0.004 & 0.967 $\pm$ 0.002 \\
& (0.8,0.05) & 0.844 $\pm$ 0.044 & 0.965 $\pm$ 0.005 & 0.967 $\pm$ 0.002 \\
& (0.8,0.1) & 0.844 $\pm$ 0.044 & 0.965 $\pm$ 0.005 & 0.967 $\pm$ 0.002 \\
& (0.8,0.5) & 0.890 $\pm$ 0.037 & 0.958 $\pm$ 0.008 & 0.964 $\pm$ 0.004 \\
& (0.9,0.0) & 0.883 $\pm$ 0.010 & 0.965 $\pm$ 0.001 & 0.969 $\pm$ 0.003 \\
& (0.9,0.01) & 0.883 $\pm$ 0.010 & 0.965 $\pm$ 0.001 & 0.969 $\pm$ 0.003 \\
& (0.9,0.05) & 0.886 $\pm$ 0.043 & 0.963 $\pm$ 0.004 & 0.969 $\pm$ 0.003 \\
& (0.9,0.1) & 0.886 $\pm$ 0.043 & 0.963 $\pm$ 0.004 & 0.969 $\pm$ 0.003 \\
& (0.9,0.5) & 0.887 $\pm$ 0.040 & 0.960 $\pm$ 0.002 & 0.966 $\pm$ 0.005 \\
& (1.0,0.0) & 0.857 $\pm$ 0.042 & 0.963 $\pm$ 0.002 & 0.969 $\pm$ 0.002 \\
& (1.0,0.01) & 0.857 $\pm$ 0.042 & 0.963 $\pm$ 0.002 & 0.969 $\pm$ 0.002 \\
Fashion & (1.0,0.05) & 0.910 $\pm$ 0.029 & 0.962 $\pm$ 0.004 & 0.970 $\pm$ 0.003 \\
& (1.0,0.1) & 0.910 $\pm$ 0.029 & 0.962 $\pm$ 0.004 & 0.970 $\pm$ 0.003 \\
& (1.0,0.5) & 0.891 $\pm$ 0.054 & 0.960 $\pm$ 0.003 & 0.968 $\pm$ 0.003 \\
& (1.1,0.0) & 0.827 $\pm$ 0.063 & 0.964 $\pm$ 0.003 & 0.971 $\pm$ 0.002 \\
& (1.1,0.01) & 0.827 $\pm$ 0.063 & 0.964 $\pm$ 0.003 & 0.971 $\pm$ 0.002 \\
& (1.1,0.05) & 0.857 $\pm$ 0.035 & 0.966 $\pm$ 0.003 & 0.971 $\pm$ 0.001 \\
& (1.1,0.1) & 0.857 $\pm$ 0.035 & 0.966 $\pm$ 0.003 & 0.971 $\pm$ 0.001 \\
& (1.1,0.5) & 0.881 $\pm$ 0.034 & 0.962 $\pm$ 0.003 & 0.967 $\pm$ 0.002 \\
& (1.2,0.0) & 0.869 $\pm$ 0.029 & 0.966 $\pm$ 0.002 & 0.972 $\pm$ 0.001 \\
& (1.2,0.01) & 0.869 $\pm$ 0.029 & 0.966 $\pm$ 0.002 & 0.972 $\pm$ 0.001 \\
& (1.2,0.05) & 0.886 $\pm$ 0.011 & 0.967 $\pm$ 0.002 & 0.972 $\pm$ 0.003 \\
& (1.2,0.1) & 0.886 $\pm$ 0.011 & 0.967 $\pm$ 0.002 & 0.972 $\pm$ 0.003 \\
& (1.2,0.5) & 0.902 $\pm$ 0.019 & 0.962 $\pm$ 0.003 & 0.969 $\pm$ 0.002 \\
\bottomrule
\end{tabular}
\end{small}
\end{center}
\end{table}

\begin{table}
\caption{The detail of accuracy (mean$\pm$std) in Fig~\ref{fig:robust_CIFAR-10}.}
\label{tab:CIFAR10 robust}
\begin{center}
\begin{small}
\begin{tabular}{lcccc}
\toprule
Dataset & ($\epsilon,\sigma$) & Unbiased & ReLU & ABS \\
\midrule
& (0.8,0.0) & 0.875 $\pm$ 0.003 & 0.829 $\pm$ 0.040 & 0.883 $\pm$ 0.005 \\
& (0.8,0.01) & 0.875 $\pm$ 0.003 & 0.829 $\pm$ 0.040 & 0.883 $\pm$ 0.005 \\
& (0.8,0.05) & 0.865 $\pm$ 0.008 & 0.874 $\pm$ 0.014 & 0.882 $\pm$ 0.007 \\
& (0.8,0.1) & 0.865 $\pm$ 0.008 & 0.874 $\pm$ 0.014 & 0.882 $\pm$ 0.007 \\
& (0.8,0.5) & 0.801 $\pm$ 0.038 & 0.843 $\pm$ 0.007 & 0.871 $\pm$ 0.004 \\
& (0.9,0.0) & 0.873 $\pm$ 0.007 & 0.869 $\pm$ 0.012 & 0.884 $\pm$ 0.005 \\
& (0.9,0.01) & 0.873 $\pm$ 0.007 & 0.869 $\pm$ 0.012 & 0.884 $\pm$ 0.005 \\
& (0.9,0.05) & 0.864 $\pm$ 0.008 & 0.869 $\pm$ 0.013 & 0.882 $\pm$ 0.004 \\
& (0.9,0.1) & 0.864 $\pm$ 0.008 & 0.869 $\pm$ 0.013 & 0.882 $\pm$ 0.004 \\
& (0.9,0.5) & 0.829 $\pm$ 0.014 & 0.848 $\pm$ 0.006 & 0.872 $\pm$ 0.003 \\
& (1.0,0.0) & 0.871 $\pm$ 0.007 & 0.869 $\pm$ 0.010 & 0.885 $\pm$ 0.004 \\
& (1.0,0.01) & 0.871 $\pm$ 0.007 & 0.869 $\pm$ 0.010 & 0.885 $\pm$ 0.004 \\
CIFAR10 & (1.0,0.05) & 0.862 $\pm$ 0.009 & 0.866 $\pm$ 0.010 & 0.883 $\pm$ 0.005 \\
& (1.0,0.1) & 0.862 $\pm$ 0.009 & 0.866 $\pm$ 0.010 & 0.883 $\pm$ 0.005 \\
& (1.0,0.5) & 0.829 $\pm$ 0.008 & 0.847 $\pm$ 0.007 & 0.876 $\pm$ 0.004 \\
& (1.1,0.0) & 0.871 $\pm$ 0.009 & 0.868 $\pm$ 0.007 & 0.887 $\pm$ 0.004 \\
& (1.1,0.01) & 0.871 $\pm$ 0.009 & 0.868 $\pm$ 0.007 & 0.887 $\pm$ 0.004 \\
& (1.1,0.05) & 0.856 $\pm$ 0.015 & 0.869 $\pm$ 0.009 & 0.880 $\pm$ 0.001 \\
& (1.1,0.1) & 0.856 $\pm$ 0.015 & 0.869 $\pm$ 0.009 & 0.880 $\pm$ 0.001 \\
& (1.1,0.5) & 0.829 $\pm$ 0.016 & 0.847 $\pm$ 0.011 & 0.879 $\pm$ 0.002 \\
& (1.2,0.0) & 0.873 $\pm$ 0.010 & 0.870 $\pm$ 0.010 & 0.884 $\pm$ 0.004 \\
& (1.2,0.01) & 0.873 $\pm$ 0.010 & 0.870 $\pm$ 0.010 & 0.884 $\pm$ 0.004 \\
& (1.2,0.05) & 0.864 $\pm$ 0.005 & 0.868 $\pm$ 0.010 & 0.880 $\pm$ 0.002 \\
& (1.2,0.1) & 0.864 $\pm$ 0.005 & 0.868 $\pm$ 0.010 & 0.880 $\pm$ 0.002 \\
& (1.2,0.5) & 0.835 $\pm$ 0.006 & 0.850 $\pm$ 0.005 & 0.882 $\pm$ 0.004 \\
\bottomrule
\end{tabular}
\end{small}
\end{center}
\end{table}

\begin{table}
\caption{The detail of accuracy (mean$\pm$std) Optdigits in Fig~\ref{fig:robust_result}.}
\label{tab:Optdigits robust}
\begin{center}
\begin{small}
\begin{tabular}{lcccc}
\toprule
Dataset & ($\epsilon,\sigma$) & Unbiased & ReLU & ABS \\
\midrule
& (0.8,0.0) & 0.917 $\pm$ 0.042 & 0.950 $\pm$ 0.019 & 0.961 $\pm$ 0.006 \\
& (0.8,0.01) & 0.917 $\pm$ 0.042 & 0.950 $\pm$ 0.019 & 0.961 $\pm$ 0.006 \\
& (0.8,0.05) & 0.914 $\pm$ 0.031 & 0.953 $\pm$ 0.005 & 0.961 $\pm$ 0.007 \\
& (0.8,0.1) & 0.914 $\pm$ 0.031 & 0.953 $\pm$ 0.005 & 0.961 $\pm$ 0.007 \\
& (0.8,0.5) & 0.875 $\pm$ 0.044 & 0.933 $\pm$ 0.020 & 0.942 $\pm$ 0.018 \\
& (0.9,0.0) & 0.923 $\pm$ 0.048 & 0.949 $\pm$ 0.015 & 0.966 $\pm$ 0.007 \\
& (0.9,0.01) & 0.923 $\pm$ 0.048 & 0.949 $\pm$ 0.015 & 0.966 $\pm$ 0.007 \\
& (0.9,0.05) & 0.918 $\pm$ 0.029 & 0.946 $\pm$ 0.024 & 0.962 $\pm$ 0.011 \\
& (0.9,0.1) & 0.918 $\pm$ 0.029 & 0.946 $\pm$ 0.024 & 0.962 $\pm$ 0.011 \\
& (0.9,0.5) & 0.877 $\pm$ 0.046 & 0.917 $\pm$ 0.031 & 0.945 $\pm$ 0.014 \\
& (1.0,0.0) & 0.925 $\pm$ 0.052 & 0.954 $\pm$ 0.017 & 0.965 $\pm$ 0.007 \\
& (1.0,0.01) & 0.925 $\pm$ 0.052 & 0.954 $\pm$ 0.017 & 0.965 $\pm$ 0.007 \\
Optdigits & (1.0,0.05) & 0.886 $\pm$ 0.072 & 0.959 $\pm$ 0.012 & 0.970 $\pm$ 0.006 \\
& (1.0,0.1) & 0.886 $\pm$ 0.072 & 0.959 $\pm$ 0.012 & 0.970 $\pm$ 0.006 \\
& (1.0,0.5) & 0.872 $\pm$ 0.052 & 0.927 $\pm$ 0.024 & 0.949 $\pm$ 0.016 \\
& (1.1,0.0) & 0.921 $\pm$ 0.058 & 0.957 $\pm$ 0.012 & 0.970 $\pm$ 0.006 \\
& (1.1,0.01) & 0.921 $\pm$ 0.058 & 0.957 $\pm$ 0.012 & 0.970 $\pm$ 0.006 \\
& (1.1,0.05) & 0.908 $\pm$ 0.055 & 0.955 $\pm$ 0.020 & 0.971 $\pm$ 0.007 \\
& (1.1,0.1) & 0.908 $\pm$ 0.055 & 0.955 $\pm$ 0.020 & 0.971 $\pm$ 0.007 \\
& (1.1,0.5) & 0.870 $\pm$ 0.052 & 0.947 $\pm$ 0.019 & 0.953 $\pm$ 0.012 \\
& (1.2,0.0) & 0.910 $\pm$ 0.057 & 0.970 $\pm$ 0.008 & 0.974 $\pm$ 0.007 \\
& (1.2,0.01) & 0.910 $\pm$ 0.057 & 0.970 $\pm$ 0.008 & 0.974 $\pm$ 0.007 \\
& (1.2,0.05) & 0.919 $\pm$ 0.058 & 0.964 $\pm$ 0.009 & 0.972 $\pm$ 0.006 \\
& (1.2,0.1) & 0.919 $\pm$ 0.058 & 0.964 $\pm$ 0.009 & 0.972 $\pm$ 0.006 \\
& (1.2,0.5) & 0.892 $\pm$ 0.051 & 0.944 $\pm$ 0.013 & 0.953 $\pm$ 0.013 \\
\bottomrule
\end{tabular}
\end{small}
\end{center}
\end{table}

\begin{table}
\caption{The detail of accuracy (mean$\pm$std) on in Fig~\ref{fig:robust_pendigits}.}
\label{tab:Pendigits robust}
\begin{center}
\begin{small}
\begin{tabular}{lcccc}
\toprule
Dataset & ($\epsilon,\sigma$) & Unbiased & ReLU & ABS \\
\midrule
& (0.8,0.0) & 0.870 $\pm$ 0.086 & 0.986 $\pm$ 0.004 & 0.987 $\pm$ 0.002 \\
& (0.8,0.01) & 0.870 $\pm$ 0.086 & 0.986 $\pm$ 0.004 & 0.987 $\pm$ 0.002 \\
& (0.8,0.05) & 0.870 $\pm$ 0.076 & 0.987 $\pm$ 0.005 & 0.988 $\pm$ 0.004 \\
& (0.8,0.1) & 0.870 $\pm$ 0.076 & 0.987 $\pm$ 0.005 & 0.988 $\pm$ 0.004 \\
& (0.8,0.5) & 0.908 $\pm$ 0.070 & 0.977 $\pm$ 0.008 & 0.983 $\pm$ 0.005 \\
& (0.9,0.0) & 0.870 $\pm$ 0.100 & 0.988 $\pm$ 0.003 & 0.988 $\pm$ 0.003 \\
& (0.9,0.01) & 0.870 $\pm$ 0.100 & 0.988 $\pm$ 0.003 & 0.988 $\pm$ 0.003 \\
& (0.9,0.05) & 0.894 $\pm$ 0.053 & 0.987 $\pm$ 0.004 & 0.989 $\pm$ 0.003 \\
& (0.9,0.1) & 0.894 $\pm$ 0.053 & 0.987 $\pm$ 0.004 & 0.989 $\pm$ 0.003 \\
& (0.9,0.5) & 0.941 $\pm$ 0.020 & 0.976 $\pm$ 0.007 & 0.986 $\pm$ 0.004 \\
& (1.0,0.0) & 0.927 $\pm$ 0.081 & 0.987 $\pm$ 0.003 & 0.991 $\pm$ 0.003 \\
& (1.0,0.01) & 0.927 $\pm$ 0.081 & 0.987 $\pm$ 0.003 & 0.991 $\pm$ 0.003 \\
Pendigits & (1.0,0.05) & 0.954 $\pm$ 0.038 & 0.987 $\pm$ 0.004 & 0.991 $\pm$ 0.003 \\
& (1.0,0.1) & 0.954 $\pm$ 0.038 & 0.987 $\pm$ 0.004 & 0.991 $\pm$ 0.003 \\
& (1.0,0.5) & 0.951 $\pm$ 0.026 & 0.981 $\pm$ 0.007 & 0.986 $\pm$ 0.004 \\
& (1.1,0.0) & 0.964 $\pm$ 0.042 & 0.989 $\pm$ 0.003 & 0.990 $\pm$ 0.002 \\
& (1.1,0.01) & 0.964 $\pm$ 0.042 & 0.989 $\pm$ 0.003 & 0.990 $\pm$ 0.002 \\
& (1.1,0.05) & 0.959 $\pm$ 0.034 & 0.989 $\pm$ 0.003 & 0.991 $\pm$ 0.003 \\
& (1.1,0.1) & 0.959 $\pm$ 0.034 & 0.989 $\pm$ 0.003 & 0.991 $\pm$ 0.003 \\
& (1.1,0.5) & 0.938 $\pm$ 0.027 & 0.982 $\pm$ 0.008 & 0.986 $\pm$ 0.004 \\
& (1.2,0.0) & 0.964 $\pm$ 0.036 & 0.990 $\pm$ 0.004 & 0.992 $\pm$ 0.002 \\
& (1.2,0.01) & 0.964 $\pm$ 0.036 & 0.990 $\pm$ 0.004 & 0.992 $\pm$ 0.002 \\
& (1.2,0.05) & 0.964 $\pm$ 0.037 & 0.989 $\pm$ 0.004 & 0.991 $\pm$ 0.002 \\
& (1.2,0.1) & 0.964 $\pm$ 0.037 & 0.989 $\pm$ 0.004 & 0.991 $\pm$ 0.002 \\
& (1.2,0.5) & 0.924 $\pm$ 0.031 & 0.981 $\pm$ 0.006 & 0.984 $\pm$ 0.004 \\
\bottomrule
\end{tabular}
\end{small}
\end{center}
\end{table}

\begin{table}
\caption{The detail of accuracy (mean$\pm$std) in Fig~\ref{fig:robust_letter}.}
\label{tab:Letter robust}
\begin{center}
\begin{small}
\begin{tabular}{lcccc}
\toprule
Dataset & ($\epsilon,\sigma$) & Unbiased & ReLU & ABS \\
\midrule
& (0.8,0.0) & 0.737 $\pm$ 0.037 & 0.929 $\pm$ 0.004 & 0.943 $\pm$ 0.005 \\
& (0.8,0.01) & 0.737 $\pm$ 0.037 & 0.929 $\pm$ 0.004 & 0.943 $\pm$ 0.005 \\
& (0.8,0.05) & 0.769 $\pm$ 0.035 & 0.930 $\pm$ 0.004 & 0.943 $\pm$ 0.004 \\
& (0.8,0.1) & 0.769 $\pm$ 0.035 & 0.930 $\pm$ 0.004 & 0.943 $\pm$ 0.004 \\
& (0.8,0.5) & 0.797 $\pm$ 0.034 & 0.913 $\pm$ 0.004 & 0.933 $\pm$ 0.004 \\
& (0.9,0.0) & 0.774 $\pm$ 0.023 & 0.933 $\pm$ 0.004 & 0.947 $\pm$ 0.005 \\
& (0.9,0.01) & 0.774 $\pm$ 0.023 & 0.933 $\pm$ 0.004 & 0.947 $\pm$ 0.005 \\
& (0.9,0.05) & 0.782 $\pm$ 0.042 & 0.932 $\pm$ 0.007 & 0.948 $\pm$ 0.005 \\
& (0.9,0.1) & 0.782 $\pm$ 0.042 & 0.932 $\pm$ 0.007 & 0.948 $\pm$ 0.005 \\
& (0.9,0.5) & 0.807 $\pm$ 0.036 & 0.919 $\pm$ 0.007 & 0.936 $\pm$ 0.005 \\
& (1.0,0.0) & 0.752 $\pm$ 0.067 & 0.934 $\pm$ 0.007 & 0.948 $\pm$ 0.005 \\
& (1.0,0.01) & 0.752 $\pm$ 0.067 & 0.934 $\pm$ 0.007 & 0.948 $\pm$ 0.005 \\
Letter & (1.0,0.05) & 0.750 $\pm$ 0.054 & 0.934 $\pm$ 0.003 & 0.949 $\pm$ 0.004 \\
& (1.0,0.1) & 0.750 $\pm$ 0.054 & 0.934 $\pm$ 0.003 & 0.949 $\pm$ 0.004 \\
& (1.0,0.5) & 0.831 $\pm$ 0.043 & 0.920 $\pm$ 0.005 & 0.939 $\pm$ 0.005 \\
& (1.1,0.0) & 0.779 $\pm$ 0.048 & 0.933 $\pm$ 0.007 & 0.951 $\pm$ 0.001 \\
& (1.1,0.01) & 0.779 $\pm$ 0.048 & 0.933 $\pm$ 0.007 & 0.951 $\pm$ 0.001 \\
& (1.1,0.05) & 0.817 $\pm$ 0.053 & 0.937 $\pm$ 0.005 & 0.951 $\pm$ 0.005 \\
& (1.1,0.1) & 0.817 $\pm$ 0.053 & 0.937 $\pm$ 0.005 & 0.951 $\pm$ 0.005 \\
& (1.1,0.5) & 0.771 $\pm$ 0.030 & 0.918 $\pm$ 0.006 & 0.941 $\pm$ 0.004 \\
& (1.2,0.0) & 0.745 $\pm$ 0.059 & 0.936 $\pm$ 0.007 & 0.952 $\pm$ 0.005 \\
& (1.2,0.01) & 0.745 $\pm$ 0.059 & 0.936 $\pm$ 0.007 & 0.952 $\pm$ 0.005 \\
& (1.2,0.05) & 0.749 $\pm$ 0.042 & 0.936 $\pm$ 0.007 & 0.952 $\pm$ 0.004 \\
& (1.2,0.1) & 0.749 $\pm$ 0.042 & 0.936 $\pm$ 0.007 & 0.952 $\pm$ 0.004 \\
& (1.2,0.5) & 0.736 $\pm$ 0.024 & 0.920 $\pm$ 0.003 & 0.942 $\pm$ 0.006 \\
\bottomrule
\end{tabular}
\end{small}
\end{center}
\end{table}

\begin{table}
\caption{The detail of accuracy (mean$\pm$std) in Fig~\ref{fig:robust_pmu-ud}.}
\label{tab:PMU-UD robust}
\begin{center}
\begin{small}
\begin{tabular}{lcccc}
\toprule
Dataset & ($\epsilon,\sigma$) & Unbiased & ReLU & ABS \\
\midrule
& (0.8,0.0) & 0.873 $\pm$ 0.066 & 0.973 $\pm$ 0.008 & 0.975 $\pm$ 0.010 \\
& (0.8,0.01) & 0.873 $\pm$ 0.066 & 0.973 $\pm$ 0.008 & 0.975 $\pm$ 0.010 \\
& (0.8,0.05) & 0.870 $\pm$ 0.064 & 0.953 $\pm$ 0.020 & 0.971 $\pm$ 0.009 \\
& (0.8,0.1) & 0.870 $\pm$ 0.064 & 0.953 $\pm$ 0.020 & 0.971 $\pm$ 0.009 \\
& (0.8,0.5) & 0.862 $\pm$ 0.058 & 0.950 $\pm$ 0.019 & 0.964 $\pm$ 0.011 \\
& (0.9,0.0) & 0.875 $\pm$ 0.085 & 0.960 $\pm$ 0.032 & 0.975 $\pm$ 0.009 \\
& (0.9,0.01) & 0.875 $\pm$ 0.085 & 0.960 $\pm$ 0.032 & 0.975 $\pm$ 0.009 \\
& (0.9,0.05) & 0.872 $\pm$ 0.083 & 0.959 $\pm$ 0.017 & 0.977 $\pm$ 0.009 \\
& (0.9,0.1) & 0.872 $\pm$ 0.083 & 0.959 $\pm$ 0.017 & 0.977 $\pm$ 0.009 \\
& (0.9,0.5) & 0.863 $\pm$ 0.072 & 0.955 $\pm$ 0.006 & 0.967 $\pm$ 0.009 \\
& (1.0,0.0) & 0.892 $\pm$ 0.092 & 0.970 $\pm$ 0.018 & 0.978 $\pm$ 0.007 \\
& (1.0,0.01) & 0.892 $\pm$ 0.092 & 0.970 $\pm$ 0.018 & 0.978 $\pm$ 0.007 \\
PMU-UD & (1.0,0.05) & 0.898 $\pm$ 0.069 & 0.978 $\pm$ 0.005 & 0.976 $\pm$ 0.007 \\
& (1.0,0.1) & 0.898 $\pm$ 0.069 & 0.978 $\pm$ 0.005 & 0.976 $\pm$ 0.007 \\
& (1.0,0.5) & 0.872 $\pm$ 0.068 & 0.955 $\pm$ 0.023 & 0.965 $\pm$ 0.014 \\
& (1.1,0.0) & 0.904 $\pm$ 0.074 & 0.978 $\pm$ 0.009 & 0.974 $\pm$ 0.012 \\
& (1.1,0.01) & 0.904 $\pm$ 0.074 & 0.978 $\pm$ 0.009 & 0.974 $\pm$ 0.012 \\
& (1.1,0.05) & 0.901 $\pm$ 0.080 & 0.979 $\pm$ 0.008 & 0.980 $\pm$ 0.007 \\
& (1.1,0.1) & 0.901 $\pm$ 0.080 & 0.979 $\pm$ 0.008 & 0.980 $\pm$ 0.007 \\
& (1.1,0.5) & 0.889 $\pm$ 0.075 & 0.959 $\pm$ 0.011 & 0.968 $\pm$ 0.014 \\
& (1.2,0.0) & 0.912 $\pm$ 0.045 & 0.981 $\pm$ 0.009 & 0.977 $\pm$ 0.014 \\
& (1.2,0.01) & 0.912 $\pm$ 0.045 & 0.981 $\pm$ 0.009 & 0.977 $\pm$ 0.014 \\
& (1.2,0.05) & 0.900 $\pm$ 0.052 & 0.978 $\pm$ 0.002 & 0.980 $\pm$ 0.006 \\
& (1.2,0.1) & 0.900 $\pm$ 0.052 & 0.978 $\pm$ 0.002 & 0.980 $\pm$ 0.006 \\
& (1.2,0.5) & 0.906 $\pm$ 0.072 & 0.950 $\pm$ 0.011 & 0.967 $\pm$ 0.013 \\
\bottomrule
\end{tabular}
\end{small}
\end{center}
\end{table}

\newpage

\section{Geometric Interpretation of Sconf and ConfDiff}\label{appendix:geometry}
In this section, we provide a geometric interpretation of similarity-confidence and confidence-difference. Based on the definition of similarity-confidence, the following equation holds.
\begin{align*}
s(\bm{x},\bm{x}')=&p(y=y' \mid \bm{x},\bm{x}')\\
=&p(y=y'=+1 \mid \bm{x},\bm{x}')+p(y=y'=-1 \mid \bm{x},\bm{x}')\\
=&\frac{p(\bm{x},y=+1,\bm{x}',y'=+1)+p(\bm{x},y=-1,\bm{x}',y'=-1)}{p(\bm{x},\bm{x}')}\\
=&\frac{p(\bm{x},y=+1)p(\bm{x}',y'=+1)+p(\bm{x},y=-1)p(\bm{x}',y'=-1)}{p(\bm{x})p(\bm{x}')}\\
=&p(y=+1 \mid \bm{x})p(y'=+1 \mid \bm{x}')+p(y=-1 \mid \bm{x})p(y'=-1 \mid \bm{x}')\\
=&p(y=+1 \mid \bm{x})p(y'=+1 \mid \bm{x}')+(1-p(y=+1 \mid \bm{x}))(1-p(y'=+1 \mid \bm{x}'))
\end{align*}

According to this result, in Fig.~\ref{fig:geometry}, similarity-confidence computes the inner product of their confidences between two samples using the inner product of their confidences. In contrast, confidence-difference is based on the difference along the horizontal axis in Fig.~\ref{fig:geometry}. Therefore, both similarity-confidence and confidence-difference can be interpreted as measuring the relative position of confidence values between two samples. However, since they measure it in different ways, the label information encoded in the data pairs is considered to be fundamentally different.

\begin{figure}
    \centering
    \includegraphics[width=0.8\linewidth]{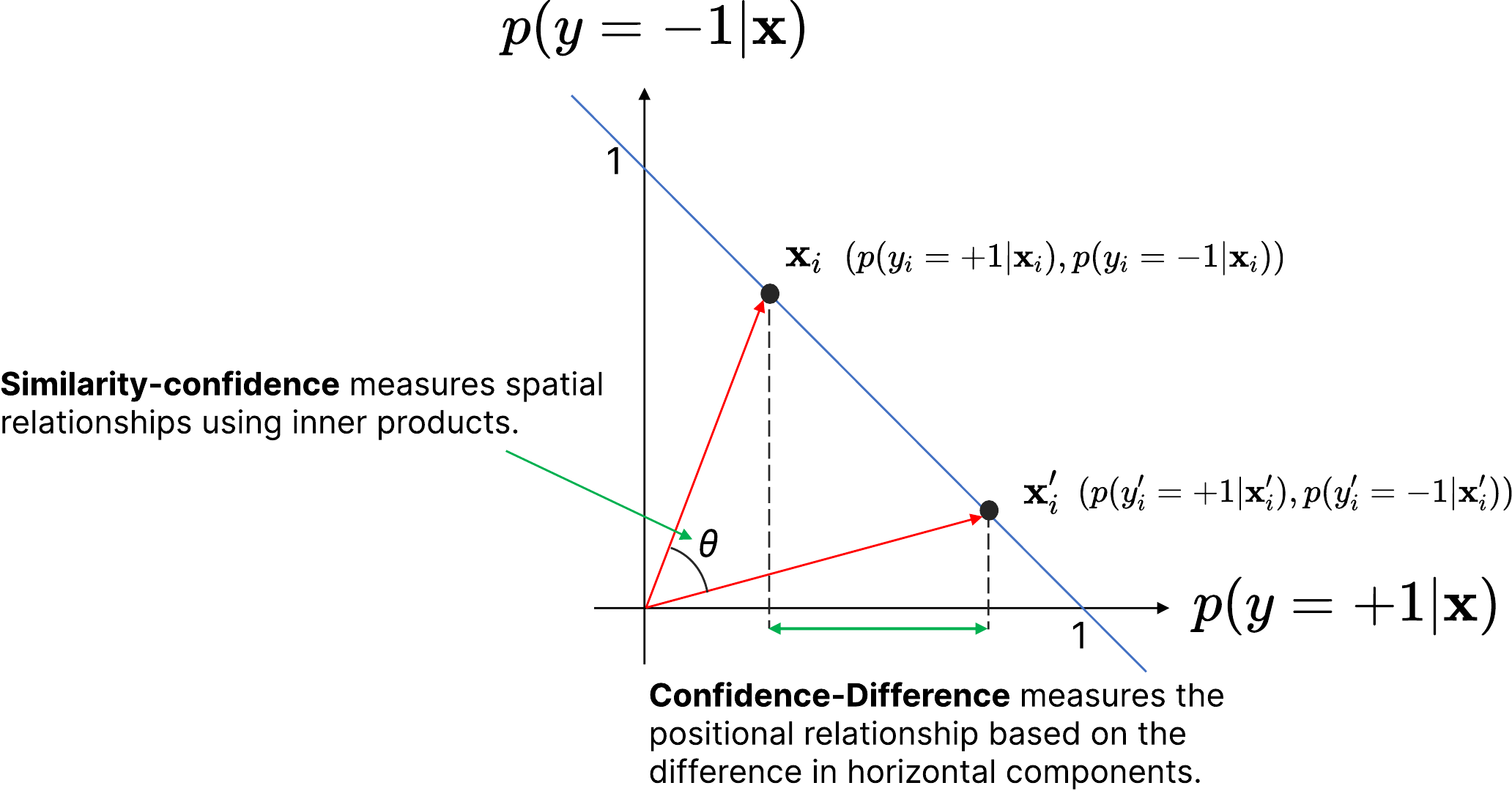}
    \caption{Geometric Interpretation of Similarity-Confidence and Confidence-Difference.}
    \label{fig:geometry}
\end{figure}

\end{document}